\documentclass{article}

\usepackage[preprint]{neurips_2024}

\usepackage[utf8]{inputenc} 
\usepackage[T1]{fontenc}    
\usepackage{url}            
\usepackage{booktabs}       
\usepackage{amsfonts,amsmath,amssymb,mathtools,amsthm}       
\usepackage{nicefrac}       
\usepackage{microtype}      
\usepackage{xcolor}         
\usepackage{enumerate}
\usepackage{soul}
\usepackage{caption,wrapfig,subcaption}
\usepackage{algorithm}
\usepackage[noend]{algpseudocode}
\usepackage{hyperref}
\theoremstyle{plain}
\newtheorem{thm}{Theorem}[section]

\newtheorem{lem}[thm]{Lemma}

\theoremstyle{definition}
\newtheorem{defn}[thm]{Definition}

\theoremstyle{remark}
\newtheorem{remark}[thm]{Remark}

\usepackage{bm}

\DeclareMathOperator*{\argmin}{arg\,min}

\title{Learning the Pareto Front Using Bootstrapped Observation Samples}

\author{%
  Wonyoung Kim\thanks{Correspondence to Wonyoung Kim (wk2389@columbia.edu)} \\
  Columbia University\\
  \And
  Garud Iyengar \\
  Columbia University\\
   \And
  Assaf Zeevi \\
  Columbia University\\
}

\begin{document}

\def\gi#1{\textcolor{purple}{#1}}
\def\wy#1{\textcolor{blue}{#1}}

\global\long\def\Real{\mathbb{R}}%
\global\long\def\Natural{\mathbb{N}}%
\global\long\def\Expectation{\mathbb{E}}%
\global\long\def\Probability{\mathbb{P}}%
\global\long\def\Var{\mathbb{V}}%
\global\long\def\CE#1#2{\Expectation\left[\left.#1\right|#2\right]}%
\global\long\def\CP#1#2{\Probability\left(\left.#1\right|#2\right)}%
\global\long\def\abs#1{\left|#1\right|}%
\global\long\def\norm#1{\left\Vert #1\right\Vert }%
\global\long\def\Indicator#1{\mathbb{I}\left(#1\right)}%
\global\long\def\mP{\mathcal{P}}%
\global\long\def\mA{\mathcal{A}}%
\global\long\def\mD{\mathcal{D}}%
\global\long\def\mE{\mathcal{E}}%
\global\long\def\mS{\mathcal{S}}%
\global\long\def\mC{\mathcal{C}}%
\global\long\def\mX{\mathcal{X}}%
\global\long\def\Parameter#1{\theta_{\star}^{\langle#1\rangle}}%
\global\long\def\Stackedparameter#1#2{\Theta_{#2}^{\star(#1)}}%
\global\long\def\Setofcontexts#1{\mathcal{X}_{#1}}%
\global\long\def\Reward#1#2#3{Y_{#2,#3}^{\langle #1\rangle}}%
\global\long\def\Vecreward#1#2{\mathbf{Y}_{#1,#2}}%
\global\long\def\Ereward#1#2{y_{#2}^{\langle #1\rangle}}%
\global\long\def\ErewardVec#1{y_{#1}}%
\global\long\def\Hreward#1#2#3{\widehat{y}_{#2,#3}^{(#1)}}%
\global\long\def\Context#1{x_{#1}}%
\global\long\def\Tx#1#2{\check{X}_{#1,#2}}%
\global\long\def\Ty#1#2#3{\check{Y}_{#2,#3}^{\langle #1 \rangle}}%
\global\long\def\Newy#1#2#3{\tilde{Y}_{#2,#3}^{\langle #1\rangle}}%
\global\long\def\Error#1#2{\eta_{#2}^{(#1)}}%
\global\long\def\Econtext#1{x_{#1}}%
\global\long\def\Estimator#1#2{\widehat{\theta}_{#2}^{\langle#1\rangle}}%
\global\long\def\Newcontext#1#2{\tilde{X}{}_{#1,#2}}%
\global\long\def\StackedEstimator#1{\widehat{\Theta}_{#1}}%
\global\long\def\DREstimator#1{\widehat{\Theta}_{#1}^{DR}}%
\global\long\def\IPW#1#2{\widehat{\theta}_{#2}^{IPW(#1)}}%
\global\long\def\PseudoY#1#2#3{\widehat{Y}_{#2,#3}^{\langle#1\rangle}}%
\global\long\def\TildeError#1#2#3{\tilde{\eta}_{#2,#3}^{(#1)}}%
\global\long\def\NewProb#1#2{\phi_{#1,#2}}%
\global\long\def\StackedContext#1#2#3{\mathbf{X}_{#1,#2,#3}}%
\global\long\def\Action#1{a_{#1}}%
\global\long\def\PseudoAction#1{\tilde{a}_{#1}}%
\global\long\def\Filtration#1{\mathcal{F}_{#1}}%
\global\long\def\History#1{\mathcal{H}_{#1}}%
\global\long\def\XX#1#2{\boldsymbol{X}_{#1,#2}}%
\global\long\def\sign#1{\text{sign}(#1)}%
\global\long\def\Impute#1#2{\check{\theta}_{#1}^{\langle#2\rangle}}%
\global\long\def\Maxeigen#1{\lambda_{\max}\!\left(#1\right)}%
\global\long\def\Mineigen#1{\lambda_{\min}\!\left(#1\right)}%
\global\long\def\Trace#1{\text{Tr}\left(#1\right)}%
\global\long\def\SetofContexts#1{\mathcal{X}_{#1}}%
\global\long\def\Ridgebeta#1{\widehat{\beta}_{#1}^{ridge}}%
\global\long\def\PFIwR{\texttt{PFIwR}}%

\maketitle

\begin{abstract}
We consider Pareto front identification~(PFI) for linear bandits (PFILin), i.e., the goal is to identify a set of arms with undominated 
mean reward vectors when the mean reward vector is a linear function of the context. 
PFILin includes the best arm identification problem and multi-objective active learning as special cases.   
The sample complexity of our proposed algorithm is optimal up to a logarithmic factor.   
In addition, the regret incurred by our algorithm during the estimation is within a logarithmic factor of the optimal regret among all algorithms that identify the Pareto front.
Our key contribution is a new estimator that in every round updates the estimate for the unknown parameter along \emph{multiple} context directions -- in contrast to the conventional estimator that only updates
the parameter estimate along the chosen context. 
This allows us to use low-regret arms to collect information about Pareto optimal arms. 
Our key innovation is to reuse the exploration
samples multiple times; in contrast to conventional estimators that use each sample only once. 
Numerical experiments demonstrate that the proposed algorithm successfully identifies the Pareto front while controlling the regret. 
\end{abstract}


\section{Introduction}
Consider a setting where one has to select among a finite set of actions that have multiple different characteristics, see, e.g.,  ~\citep{lizotte2010efficient,van2014multi,lin2019pareto}. 
A classical example is prescribing a drug to a patient, where one needs to consider its efficacy, toxicity, and potentially all its side effects.  
The efficacy and various side effects typically also depend on patient characteristics. 
Such examples can be found also in online platforms, e-commerce sites, and are pertinent to the design of most recommender systems.   

The problem of selecting an action that has multiple attributes is typically modeled using the concept of  Pareto optimality, and the learning problem reduces to identifying the Pareto front~\citep{goel2007response}, i.e. the set of actions that are not dominated, and therefore, potentially optimal for some user. 
We consider Pareto front identification~(PFI) for linear bandits (PFILin), where the attributes of each action are a linear function of an associated context. 
PFILin generalizes both the best arm identification (BAI) problems
and PFI for MABs. 
We propose an algorithm \PFIwR\ whose sample complexity is optimal to within logarithmic factors.

A ``good'' PFI algorithm should ideally have a both low sample complexity as well as low regret during the identification period. 
\citet{degenne2019bridging,zhong2023achieving} discuss the trade-off between regret and sample complexity in the context of BAI for drug testing. 
Such considerations are also important for e-commerce platforms where high regret could lead to low customer satisfaction and underexposure of products. 
We show \PFIwR\ has close to optimal (within logarithmic factors) regret among all PFI
algorithms.  

In particular, the Pareto front in the multi-objective setting typically has multiple arms, and hence, an algorithm may be forced to collect samples from high regret arms in order to decide whether it is on the Pareto front and minimizing regret is more challenging.
In the linear bandit setting, an algorithm must carefully choose actions so that corresponding contexts support efficient parameter estimation~\citep{soare2014best,tao2018best}.
Consequently, the challenging part of designing an algorithm is to allow suitable exploration and identification of the Pareto front, while controlling for the regret associated with these
arm choices.

To resolve this challenge, we propose the \emph{exploration-mixed estimator} which ``mixes'' the observations during an exploitation round with bootstrapped samples from a previous exploration round, i.e., the estimator ``recycles" the samples in the exploration phase.  
The recycling is key in enabling the estimator to update along several context directions in every round.  
This allows us to explore high-regret actions only for logarithimically increasing exploration rounds, and exploiting low-regret actions after that.
However, recycling samples may cause dependency, and higher estimation error as compared to that of the conventional estimators.
We offset the higher error of the exploration-mixed estimator, by using a doubly-robust (DR) estimator \citep{bang2005doubly}, that is robust to the error of the estimator used to impute the rewards for actions that are not selected.   
These methods ensure we can simultaneously learn rewards for PFI and select arms to minimize regret.
The main contributions of this paper are as follows:
\begin{enumerate}[(i)]
\item We introduce a novel estimation procedure for linear bandit feedback
  that ensures $\tilde{O}(\sqrt{d/t})$ convergence rate for the reward
  vectors of \emph{all} arms while largely exploiting low regret
  arms~(Theorem~\ref{thm:self}).    
This uniform convergence is possible due to two innovations: (i) the novel
\emph{exploration-mixed estimator} that reuses the observations in the
past exploration rounds (Section~\ref{sec:mix}); and (ii) construction of
a DR estimate for unobserved rewards which is robust to the error of the
exploration-mixed estimator (Section~\ref{sec:DR}). 
\item We apply the novel estimation paradigm to PFILIn and propose a new
  algorithm \texttt{PFIwR} with sample complexity that is optimal up to
  logarithmic factors, and has $\tilde{O}(\sqrt{d/t})$ Pareto regret in round
  $t$ with context dimension $d$, after $O(d^3\log dt^2)$ initial
  exploration rounds independent of the problem complexity
  (Theorem~\ref{thm:regret}).   
  Further, the algorithm is shown to achieves optimal order regret among all
  PFI algorithms (Theorem~\ref{thm:regret_lower_bound}).   
\item Experimental results clearly show the estimator converges on the
  rewards of all contexts while exploiting low-regret arms,  and
  \texttt{PFIwR} has significantly superior performance to previously
  known algorithms for {\it both} PFI and regret minimization.   
\end{enumerate}

\begin{table*}[t]
\caption{A comparison of the related works in terms of settings and theoretical guarantees.}
\label{tab:related_works}
\vskip -20pt
\begin{center}
\begin{scriptsize}
\begin{sc}
\begin{tabular}{lccccr}
\toprule
 & Bandit Setting & Multi-objective? & Regret Bound? & PAC bound? \\
\midrule
\citet{valko2013finite} & Kernel & $\times$ & $\surd$ & $\times$\\
\citet{soare2014best} & Linear & $\times$ &  $\times$ & $\surd$\\
\citet{zuluaga2016varepsilon} & Gaussian Process & $\surd$ & $\times$ & $\surd$ \\
\citet{auer2016pareto} & Multi-armed & $\surd$ & $\times$ & $\surd$ \\
\citet{lu2019multi} & Generalized linear & $\surd$ &  $\surd$ & $\times$ \\
\citet{degenne2019bridging} & Multi-armed & $\times$ & $\surd$ & $\surd$\\
\citet{zhong2023achieving} & Multi-armed & $\times$ & $\surd$ & $\surd$ \\
Our work   & Linear & $\surd$ & $\surd$ & $\surd$ \\
\bottomrule
\end{tabular}
\end{sc}
\end{scriptsize}
\end{center}
\vskip -10pt
\end{table*}

\section{Related Work}
The typical approach in multi-objective rewards is to scalarize the problem by either setting the objective to be a weighted combination of all the objective~\citep{roijers2017interactive,roijers2018interactive,wanigasekara2019learning}, or optimizing one while imposing constraints on the rest~\citep{agrawal2016linear,kim2023improved}.
While these approaches identify only one action on the Pareto front,  
we identify all actions on the Pareto front, i.e. identify the set of actions that are potentially optimal for any scalarization approach. 

Table~\ref{tab:related_works} compares our contribution with the existing bandit literature.
The PFILin problem is a generalization of the BAI 
\citep{even2002pac,soare2014best} and single-objective regret minimization \citep{auer2002using,valko2013finite} to the multi-objective vector rewards.  
Existing algorithms for multi-objective PFI problems have focused on the Gaussian reward setting~\citep{zuluaga2016varepsilon} and non-contextual MAB setting~\citep{auer2016pareto}, and the optimal regret guarantees remain open.  
\citet{lu2019multi} proposed an algorithm that achieves a bound on regret for multi-objective contextual bandits; however the identification of \emph{all} arms in the Pareto front is not established. 
While \citet{degenne2019bridging} and \citet{zhong2023achieving} obtained theoretical guarantees for both regret and sample complexity for non-contextual single-objective rewards, extension to linear and multi-objective rewards remains open. 



\section{Problem Formulation: Pareto Front Identification for Linear Bandits}
\label{sec:PFI} 
For a positive integer $N$, let $[N]:=\{1,\ldots,N\}$. 
In PFILin, an action $k \in [K]$ is associated with a known
$d$-dimensional context vector $x_{k}\in\Real^{d}$. Let
$\mX:=\{x_{1},\ldots,x_{K}\}$.  
Without loss of generality, we assume $\|x_{k}\|_{2}\le1$ and, as is standard in this literature (e.g., \citet{tao2018best}), we assume that
$\mX$ spans $\Real^d$. 

In period $t$, the decision-maker chooses an $\Action t\in[K]$, and observes a sample of the random reward vector $\Vecreward{\Action{t}}{t}
=\Theta_{\star}^{\top}x_{a_{t}}+\eta_{t}$, where $\Theta_{\star}:=(\Parameter 1,\ldots,\Parameter L)\in\Real^{d\times L}$ is the unknown (but fixed) parameters with $\|\Parameter \ell\|_{2}\le \theta_{\max}$, for all $\ell \in [L]$, and $\eta_{t}\in\Real^{L}$ is a mean-zero, $\sigma$-sub-Gaussian random error vector that is independent of actions $\{\Action s\}_{s\in[t-1]}$, and other error vectors $\{\eta_{s}\in\Real^{L}:s\neq t\}$; however, we allow for the $L$ components of $\eta_{t}$ to be correlated.
Let $y_{k}:=\Theta_{\star}^{\top}x_{k}=\Expectation[\Vecreward kt] \in \Real^L$ denote the true mean reward vector for arm $k\in[K]$.  
We want to identify the Pareto front of the $\{y_{k}\}_{k\in[K]}$ defined as follows.

\begin{defn}[Pareto Front]
For vectors $a=(a^{(1)},\ldots,a^{(L)}$, $b=(b^{(1)},\ldots,b^{(L)}) \in\Real^{L}$, the vector $b$ \emph{dominates} $a$ (denoted by $a \prec b$) if $a^{(\ell)} \le b^{(\ell)}$, for all $\ell \in [L]$ and there exists $\ell \in [L]$ such that $a^{(\ell)} < b^{(\ell)}$.
The \emph{Pareto front} $\mP_{\star}:=\{ k\in[K]|\nexists k^{\prime}:y_{k}\prec y_{k^{\prime}}\}$ is a set of arms whose mean reward vector is not dominated by the reward of any other arm.  
\end{defn}
To identify the Pareto front $\mP_{\star}$, one must compute a reasonable estimate for the entire set of reward vectors $\{y_{k}\}_{k\in[K]}$.   
Following~\citet{auer2016pareto}, let
\[
  m(k,j) =\min\big\{\alpha\ge0\big|\exists \ell \in [L]:\Ereward \ell k+\alpha\ge\Ereward \ell j\big\} 
  =\max\big\{0,\; \min_{\ell \in [L]}(\Ereward \ell j-\Ereward \ell k)\big\}.
\]
denote the amount by which
arm~$j$ dominates arm~$k$. 
We have $m(k,j) > 0$ if and only if $\Ereward{\ell}{k} <\Ereward{\ell}{j}$, for all $\ell \in [L]$.
Therefore, the distance
$\Delta_{k}^{\star}:=\max_{k_{\star}\in\mP_{\star}}m(k,k_{\star})$ denotes the
minimum amount by which each component of the reward vector $y_{k}$ must
be increased to ensure that action~$k$ is not dominated by any Pareto optimal
action $k_{\star} \in \mP_{\star}$. 
By definition, the distance $\Delta_{k_{\star}}^{\star}=0$ for all Pareto
optimal actions $k_{\star}\in\mP_{\star}$. 
Next, we define the PFI 
success condition. 
\begin{defn}[PFI success condition]
For precision $\epsilon>0$ and confidence $\delta\in(0,1)$, a
PFI~algorithm must output a set of arms $\mP\subseteq[K]$ such that, with
probability at least $1-\delta$, 
\begin{equation}
\mP_{\star}\subseteq\mP\;\text{ and
}\;\Delta_{k}^{\star}\le\epsilon,\text{ for all
}k\in\mP\setminus\mP_{\star}.\label{eq:PFI_condition} 
\end{equation}
\end{defn}
The first condition in \eqref{eq:PFI_condition} ensures that $\mP$
contains the Pareto optimal set $\mP_{\star}$, and the second condition guarantees that the
set $\mP$ only includes arms sufficiently close to the Pareto front.  
Let $\tau_{\epsilon,\delta}$ denote the number of samples required for an
algorithm to meet the success condition~\eqref{eq:PFI_condition}.  
Then the cumulative regret $R(\tau_{\epsilon,\delta})$ of an algorithm
until round $\tau_{\epsilon,\delta}$ is defined as  
\begin{equation}
R(\tau_{\epsilon,\delta}):=\sum_{t=1}^{\tau_{\epsilon,\delta}}\Delta_{\Action
  t}^{\star}:=\sum_{t=1}^{\tau_{\epsilon,\delta}}\max_{k^{\star}\in\mP_{\star}}m(\Action
t,k^{\star}), 
\label{eq:Pareto_regret}
\end{equation}
where $\Action t$ denotes the action selected by the algorithm. 
Our goal is to \emph{simultaneously} establish an upper bound of the sample complexity $\tau_{\epsilon,\delta}$ and the Pareto regret
$R(\tau_{\epsilon,\delta})$. 

For a Pareto sub-optimal arm $k$, if the estimate $\hat{y}_k$ of the reward vector of arm $k$ has error  $\min_{\ell \in [L]}|\hat{y}_k^{\langle \ell \rangle}
- y_k^{\langle \ell \rangle}| > \Delta_{k}^{\star}$, it can erroneously appear Pareto optimal.  
Therefore, the required accuracy for a suboptimal arm $k \not\in \mP_{\star}$ is $\Delta_{k}=\Delta_{k}^{\star}$.  
Since the number of arms on the Pareto front is unknown, the algorithm must decide whether the remaining arms are all Pareto optimal or not to terminate.
Thus, we need another complexity measure,
\[
M(k,j) :=\min\big\{\alpha\ge0\big|\forall \ell \in [L]:\Ereward lk\le\Ereward lj+\alpha\big\}
=\max\big\{0, \;\max_{\ell \in [L]}\big(\Ereward lk-\Ereward lj)\big\},
\]
which is the amount by which each component of the mean reward of arm $j$ must be increased so that $k$ is weakly dominated by $j$.  
Note that $M(k,j)=0$ if and only if $\ErewardVec k\preceq\ErewardVec j$.
Fix a Pareto optimal arm $k$.
If the reward for arm $k$ is underestimated by $M(k,j)$ with respect to a Pareto optimal arm~$j$, it may appear weakly dominated by $j$.
Thus, in order to prevent misidentifying the Pareto optimal arm $k$ as a suboptimal arm, the error of the estimator has to be at most $\Delta_{k}^{+}:=\min_{j\in\mP_{\star}\backslash\{k\}}\min\left\{M(k,j),M(j,k)\right\}$. 
Next, consider a suboptimal arm $j$.
If the error of the estimator is greater than $M(j,k)+\Delta^{\star}_{j}$,
the Pareto optimal arm $k$ may appear dominated by suboptimal arm $j$.  
In order to distinguish the Pareto optimal arm $k$ from the Pareto suboptimal arms, the error of the estimator has be to at most $\Delta^{-}_{k}:=\min_{j\notin\mP_{\star}}\{M(j,k)+\Delta^{\star}_{j}\}$. 
In summary, to identify whether arm $k$ is in Pareto front, the estimation error has to be at most
\begin{equation}
\Delta_{k}:=\Delta_{k}^{\star},\;k\notin\mP_{\star},\text{ and }\Delta_{k}:=\min\{\Delta_{k}^{+},\Delta_{k}^{-}\},\;k\in\mP_{\star}.
\label{eq:est_acc}
\end{equation}
We index the arms in increasing order of required accuracy, i.e. $\Delta_{(1)} \le \cdots \le \Delta_{(K)}$.

\begin{thm}[A lower bound of the sample complexity for PFILin.]
\label{thm:lower_bound}
Fix $\epsilon>0$, and let 
$\Delta_{(k),\epsilon}:=\max\{\Delta_{(k)},\epsilon\}$.
Suppose the set of context vectors $\mX$ spans $\Real^{d}$ and $\|\Parameter{\ell}\|_{0} = d$, for all $\ell \in [L]$. 
Then, for any $\delta\in(0,1/4)$ and $\sigma>0$, there exist a $\sigma$-Gaussian distribution for the i.i.d. noise sequence $\{\eta_{t}\}_{t\ge1}$ such that any algorithm requires at least
$(\sigma^{2}/3)\sum_{k=1}^{d}\Delta_{(k),\epsilon}^{-2} \log(3L/4\delta)$
rounds to meet the success condition~\eqref{eq:PFI_condition}. 
\end{thm}
Theorem~\ref{thm:lower_bound} generalizes the lower bound in \citet{auer2016pareto} to the linear bandit setting.
Since $\Parameter{\ell} \in \Real^d$, the number of rounds required for PFI depends only on the $d$ smallest gaps instead of all $K$ gaps.

\section{Estimating Rewards with Low-Regret Actions}
\label{sec:estimation}
\textbf{Overview.} Our main contribution is a   novel estimation strategy
that simultaneously learns rewards of \emph{all} actions while largely
exploiting low-regret arms. We address the following two main challenges: 
(i) the number of arms $K$ can be exponentially large; 
(ii) exploiting the low-regret actions may not 
yield the
information required to learn the rewards of unexploited arms.
We resolve (i) by reducing $K$ context vectors into $d$ basis vectors
(Section~\ref{sec:dimension_reduction}); and (ii) by reusing the reward
samples in the exploration phase (Section~\ref{sec:mix}) along with
doubly-robust estimation (Section~\ref{sec:DR}) to compensate for
dependencies that arise from the data  ``reuse"  (imputation) scheme.  
Our strategy is applicable to more broadly to online learning problems 
under linear bandit feedback, e.g., BAI~\citep{tao2018best}, policy
optimization in reinforcement learning~\citep{he2021uniform}. We now
describe in more detail each ingredient in our approach and its
theoretical properties.  

\subsection{Exploration Strategy with Context Basis}
\label{sec:dimension_reduction}
Let $X=[x_1,\ldots,x_K]\in \Real^{d\times K}$ denote the matrix of contexts vectors. 
Using the (reduced) singular value decomposition (SVD), one can 
compute
orthonormal vectors $\{u_i\in\Real^d:i\in[d]\},
\{v_i\in\Real^{K}:i\in[d]\}$ and scalars $\{\lambda_i \geq 0: i = 1,
\ldots, d\}$ such that $X = \sum_{i=1}^d \lambda_i u_i v_i^\top$. 
Thus, it follows that $v_i^\top X^\top \Parameter{\ell} = \sqrt{\lambda_i}
u_i^\top \Parameter{\ell}$, for $\ell \in [L]$ and $i \in [d]$. 
For each $i\in[d]$, let $\pi^{(i)} \in \Real^{K}_+$ denote the probability mass function 
$\pi^{(i)}_k = |v_{ik}|/\|v_i\|_1$ over actions $k\in[K]$.
Then, for a randomized action $a \sim \pi^{(i)}$, we have
\vspace{-5pt}
\begin{equation}
\Expectation\left[\|v_{i}\|_{1}\sign{v_{ia}}\Reward{\ell}{a}{s}\right]
= \Expectation \Big[\sum_{k=1}^Kv_{ik}\Reward{\ell}{k}{s}\Big] =  \sum_{k=1}^{K}v_{ik}x_{k}^{\top}\Parameter{\ell}\\
=v_{i}^{\top}X^{\top}\Parameter{\ell}=(\sqrt{\lambda_{i}}u_{i})^{\top}\Parameter{\ell}, 
\label{eq:linear_basis}
\end{equation}
Thus, $\sum_{k=1}^Kv_{ik}\Reward{\ell}{k}{s}$ 
can be viewed as the random reward
corresponding to the ``context basis'' $\sqrt{\lambda_i} u_i$.  Note that
there is no ``pure'' action that corresponds to the 
``context basis'' -- it corresponds to a randomized mixture of actions. We
will combine these randomized actions with pure actions to efficiently
learn the parameter $\Theta_{\star}$. 
Because $XX^\top = \sum_{k=1}^{K} x_k x_k^\top = \sum_{i=1}^{d} \lambda_i
u_i u_i^\top$, sampling $a_i \sim \pi^{(i)}$ for $i\sim\text{unif}[d]$
yields the design matrix $d^{-1}\sum_{k=1}^{K} x_k x_k^\top$ that
satisfies $\max_{k\in[K]} \|x_k\|^2_{(d^{-1}\sum_{k^\prime=1}^{K}
  x_{k^\prime} x_{k^\prime}^\top)^{-1}} \le d$ (see
Section~\ref{sec:optimal_design} for details).  

For each $t\ge1$, let $\mE_t \subset [t]$ denote the set of rounds reserved for exploration.   
Fix a confidence level $\delta\in(0,1)$, and let $\gamma_t:=
C d^{3}\log(2dt^{2}/\delta)$ where $C$ is an absolute constant specified in~\eqref{eq:C}. 
Define $\mE_0:=\emptyset$, and in each round $t\ge1$, sample a basis index
$i_t \sim \text{unif}[d]$ and sample the action $\check{a}_{t} \sim
\pi^{(i_t)}$ according to corresponding probability mass function. 
Define 
\begin{equation}
\mE_{t}:=\mE_{t-1}\cup\{t\}\text{ if
} \sum_{u\in\mE_{t}}\Indicator{a_{u}=\check{a}_{t}}
\le\frac{\gamma_{t}}{t}\sum_{s=1}^{t}\Indicator{\check{a}_{s}=\check{a}_{t}},    
\qquad 
\mE_{t}:=\mE_{t-1}\text{, otherwise}. 
\label{eq:exploration_set}
\end{equation}
This definition is to ensure that the number of actions in $\mE_t$ increases logarithmically in $t$, and ensures that $\check{a}_t \sim \pi^{(i_t)}$, $i_t \sim \text{unif}[d]$. 
By construction, $\sum_{u\in\mE_{t}} \mathbb{I}(a_u=\check{a}_{t}) \ge
(\gamma_t/t) \sum_{s=1}^{t} \mathbb{I}(\check{a}_{s}=\check{a}_{t})$ for
$t \geq T_\gamma:=8C 
d^3\big(1+\log\frac{4C d^3\sqrt{2d}}{e\sqrt{\delta}}\big) \geq
\inf\{s\ge1:s\ge \gamma_s\}$. 
When $t\in\mE_{t}$, i.e, in an exploration round, the algorithm selects
the sampled action $a_t=\check{a}_{t}$, and when $t\notin\mE_{t}$, the
algorithm choose an arm from the set of unidentified arms that has low
estimated regret.  \subsection{Recycling Reward Samples in the Exploration Phase}
\label{sec:mix}
For each $\ell \in [L]$ and $t\in[t]\setminus\mE_t$, i.e., when $t$ is in
\emph{exploitation} phase, let $a_t$ denote the arm chosen by the
algorithm. 
In order to learn rewards of \emph{multiple} arms, we ``recycle" the
reward sample observed in a previous exploration round by bootstrapping as
follows. 
Recall that at the beginning of each round $t$, the context basis index  $i_t \sim
\text{unif}([d])$, and $\check{a}_t \sim \pi^{(i_t)}$. 
Let $\mE_t(\check{a}_t) = 
\{s\in\mE_{t}:a_{s}=\check{a}_t\}$ denote the set of previous exploration
rounds where the action $\check{a}_t$ was chosen. 
Note that, by definition of $\mE_t$ in~\eqref{eq:exploration_set}, we are
guaranteed that
$\mE_t(\check{a}_t) \neq \emptyset$.
Let $\check{n}_{\tau} \in \mE_{\tau}$ denote time index of the exploration sample ``mixed'' with
the action $a_{\tau}$ chosen in the exploitation round 
$\tau \in [t-1] \setminus \mE_t$.
We ``mix'' the action $a_t$ with the exploration sample in round 
\begin{equation}
  \check{n}_t:=\argmin_{n \in
  \mE_t(\check{a}_t)} \sum_{\tau \in [t-1] \setminus \mE_{t}}\mathbb{I}\big\{\check{n}_{\tau}=n \big\},
\label{eq:mix_number}
\end{equation}
i.e. 
we want to balance the  
reuse choice over the set $\mE_t(\check{a}^{(i_t)}_t)$. 
We define the exploration-mixed contexts and rewards as follows: 
\begin{equation}
\Tx{\Action t}t:=w_{t}x_{\Action t}+\check{w}_{t}\sqrt{\lambda_{i_{t}}}u_{i_{t}},
\quad
\Ty{\ell}{\Action t}t:=w_{t}\Reward{\ell}{\Action
  t}t+\check{w}_{t}\|v_{i_t}\|\sign{v_{i_t,a_{\check{n}_t}}}\Reward{\ell}{a_{\check{n}_{t}}}{\check{n}_{t}},
\forall\ell\in[L],  
\label{eq:mixup_data}
\end{equation}
where $w_{t},\check{w}_{t} \sim \text{unif}[-\sqrt{3},\sqrt{3}]$ are sampled independently.
The following properties follow from the linear structure and the
distribution of weights $(w_t, \check{w}_t)$ that has mean zero and unit
variance, and the definition of the reduced SVD $(u_i,v_i,\lambda_i)$, $i \in [d]$.
\begin{lem}
\label{lem:mixup}
(Exploration-mixed contexts and rewards.) 
Let $\mathcal{F}_{t}$ denote the sigma-algebra generated by $\{\Action{s}, \Vecreward{\Action
  s}s\}_{s=1}^{t-1}$ and $\Action t \in [K]$. 
For any $\ell \in [L]$ and $t$ such that $t \notin\mE_t$, 
$\Expectation[\Ty \ell{\Action t}t-\Tx{\Action t}t^{\top}\Parameter \ell|\mathcal{F}_{t}]=0$, and
$\Expectation[\Tx{\Action t}t\Tx{\Action t}t^{\top}|\Filtration t]\succeq
d^{-1}\sum_{k=1}^{K} x_{k}x_{k}^{\top}$.
\end{lem}
We can view $\Tx{\Action t}{t}, \Ty{\ell}{\Action t} {t}$ as a stochastic
feedback from a new linear bandit problem with the same parameters
$\{\Parameter{\ell}\}_{\ell \in [L]}$. 
Since the random contexts $\Tx{\Action t}{t}$ contains the (randomized)
context basis, the (expected) design matrix includes information on all
$K$ arms for any selected action $a_t$. 
``Recycling'' the reward sample
$\Reward{\ell}{a_{\check{n}_t}}{\check{n}_t}$ allows us to get information
on the rewards of the unselected (and hence unobserved) contexts while
exploiting low regret action. 
Next, we define the \textbf{exploration-mixed estimator},
\begin{equation}
\Impute t \ell\!:= \big(\sum_{s\in\mE_{t}}x_{\Action s}x_{\Action
  s}\!\!+\!\!\sum_{s\in[t]\setminus\mE_{t}}\Tx{\Action s}s\Tx{\Action
  s}s^{\top}\!+\! (1/2)I_{d}\big)^{-1}\big(\sum_{s\in\mE_{t}}x_{\Action
  s}\Reward l{\Action s}s\!+\!\sum_{s\in[t]\setminus\mE_{t}}\Tx{\Action
  s}s\Ty l{\Action s}s\big). 
\label{eq:mixup}
\end{equation}
While the exploration-mixed estimator gains information on the unknown
parameter on \emph{multiple contexts}, reusing samples from previous
rounds causes dependency that complicates the analysis of the convergence
rate of the estimator (see Section~\ref{sec:reuse_error} for details). 
To address this, we apply the doubly-robust (DR) technique from the
missing data literature instead of directly using the exploration-mixed
estimator, as we explain next. 

\subsection{Doubly-Robust Estimation}
\label{sec:DR}
Doubly-robust estimation uses an estimate to impute the
missing value, and is robust to the estimation error for the missing value. 
In each round $t$, the unselected rewards
$\{\Reward{\ell}{k}{t}:\ell\in[L],k\in[K]\setminus\{a_t\}\}$ are missing. 
One possible approach is computing a ridge estimator $\theta_{R}^{\langle
  \ell \rangle}$ and imputing $x_k^\top \theta_{R}^{\langle \ell \rangle}$
for $\Reward{\ell}{k}{t}$ to apply doubly-robust estimation, as proposed
in~\citet{kim2021doubly,kim2022double,kim2024doubly}.  
However, their approach assumes stochastic contexts that are IID over
rounds with finite $K$, and therefore, not applicable to PFILin where the
contexts are fixed and $K$ can be exponentially large.   
Further, since the ridge estimator only gains information on the selected
actions, their DR estimator does not converge while exploiting low regret
arms (See Appendix~\ref{sec:impute} for detailed comparisons.)  

We first reduce $K$ rewards into $d+1$ rewards
using~\eqref{eq:linear_basis}: $\Newy{\ell}i
t:=\sum_{k=1}^{K}v_{ik}\Reward{\ell}kt$, corresponding to the $d$ context
basis $\sqrt{\lambda_i} u_i$, $i=1,\ldots,d$, and
$\Newy{\ell}{d+1}t:=\Reward{\ell}{a_t}t$.  
Note that $\Expectation[\Newy{\ell}it]=(\sqrt{\lambda_i} u_i )^\top
\Parameter \ell$ for $i\in[d]$, and we learn $\Theta_{\star}$ using $d$
context basis vectors $\{\sqrt{\lambda_i} u_i$, $i \in [d]\}$, instead of $K$ contexts. 
We view $\{\Newy{\ell}i t: i\in[d]\}$ as missing data, and only
$\Newy{\ell}{d+1}{t}$ is observed. 
To induce a specified probability for missing data (needed to ensure
robustness of the DR estimation),  we define a probability mass
function $\tilde{\pi}$ defined as follows: 
\begin{equation}
\tilde{\pi}_i = 1/(2d),\ \forall i=1,\ldots,d, 
\qquad 
\tilde{\pi}_{d+1} = 1/2,  
\label{eq:tilde_distribution}
\end{equation}
and let $\tilde{a}_t \sim \tilde{\pi}$ denote the pseudo-action on $d+1$
arms.

To couple the observed reward $\Reward{\ell}{a_t}{t}$ and the randomly
selected reward $\Newy{\ell}{\tilde{a}_t}{t}$, we resample both action
$a_t$ and pseudo-action $\tilde{a}_t$ until the matching event
$\{\Reward{\ell}{a_t}{t} = \Newy{\ell}{\tilde{a}_t}{t}\}$ happens. 
For given $\delta^\prime\in(0,1)$, let $\mathcal{M}_{t}$ denote the event
of obtaining the matching $\{\Reward{\ell}{a_t}{t} =
\Newy{\ell}{\tilde{a}_t}{t}\}=\{\tilde{a}_t=d+1\}$ within
$\rho_{t}:=\log((t+1)^{2}/\delta^\prime)/\log(2)$ number of resampling so
that the event $\mathcal{M}_{t}$ happens with probability at least
$1-\delta^\prime/(t+1)^{2}$. 
If the event $\mathcal{M}_t$ does not happen, we do not update the estimator (and use the estimator value obtained in the previous round).
Define new contexts $\tilde{x}_{i,t}:=\sqrt{\lambda_i}u_i$ for $i\in[d]$ and $\tilde{x}_{d+1,t}=x_{a_t}$.
With the coupled pseudo-action $\tilde{a}^{(t)}$ and its distribution $\tilde{\pi}$, we construct the DR estimate for the reduced missing rewards as:
\begin{equation}
\PseudoY \ell it:=\tilde{x}_{i,t}^{\top}\Impute t\ell+\frac{\mathbb{I}(\PseudoAction t=i)}{\tilde{\pi}_{i}}(\Newy{\ell}it-\tilde{x}_{i,t}^{\top}\Impute t\ell),\:\forall i=1,\ldots,d+1.
\label{eq:pseudoY}
\end{equation}
For $i \neq \PseudoAction t$, we impute a reward
$\tilde{x}_{i,t}^{\top}\Impute{t}{\ell}$ for the new ``context'' basis.  
For $i=\PseudoAction t$, the second term $(\Reward \ell{\PseudoAction
  t}t-\tilde{x}_{\tilde{a}_t,t}^{\top}\Impute t
\ell)/\tilde{\pi}_{\tilde{a}_t}$ 
corrects
the 
imputed 
reward to 
ensure
unbiasedness of the pseudo-rewards for all arms. 
Taking the expectation over $\PseudoAction t$ on both sides
of~\eqref{eq:pseudoY} gives $\Expectation[\PseudoY \ell i
t]=\tilde{x}_{i,t}^\top \Parameter \ell$ for all $i\in [d+1]$.   
Then, our proposed \textbf{DR-mix estimator} is 
\begin{equation}
\Estimator \ell t=\big(\sum_{s=1}^{t}\Indicator{\mathcal{M}_{s}}\sum_{i=1}^{d+1}\tilde{x}_{i,s}\tilde{x}_{i,s}^{\top}+I_{d}\big)^{-1}\sum_{s=1}^{t}\Indicator{\mathcal{M}_{s}}\sum_{i=1}^{d+1}\tilde{x}_{i,s}\PseudoY \ell is.
\label{eq:estimator}
\end{equation}
The estimator~\eqref{eq:estimator} is recursively computable with a rank-1 update of the Gram matrix and summation of  the weighted context vectors.

\begin{thm}[Estimation error bound for the DR-mix estimator]
\label{thm:self} 
Let $\Estimator \ell t$ denote the DR-mix estimator~\eqref{eq:estimator} with the exploration-mixed estimator~\eqref{eq:mixup} as the imputation
estimator and pseudo-rewards~\eqref{eq:pseudoY}. 
Let $F_{t}:=\sum_{k=1}^{K}x_{k}x_{k}^{\top}+I_{d}$,
Then, for all $k\in[K]$, $\ell \in [L]$, and $t\ge T_{\gamma}$,
\begin{equation}
\big|x_k^{\top}(\Estimator \ell t-\Parameter \ell)\big|\le3\| x_k\|_{F_{t}^{-1}}\{\theta_{\max}+\sigma\sqrt{d\log(Lt/\delta)}\}.
\end{equation}
with probability at least $1-7\delta$.
For each $k\in[K]$, with probability at least $1-7\delta$,
\begin{equation}
\big|x_k^{\top}(\Estimator \ell t-\Parameter \ell)\big|\le3\| x_k\|_{F_{t}^{-1}} \big\{ \theta_{\max}+3\sigma\sqrt{\log(4Lt^{2}/\delta)}\big\}.
\end{equation}
\end{thm}
In early rounds, we have $\Omega(K)$ number of undetermined arms to estimate and we use the union bound~\eqref{eq:est_bound_union} to avoid the dependency on $K$. 
After eliminating the suboptimal arms and when the number of undetermined arms is $O(d)$, we can use the tighter bound~\eqref{eq:est_bound}. 
Because the contexts are normalized by the Gram matrix $F_{t} \succeq
\sum_{k=1}^{K} x_k x_k^\top + I_d$, we obtain
$\max_{k\in[K]}\|x_k\|_{F_t^{-1}} \le t^{-1/2}$ (derived in
Section~\ref{sec:optimal_design}). 
With only $\gamma_t=O(d^3\log dt^2)$ exploration rounds, we obtain
$\tilde{O}(\sqrt{d/t})$ convergence rate for the reward estimates of
\emph{all} arms. 
This is possible because the DR-mixed estimator gains information on all
arms through the exploration-mixed estimate and is robust to the error of
the exploration-mixed estimator caused by the dependency of reusing the
samples (for details see Section~\ref{sec:DR_robustness}). 
Therefore, our estimator enjoys the freedom to choose low-regret arms
while simultaneously  learning the rewards on all arms.

\section{Algorithm for Pareto Front Identification with Regret Minimization}
In this section, we apply our novel estimation strategy to PFILin and
establish novel algorithm with nearly optimal sample complexity and
regret. 

\subsection{\PFIwR\ Algorithm for Linear Contextual Bandits}
Our proposed algorithm, PFI with regret minimization~(\texttt{PFIwR}), is displayed in Algorithm~\ref{alg:PFIwR}.
While any undetermined arms remains, the algorithm employs our novel
estimation strategy to compute the reward estimates $\Hreward \ell
kt:=\Context k^{\top} \Estimator \ell t$ for $k\in\mA_t$ and gap estimates
required for PFI.  
The algorithm uses two different confidence bound
in~\eqref{eq:confidence_bound} based on the two convergence rates in
Theorem~\ref{thm:self}. 
The first bound uses the union bound~\eqref{eq:est_bound_union} because
the estimator must converge on all $K$ arms. 
However, when the number of undetermined arms are less than $d$, we need
at most $d$ reward estimate for $k\in\mA_t$ and at most $d$ estimates for
the arm in $[K]\setminus \mA_t$ that are ``nearest'' to $\mA_t$ and
critically affect the PFI. 
Specifically, we only need the reward estimate for arm $k^{\star}$ such
that $\Delta^{\star}_k = m(k,k^\star)$ for suboptimal arm $k\in\mA_t
\setminus \mP_{\star}$ and for arm either $j\in[K]\setminus \mP_{\star}$
such that $\Delta_k^{-}=M(j,k)+\Delta^{\star}_j$ or
$k_{\star}\in\mP_{\star}$ such that
$\Delta_k^{+}=\min\{M(k,k_{\star}),M(k_{\star},k)\}$ for each optimal arm
$k\in\mP_\star$. 

\begin{algorithm}[t]
\footnotesize
\caption{Pareto Front Identification with Regret Minimization
  (\texttt{PFIwR}).}
\label{alg:PFIwR}
\begin{algorithmic}[1]
\State \textbf{INPUT:} context matrix $X=[x_1,\ldots,x_K]$, accuracy parameter $\epsilon>0$, confidence $\delta>0$.
\State Set $\mA_{0}=[K]$,
$\mP_{0}=\mE_{0}=\emptyset$ and $\Estimator l0=\mathbf{0}, \forall \ell \in [L]$ and apply reduced SVD on $X=UDV^\top$.
\While{$\mA_{t}\neq\emptyset$}
\State Sample $i_t\sim\text{unif}([d])$ and $\check{a}_{i_t} \sim \pi^{(i_t)}$ and update $\mE_t$ as in~\eqref{eq:exploration_set}.
\State \textbf{if} $t \in \mE_{t}$: set $a_t=\check{a}_{i_t}$ \textbf{else:} randomly sample $\Action t$ over $\{k\in \mA_{t-1}:\nexists
k^{\prime}\in\mA_{t-1},\widehat{y}_{k,t}\prec\widehat{y}_{k^{\prime},t}\}$. 
\State Sample $\PseudoAction{t}$ from the
distribution~\eqref{eq:tilde_distribution}, and set $m=1$ 
\While{$\PseudoAction{t} \neq d+1$ and $m\le\rho_t$} 
\State Resample $a_t$ and $\PseudoAction{t}$, and set $m \leftarrow m+1$.
\EndWhile
\State Select action $a_t$ and observe $\Reward{\ell}{a_t}{t}$ for $\ell\in[L]$.
\If{$m \le \rho_t$, (i.e., the resampling succeeds)}
\State Update the estimators $\{\Estimator \ell t\}_{\ell \in [L]}$ as in~\eqref{eq:estimator} then compute $\Hreward \ell kt:=\Context k^{\top}\Estimator \ell t$ for $k\in\mA_t$~and,
\vspace{-5pt}
\begin{equation}
\begin{split}
  \widehat{m}_{t}(k,k^{\prime}):=
  & \min\{ \alpha\ge0\big|\exists
    \ell \in [L]:\Hreward
    \ell kt+\alpha\ge\Hreward
    \ell {k^{\prime}}t\},\\ 
  \widehat{M}_{t}^{2\epsilon}(k,k^{\prime}):=
  & \min\{ \alpha\ge0\big|\forall \ell \in [L]:\Hreward
    \ell kt+2\epsilon\le\Hreward \ell{k^{\prime}}t\!+\!\alpha\}, 
\end{split}
\quad \forall k,k^\prime \in \mA_t.
\label{eq:gap_estimate}
\end{equation}
\State Compute confidence intervals,
\vspace{-5pt}
\begin{equation}
\beta_{k,t}:=
\begin{cases}
3\|x_{k}\|_{F_{t}^{-1}}\{\theta_{\max}+\sigma\sqrt{d\log(7Lt/\delta)}\} & \abs{\mA_{t}}>d\\
3\|x_{k}\|_{F_{t}^{-1}}\{\theta_{\max}+3\sigma\sqrt{\log(28Ldt^{2}/\delta)}\} & \abs{\mA_{t}}\le d.
\end{cases}
\label{eq:confidence_bound}
\end{equation}
\State Estimate Pareto front,
\vspace{-5pt}
\begin{equation}
 \begin{split}
 \mC_{t}&\!:=\!\{ k\!\in\!\mA_{t-1} |\forall k^{\prime}\in\mA_{t-1}
 \!\cup\! \mP_{t-1}\!:\!\widehat{m}_{t}(k,k^{\prime})\!\le\!
 \beta_{k,t}+\beta_{k^\prime,t}\}, 
 \\
 \mP_{t}^{(1)}&\!:=\!\{k\in\mC_{t}|\forall
 k^{\prime}\in\mC_{t}\cup\mP_{t-1}\!\!
 \setminus\!\!\{k\}\!:\!\widehat{M}_{t}^{2\epsilon}(k,k^{\prime})\!\ge\!\beta_{k,t}+\beta{k^\prime,t}\}. 
\end{split}
\end{equation}
\State Update $\mP_{t}\leftarrow\mP_{t-1}\cup\mP^{(1)}_t$, $\mA_{t}\leftarrow\mC_{t}\backslash\mP^{(1)}_t$.
\Else
\State Update the estimator $\Estimator {\ell}{t} \leftarrow \Estimator{\ell}{t-1}$ for $\ell\in[L]$.
\EndIf
\EndWhile
\State \textbf{OUTPUT:} $\mP_{t}$
\end{algorithmic}
\end{algorithm}

\PFIwR\ computes the set $\mC_{t}$ by eliminating the suboptimal arms that are dominated by other arms by the amount more than the confidence bound $\beta_{k,t}+\beta_{k^{\prime},t}$. 
The set $\mP_{t}^{(1)}$ is the current estimate for $\epsilon$-Pareto optimal arms that are not dominated by any other arms.
This arm elimination step is simplified compared to that in~\citet{auer2016pareto}.
The algorithm in \citet{auer2016pareto} leaves the identified Pareto optimal arm undetermined until all suboptimal arms dominated by the identified arm are eliminated, in order to ensure that a dominated arm is not spuriously declared as Pareto optimal.
In contrast, \PFIwR\ does not keep the (identified) dominating arms in $\mA_t$ because the DR estimate converges on \emph{all} arms in $\mA_t \cup \mP_t$, including the identified arms, in contrast to the conventional estimator that does not converge on identified arms unless they are selected.
Consequently, the cardinality $|\mA_t|$ of the set undetermined arms decreases faster in \PFIwR\ ~(derived in Section~\ref{sec:elimination}), and this allows it to invoke the tighter confidence bound in~\eqref{eq:confidence_bound}, that are only available when $|\mA_t|\le d$, earlier.

In addition to efficient estimation for PFI, the proposed \PFIwR\ is able to choose low estimated regret actions after $O(d^3 \log t)$ exploration rounds.
The novel estimator and its convergence (Theorem~\ref{thm:self}) ensure that sampling arms with low estimated regret does not harm the convergence rate of the reward estimates of other arms. 
Thus, \PFIwR\ is efficient in both for PFI and minimizing regret.

\subsection{Sample Complexity and Regret Analysis}
\label{sec:sample_regret}

\begin{thm}[An upper bound on sample complexity]
\label{thm:sample} Fix $\epsilon>0$ and $\delta \in (0,1)$. 
Let $\Delta_{(k),\epsilon} = \max\{\Delta_{(k)},\epsilon\}$, where
$\Delta_{(k)}$ is the ordered gap $\Delta_{k}$ defined in~\eqref{eq:est_acc} in increasing order.
Then the stopping time $\tau_{\epsilon,\delta}$ of \texttt{PFIwR} is bounded above by 
\vspace{-5pt}
\[
\max\Big\{
O\Big( \sum_{k=1}^{d}
\frac{(\theta_{\max}+\sigma)^{2}}{\Delta_{(k),\epsilon}^{2}}
\log\frac{(\theta_{\max}+\sigma)dL}{\Delta_{(k),\epsilon}\delta}\Big),
T_{\gamma}\Big\}, 
\]
where $T_{\gamma}$ denotes an upper bound on the initial exploration rounds.
\end{thm}
The proof and explicit finite-sample bound for $\tau_{\epsilon,\delta}$ is in Appendix~\ref{subsec:sample_proof}. 
When the contexts are Euclidean basis, our result directly applies to the PFI in the MAB setting studied by \citet{auer2016pareto}. 
The sample complexity is optimal within a logarithm factor of the lower bound in Theorem~\ref{thm:lower_bound}. 

\begin{thm}[Upper bounds on Pareto regret]
\label{thm:regret} Fix $\epsilon>0$ and $\delta \in (0,1)$. Let
$\Delta^{\star}_{\epsilon}:=\max\{\epsilon,\;\min_{k\in[K]\setminus\mP_\star}\Delta_{k}^{\star}\}$ 
denote the minimum Pareto regret over suboptimal arms.  
Then, with probability at least $1-\delta$, the instantaneous Pareto regret,
$\Delta_{\Action t}^{\star}\le2\max_{j\in\mA_{t-1}}\beta_{j,t-1},$
for all $t\notin\mE_t$ and $t\ge T_\gamma$, where $\beta_{j,t}$ is the
error bound defined in \eqref{eq:confidence_bound}. 
The cumulative Pareto regret of \texttt{PFIwR},
\vspace{-2pt}
\[
R(\tau_{\epsilon,\delta})=\bar{O}\Big(\theta_{\max}d^3\log\frac{\theta_{\max}d\sigma}{\delta\Delta_{(1),\epsilon}}+\frac{\theta_{\max}d\sigma}{\Delta_{\epsilon}^{\star}}\log\frac{\theta_{\max}d\sigma}{\Delta_{\epsilon}^{\star}\delta}\Big)
\]
with probability at least $1-\delta$, where $\bar{O}$ ignores $\log\log(\cdot)$ terms. 
\end{thm}
The explicit expression for the finite-sample bound is in Appendix~\ref{subsec:regret_proof}. 
The first term is the regret from the exploration rounds $\mE_t$, whose cardinality $|\mE_t| = O(d^3 \log dt)$ for all $t$. 
Since the algorithm need to increase $\mE_t$ until it identifies all arms on the Pareto front, the bound involves $\Delta_{(1),\epsilon}$, which is the cost for identifying all arms in the Pareto front.
When $L=1$ and the contexts are Euclidean basis, Theorem~\ref{thm:sample} and Theorem~\ref{thm:regret} recovers the sample complexity bound and regret bound for the best arm identification in MAB setting established by \citet{degenne2019bridging} and \citet{zhong2023achieving}.


\begin{thm}[A regret lower bound for in PFILin]
\label{thm:regret_lower_bound}
For $\epsilon>0$, let $\Delta_{\epsilon}^{\star}:=\max\{\epsilon,
\min_{k\in[K]\setminus\mathcal{P}_{\star}}\Delta_k^{\star}\}$ denote the
minimum Pareto regret over suboptimal arms. 
Suppose the set of context vectors $\mathcal{X}$ span $\mathbb{R}^d$ and $\min_{\ell \in [L]}\|\theta_{\star}^{\langle \ell \rangle}\|_0 = d$.
Then, for any $\delta \in (0,1/4)$ and $\sigma>0$, there exists a
$\sigma$-sub-Gaussian distribution for the i.i.d. noise sequence
$\{\eta_t\}_{t\ge1}$ such that for any PFI algorithms that satisfies PFI
success condition (1) with failure probability $\delta$, 
\vspace{-2pt}
\[
  R(\tau_{\epsilon,\delta}) \ge
  \frac{\sqrt{3}d\sigma}{8\Delta_{\epsilon}^{\star}} \log
  \frac{1}{4\delta}. 
\]
\end{thm}

Theorem~\ref{thm:regret_lower_bound} shows that \PFIwR\ establishes nearly optimal regret among algorithms that achieve PFI and it is the first result on the trade-off between PFI and Pareto regret minimization. 
For $L=1$ and the contexts are Euclidean basis,  Theorem~\ref{thm:regret_lower_bound} recovers the lower bound for regret of BAI algorithms developed by~\citet{zhong2023achieving}. 
Note that the lower bound applies only to the algorithms that guarantee PFI; it is possible for an algorithm that does not guarantee PFI to have a regret lower bound that is lower than the one in Theorem~\ref{thm:regret_lower_bound}.

\section{Experiments}
\label{sec:experiments} 

\subsection{Consistency of the Proposed Estimator on All Actions}
\label{subsec:experiment_estimator}

We conduct the following experiment to empirically verify that our proposed DR-mix estimator
converges on \emph{all} arms while exploiting low-regret arms.
We consider a $3$-arm bandit, i.e. the context vectors are the Euclidean basis in $\Real^3$. 
The parameter $\Parameter{1} =(1,-1,-1)^{\top}$, and the random error is sampled from centered Gaussian distribution with variance $\sigma^2 = 0.01$.  
In rounds $n\le50$, each of three arms is pulled with equal probability; in rounds $n>50$, only the optimal arm~(arm 1) is pulled.

The plots in Figures~\ref{fig:err_opt} and~\ref{fig:err_subopt} illustrate the reward error of the proposed DR-mix estimator, the conventional ridge estimator, and an exploration-mixed estimator defined in~\eqref{eq:mixup} as a function of the number of rounds~$n$. 
The conventional ridge estimator converges only on the arm that is pulled (arm $1$), while the exploration-mixed estimator and the proposed DR-mix estimator converge for all arms, including arms $2$ and $3$ that are not observable in round $n > 50$.
While the exploration-mixed estimator converges as fast as the DR-mix estimator on arm $2$ and $3$, it converges slower on arm $1$; since the DR-mix estimator minimizes on all $d$ context basis while exploration-mixed estimator minimized only one context basis.
For further analysis on the estimators, see Section~\ref{sec:est_density}.

\begin{figure}[t]
\centering
\begin{subfigure}[b]{0.45\textwidth}
    \centering
    \includegraphics[width=\textwidth]{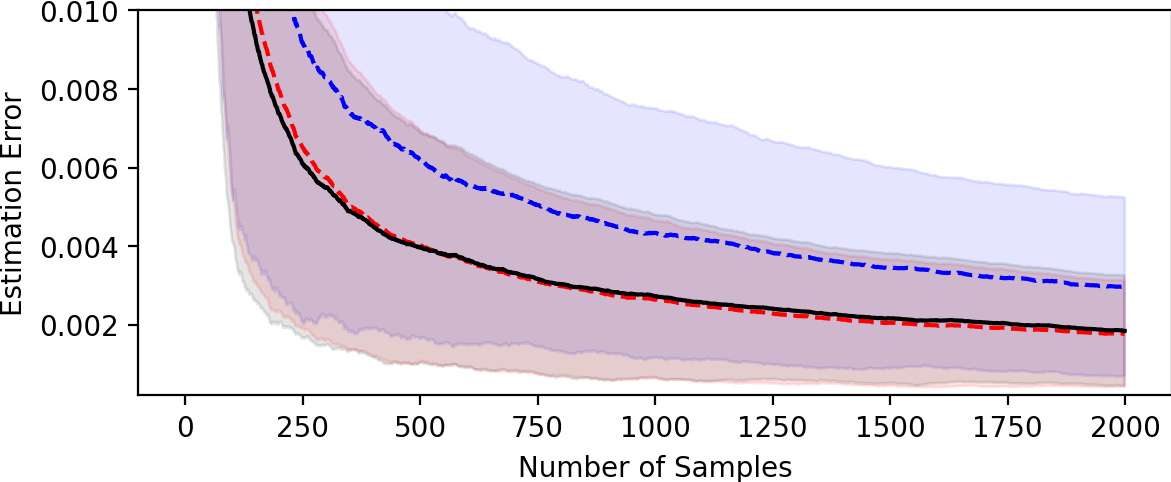}
    \vskip -5pt
    \caption{Error $|\widehat{\theta}_1 - \theta_{\star}^{(1)}|$ for the exploited arm (arm 1)}
    \label{fig:err_opt}
    \end{subfigure}
    \begin{subfigure}[b]{0.45\textwidth}
    \centering
    \includegraphics[width=\textwidth]{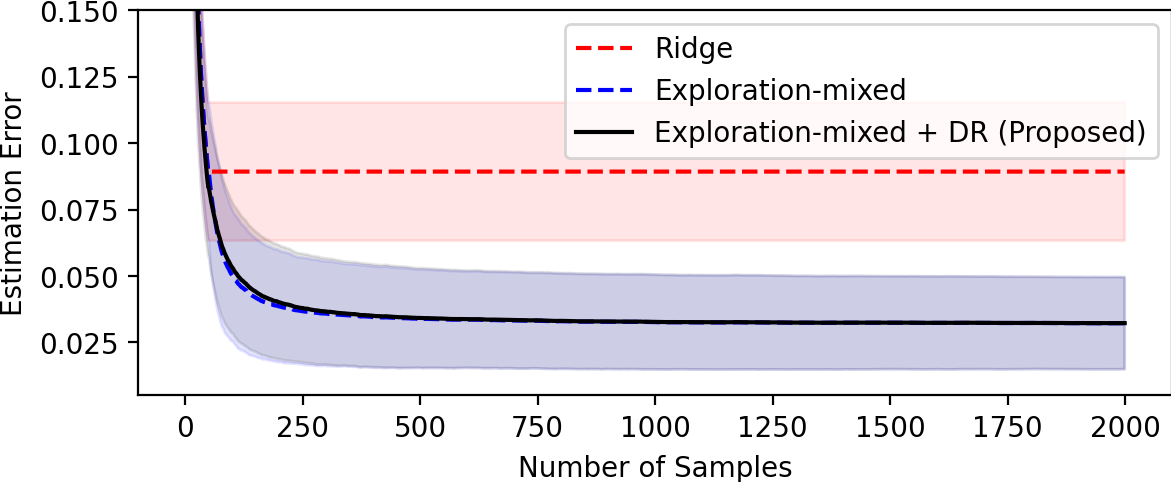}
    \vskip -5pt
    \caption{Error $\|\widehat{\theta}_{\{2,3\}} - \theta_{\star}^{(\{2,3\})}\|_2$ for arm 2 and 3}
    \label{fig:err_subopt}
    \end{subfigure}
    \caption{Estimation errors of the proposed DR-mix estimator~\eqref{eq:estimator} with the conventional ridge estimator, and the exploration-mixed estimator~\eqref{eq:mixup} for a 3-armed bandit problem.
    The line and shade represent the average and standard deviation over 1000 independent experiments.
    The estimators use samples from all arms for $n\in[50]$, and after that, only observe rewards from arm $1$.
    }
    \vspace{-10pt}
\end{figure}

\begin{figure*}[t]
\centering
\begin{subfigure}[b]{0.325\textwidth}
    \centering
    \includegraphics[width=\textwidth]{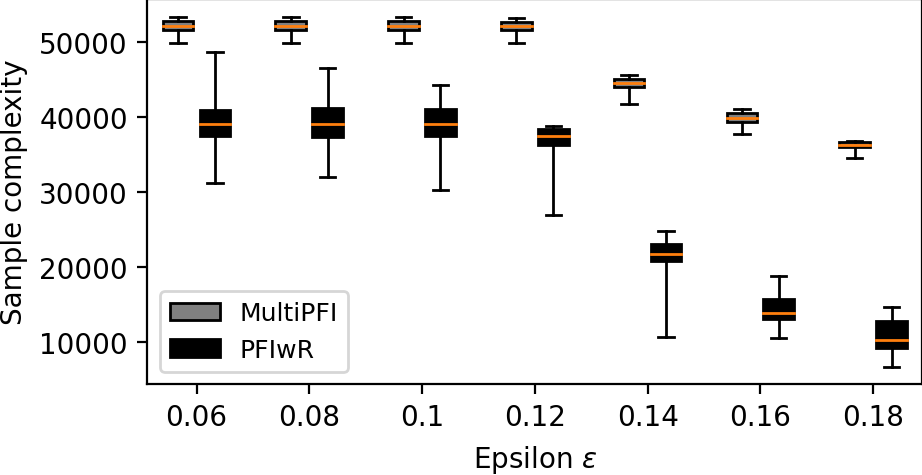}
    \vskip -5pt
    \caption{Number of samples for PFI}
    \label{fig:llvm_sample}
    \end{subfigure}
    \hfill
    \begin{subfigure}[b]{0.325\textwidth}
    \centering
    \includegraphics[width=\textwidth]{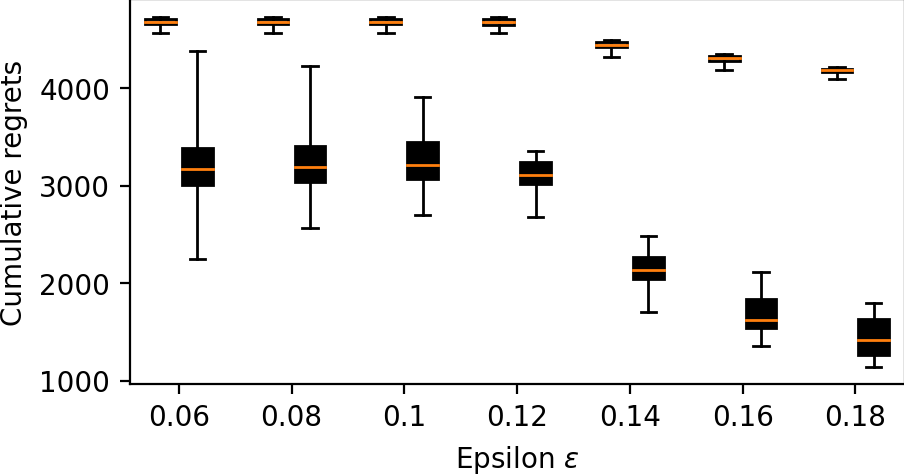}
    \vskip -5pt
    \caption{Total sum of regret at termination}
    \label{fig:llvm_regret}
    \end{subfigure}
    \hfill
    \begin{subfigure}[b]{0.325\textwidth}
    \centering
    \includegraphics[width=\textwidth]{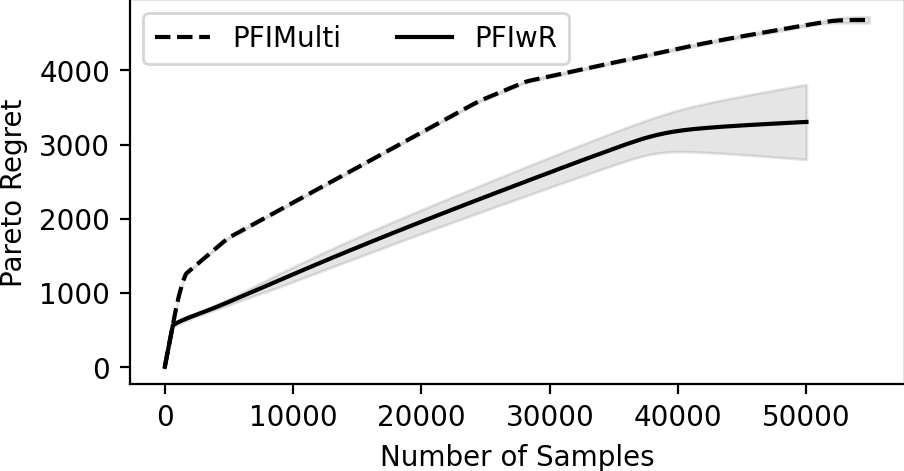}
    \caption{Cumulative regret ($\epsilon=0.06$)}
    \label{fig:llvm_regret_sample}
    \end{subfigure}
    \caption{Comparison of \texttt{PFIwR} and \texttt{MultiPFI} on the SW-LLVM dataset. 
    Both algorithms correctly identify the $\epsilon$-near Pareto optimal arms on all 500 independent experiments.
    } 
    \label{fig:comparison}
    \vskip -10pt
\end{figure*}

\subsection{Comparison of \texttt{MultiPFI} and \texttt{PFIwR}}

Next, we compare \PFIwR\ with \texttt{MultiPFI}~\citep{auer2016pareto} on the SW-LLVM dataset~\citep{zuluaga2016varepsilon} (see Section~\ref{sec:sw-llvm} for details).
Figure~\ref{fig:comparison} reports the performance of \texttt{PFIwR} and \texttt{MultiPFI} \citep{auer2016pareto} on various $\epsilon=\{0.06, 0.08, \ldots, 0.18\}$. 
Both algorithms use a fixed $\delta$ of $0.1$.
In Figure~\ref{fig:llvm_sample}, in most cases, \PFIwR\ uses fewer samples than \texttt{MultiPFI} to satisfy the success  condition~\eqref{eq:PFI_condition}. 
Even though the number of samples used by \PFIwR\ has a larger variance, in most cases, it uses fewer samples for PFI than \texttt{MultiPFI}. 
Figure~\ref{fig:llvm_regret} is a box plot of the cumulative Pareto regret of \texttt{PFIwR} and \texttt{MultiPFI} at the termination of the algorithm --  \PFIwR\ has significantly lower regret than \texttt{MultiPFI}.  
Figure~\ref{fig:llvm_regret_sample} display the cumulative Pareto regret of \texttt{PFIwR} and \texttt{MultiPFI} as a function of rounds $n$ when $\epsilon=0.06$.  
Since the number of rounds required for PFI and the horizon is random, to compute the average and standard deviation of the cumulative regret, we set the instantaneous regret to zero after the algorithm terminates in each experiment. 
The regret of \PFIwR\ increases slower than \texttt{MultiPFI} because it chooses actions that minimize regret in the exploitation phase while learning the rewards.  
The experiment demonstrates that our proposed \PFIwR\ achieves the dual goal of PFI and regret minimization.

\bibliography{ref}
\bibliographystyle{plainnat}


\appendix
\newpage

\section{Supplementary Materials for Experiments}

\subsection{SW-LLVM Dataset Description}
\label{sec:sw-llvm}
The SW-LLVM dataset \citep{zuluaga2016varepsilon} consists of $1024$ $2$-dimensional reward vectors. 
We normalized the reward vectors by subtracting the average and dividing by the standard deviation for each component.
We created a $16$-arm PFI problem using the methodology in~\citet{auer2016pareto}: we clustered the reward vectors into $16$
groups, with $64$ reward vectors in each group. 
We computed the mean reward $y_i \in \Real^2$ for the $i$-th cluster by taking the average over the $i$-th cluster, and when the algorithm selects an arm $i$ in any round, we randomly sample a reward vector from the $i$-th cluster.

\subsection{Consistency of the Doubly-Robust Estimator Without the Exploration-Mixed Estimator}
\label{sec:impute}
\begin{figure}[b]
\centering
    \includegraphics[width=0.80\textwidth]{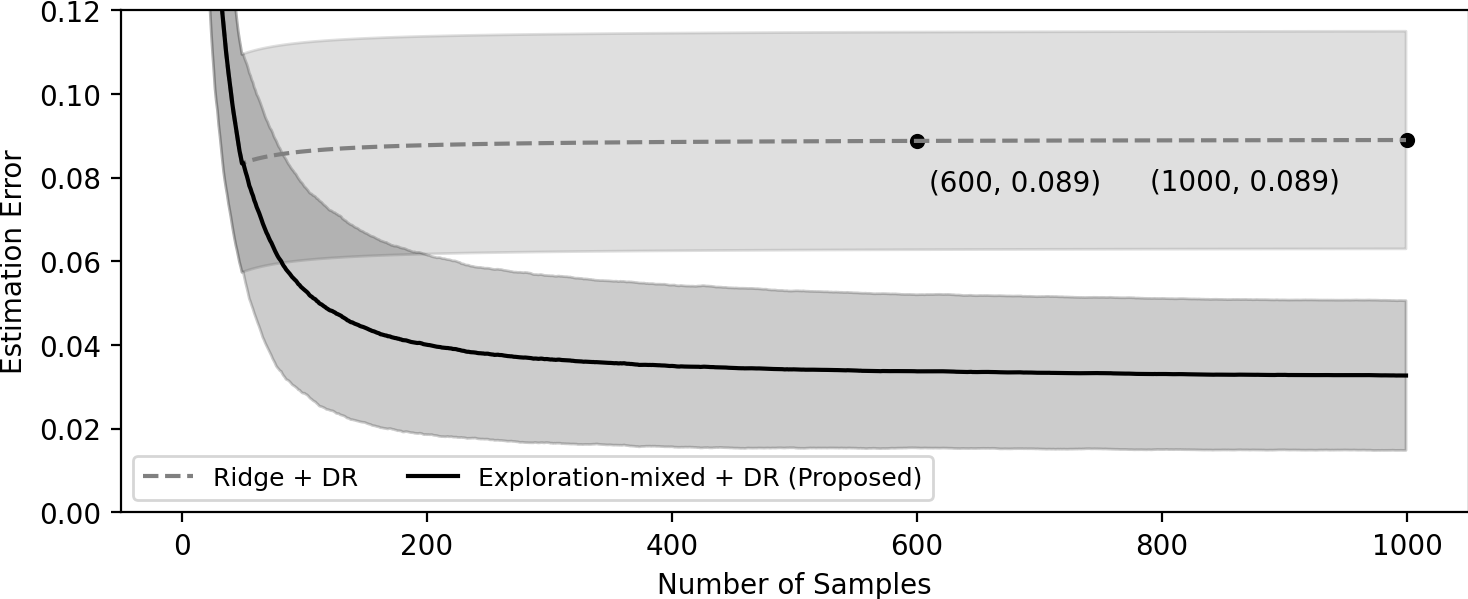}
    \caption{The $\ell_2$-error of the reward on the \textit{unexploited} arms (arms 2 and 3) of the proposed estimator and the DR estimator whose imputation estimator is the conventional ridge estimator in the 3-armed bandit problem (for detailed setting, see Section~\ref{subsec:experiment_estimator}.)
    The estimators use samples from all three arms when $t\le50$ and only arm 1 when $t>50$. When constructing a DR estimator, choosing the imputation estimator that learns rewards on all arms is crucial for convergence on all arms.
    }
\label{fig:DR_err}
\end{figure}

The convergence properties of the DR 
estimator $\Estimator \ell t$ critically depend on 
the imputation estimator $\Impute{t}{\ell}$ used in the pseudo-reward~\eqref{eq:pseudoY}.  
In Figure~\ref{fig:DR_err} we plot the error of the DR estimators with two different imputation estimators: ridge estimator and exploration-mixed estimator~\eqref{eq:mixup} as a function of the number of rounds for  a $3$-armed bandit problem, or equivalently, a linear bandit problem with
the set of context vectors $\mX$ given by the Euclidean basis.
Using the DR method with the ridge estimator does not guarantee convergence on all arms -- it only gets  information from the exploited arm that has no information about the rewards of the other arms. 
In contrast, the DR estimator $\Estimator{l}{t}$ with \eqref{eq:mixup} as an imputation
estimator, converges on all arms. 
This is possible  because ``mixing'' contexts and rewards as in~\eqref{eq:mixup_data} transforms the $3$-armed bandit data into a linear bandit with stochastic contexts that span $\Real^3$ with high probability.

\subsection{Comparison of the Density of the Estimators}
\label{sec:est_density}
The plots in Figure~\ref{fig:density} display the evolution of the density of estimates
of the three methods for arm~1 and arm~2 for $n = 50, 500, 2000$. 
For arm 1 (Figure~\ref{fig:density_opt}), the ridge estimator and the
proposed DR-mix estimator converge faster, i.e., have zero-mean with a lower
variance, compared to the exploration-mixed estimator.    
Since the exploration-mixed estimator~\eqref{eq:mixup} creates the context
and the associated reward by assigning random weights to the current
observation and one from a past exploration round, the reward estimate for
the selected arm becomes unstable. 
In contrast, the DR-mix estimator returns the focus to estimating the reward
of the selected arm and converges faster than the exploration-mixed
estimator. 
For arm 2 (Figure~\ref{fig:density_subopt}), while the ridge estimator
diverges with increasing variance, the exploration-mixed estimator and the
DR-mix estimator converge.  
Since there are no new samples from arm 2, the term $\sqrt{n}$ increases
the mean and the variance of the density of the ridge estimator.  
In contrast, the mean of the density of the exploration-mixed estimator
and the proposed DR-mix estimator converges to 0, and the variance increases slower than that of the ridge estimator.  

The fast convergence of the DR-mix estimator is a consequence of combining the
exploration-mixed data and DR technique.    
The exploration-mixed estimator~\eqref{eq:mixup} leverages the linear structure of the mean reward vector to create a pair of ``mixed'' contexts and rewards by combining the context of the selected arm (arm 1) with randomly selected arms (arms 2 and 3) from the exploration phase.
While the ``mixing'' allows the exploration-mixed estimator to learn all $d=3$ entries of the parameter vector, it minimizes $\Expectation[|\Tx{\Action t}{t}^\top(\tilde{\theta}_t-\theta_{\star})|^2]$ instead of the basis vectors for target contexts $\mX$ of interest.
Although the exploration-mixed estimator eventually converges to the true parameter, $\theta_{\star}$, the target of interest $\sum_{k=1}^{3}|x_k^\top(\tilde{\theta}_t-\theta_{\star})|^2 = \|\tilde{\theta}_t - \theta_{\star}\|_2$ converges slower than $\Expectation[|\Tx{\Action t}{t}^\top(\tilde{\theta}_t-\theta_{\star})|^2]$. 
Therefore, we apply the DR method and use pseudo-rewards~\eqref{eq:pseudoY} to move the target to the one of interest by modifying the context from $\Tx{\Action t}{t}$ to $\mX$ (equivalently, changing Gram matrix to $\sum_{k=1}^K x_kx_k^\top)$.
Thus, our proposed estimator minimizes the target $\sum_{k=1}^{3}|x_k^\top(\widehat{\theta}_t-\theta_{\star})|^2$ directly and estimates the mean rewards of the arms significantly faster.

\begin{figure}[t]
\centering 
\begin{subfigure}{0.90\textwidth} 
\includegraphics[width=0.3\textwidth]{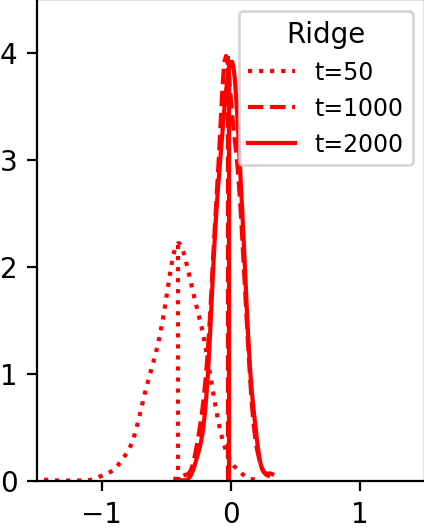}
\hfill
\includegraphics[width=0.3\textwidth]{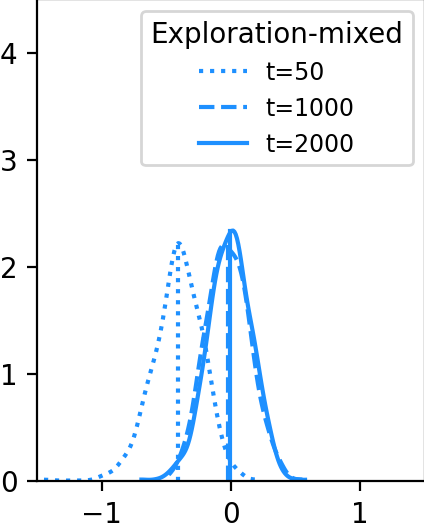}
\hfill
\includegraphics[width=0.3\textwidth]{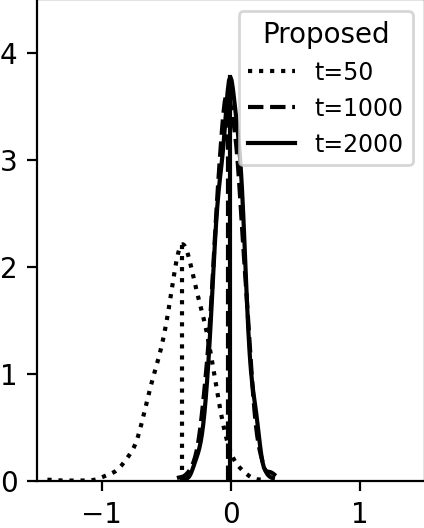}
\hfill
\caption{On the \textit{exploited} arm (arm 1)}
\label{fig:density_opt} 
\end{subfigure}
\\
\begin{subfigure}{0.90\textwidth} 
\includegraphics[width=0.3\textwidth]
{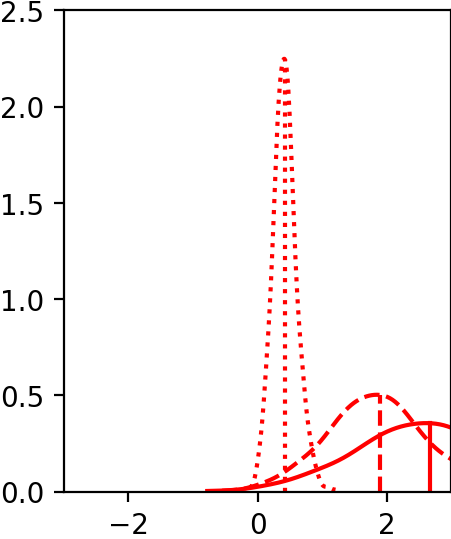}
\hfill
\includegraphics[width=0.3\textwidth]{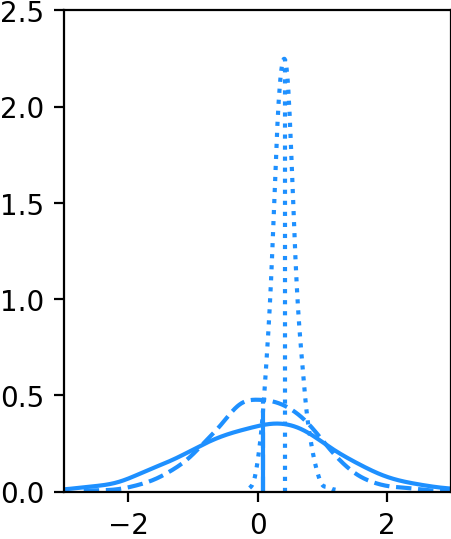}
\hfill
\includegraphics[width=0.3\textwidth]{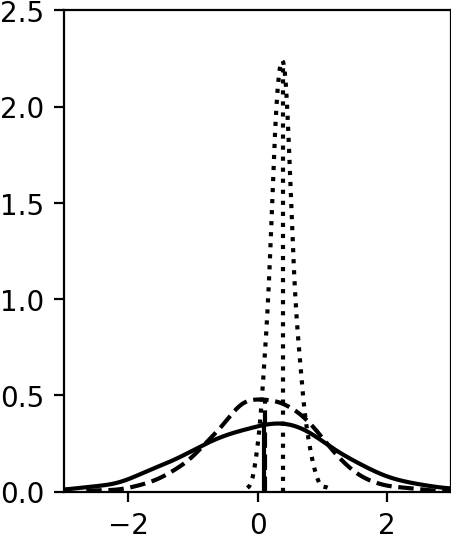}
\hfill
\caption{On the \textit{unexploited} arm (arm 2)}
\label{fig:density_subopt}
\end{subfigure} 
\caption{Changes in densities of
  $\sqrt{n}(\widehat{\theta}-\theta_{\star})$ over the number of samples
  $n=50, 500, 2000$ on the \textit{exploited} arm (arm 1) and the
  \textit{unexploited} arm (arm 2). 
The vertical line represents the average computed from 1000 independent experiments.
The proposed DR-mix estimator converges faster with lower variance than the ridge
and exploration-mixed estimator on all arms. 
}
\label{fig:density} 
\end{figure}

\section{Missing Proofs}
\label{sec:missing_proofs}

\subsection{Proof of Theorem~\ref{thm:lower_bound}}
\label{sec:lower_bound_proof}


Before we prove Theorem~\ref{thm:lower_bound}, we present a lower bound for the error of linearly parameterized rewards. 
\begin{lem}
\label{lem:lower_noise} Suppose $\mX:=\{x_{k}\in\Real^{d}:k\in[K]\}$
spans $\Real^{d}$. For $k\in[K]$ and $i_{k}\geq1$, let $Y^{(i_{k})}=\theta_{\star}^{\top}x_{k}+\eta_{k}^{(i_{k})}$,
where $\eta_{k}^{(i_{k})}$ is an identically and independently distributed
noise. Then, for any $\sigma>0$ and $\epsilon\in(0,\sigma/4)$. Then
there exist mean-zero $\sigma$-sub-Gaussian random noises such that
the error 
\[
\max_{x\in\mX}\abs{x^{\top}\big(\widehat{\theta}_{n}-\theta_{\star}\big)}>\epsilon,
\]
with probability at least $\delta\in(0,1/4)$, for any estimator $\widehat{\theta}_{n}$
that uses at most $n_{1}+\cdots+n_{K}\le\frac{d\sigma^{2}}{3\epsilon^{2}}\log\frac{1}{4\delta}$
number of independent samples $\cup_{k=1}^{K}\{x^{(i_{k})},Y^{(i_{k})}\}_{i_{k}\in[n_{k}]}$.
\end{lem}

\begin{proof}
\textbf{Step 1. Constructing a noise distribution:} For $k\in[K]$
the noise 
\[
\eta_{k}^{(i_{k})}=\begin{cases}
-\frac{\sigma}{2\left(1+\frac{2\epsilon}{\sigma}v_{k}\right)} &
\text{w.p.}\; p_{+}(v_k) := \frac{1}{2}+\frac{\epsilon}{\sigma}v_{k}\\ 
\frac{\sigma}{2\left(1-\frac{2\epsilon}{\sigma}v_{k}\right)} &
\text{w.p.}\; p_{-}(v_k) :=\frac{1}{2}-\frac{\epsilon}{\sigma}v_{k} 
\end{cases},
\]
where
$v_{k}\in\{-1,1\}$. Let $\Probability_{v_{k}}$
denote the probability measure for the noise $\eta_{k}$. Since 
\[
\abs{\frac{1}{2\left(1+\frac{2\epsilon}{\sigma}\right)}}\le\abs{\frac{1}{2\left(1-\frac{2\epsilon}{\sigma}\right)}}\le1,
\]
for all $\epsilon\in(0,\frac{\sigma}{4})$, we have $|\eta_{k}^{(i_{k})}|\le\sigma$.
It is easy to show that the difference
\[
\abs{\frac{\sigma}{2\left(1-\frac{2\epsilon}{\sigma}v_{k}\right)}+\frac{\sigma}{2\left(1+\frac{2\epsilon}{\sigma}v_{k}\right)}}>2\epsilon,
\]
for all $v_{k}\in\{-1,1\}$. 

\textbf{Step 2. Reduction to $d$ parameter estimation: }For $v\in\{-1,1\}^{d}$
let
$\Probability_{v}^{(n)}=\prod_{k=1}^{K}\prod_{i_{k}=1}^{n_{k}}\Probability_{v_{k}}^{(i_{k})}$ 
denote 
joint distribution for $\bm{\eta}^{(n)} = \{\eta_{k}^{i_k}: i_k\in [n_k], k \in [K]\}$.
Let $\mS = \{\tilde{x}_{1},\ldots,\tilde{x}_{d}\}\subset\mX$ 
denote a set of linear independent vectors, 
and let $y_{j}:=\tilde{x}_{j}^{\top}\theta_{\star}$, $j \in \mS$. 
Let $Y_{k}=y_{k}+\eta_{k}$ denote a sample for arm $k\in[K]$, independent
of $n$ initial samples, with $\eta_{k}\sim\Probability_{v_{k}}$. Let
$\Probability_{v}:=\prod_{k=1}^{K}\Probability_{v_{k}}\times\Probability_{v}^{(n)}$.

Let $\widehat{\theta}_{n}:\mX\times\Real\to\Real$ denote any estimator that estimates $\theta_{\star}$ using the $n  = \sum_{k=1}^K n_k$ data points.
Clearly, 
\[
\Probability_{v}\left(\max_{k\in[K]}
  \abs{x_{k}^{\top}\left(\widehat{\theta}_{n}(x_{k},Y_{k})-\theta_{\star}\right)}>\epsilon\right)  
\ge
\Probability_{v}\left(\max_{j\in[d]}
  \abs{\tilde{x}_{j}^{\top}\widehat{\theta}_{n}(\tilde{x}_{j},Y_{j})-y_{j}}>\epsilon\right).  
\]
Next,
\begin{eqnarray}
\lefteqn{\Probability_{v}\left(\max_{j\in[d]}\abs{\tilde{x}_{j}^{\top}\widehat{\theta}_{n}(\tilde{x}_{j},Y_{j})-y_{j}}\le\epsilon\right)}\\ 
  & \le & \Probability_{v}\left(\bigcap_{j=1}^{d}\left\{
          \abs{\tilde{x}_{j}^{\top}\widehat{\theta}_{n}(\tilde{x}_{j},Y_{j})-y_{j}}\le\epsilon\right\}
          \right) \notag \\ 
& = & \Probability_{v}\left(\bigcap_{j=1}^{d}\left\{
      \abs{\tilde{x}_{j}^{\top}\widehat{\theta}_{n}(\tilde{x}_{j},Y_{j})-Y_{j}-\eta_{j}}\le\epsilon\right\}
      \right) \notag \\
  & = &
        \Expectation_{v}^{(n)} \left[\left.\prod_{j=1}^{d}\Probability_{v_{j}}\left(\abs{\tilde{x}_{j}^{\top}
        \widehat{\theta}_{n}(\tilde{x}_{j},Y_{j})-Y_{j}-\eta_{j}}\le\epsilon 
        \right)\right|
        \bm{\eta}^{(n)} 
        \right], \label{eq:P_v_decomp}
\end{eqnarray}
where the last inequality holds because $\eta_{1},\ldots,\eta_{d}$
are independent. 
For each $j\in[d]$, 
\begin{eqnarray}
   \lefteqn{\Probability_{v_{j}}\left(\left.\abs{\tilde{x}_{j}^{\top}
  \widehat{\theta}_{n}-Y_{j}-\eta_{j}}>\epsilon\right|\boldsymbol{\eta}^{(n)}\right)}
  \notag \\    
 & = &
       \Probability_{v_{j}}\left(\left.\abs{\tilde{x}_{j}^{\top}
       \widehat{\theta}_{n}-Y_{j}+\frac{\sigma}{4p_{+}(v_{j})}}
       >\epsilon\right|\eta_{j}=-\frac{\sigma}{4p_{+}(v_{j})},   
       \boldsymbol{\eta}^{(n)}\right)p_{+}(v_{j}) \notag \\ 
 && \mbox{}
    +\Probability_{v_{j}}\left(\left.\abs{\tilde{x}_{j}^{\top}\widehat{\theta}_{n}-Y_{j}
    -\frac{\sigma}{4p_{-}(v 
    _{j})}}>\epsilon\right|\eta_{j}=\frac{\sigma}{4p_{-}(v_{j})},
    \boldsymbol{\eta}^{(n)}\right)p_{-}(v_{j}), \notag \\ 
 & \ge &  
   \Probability_{v_{j}}\left(\left.\abs{\tilde{x}_{j}^{\top}\widehat{\theta}_{n}-Y_{j}
         + \frac{\sigma}{4p_{+}(v_{j})}}>\epsilon\right|\eta_{j} =
         -\frac{\sigma}{4p_{+}(v_{j})}, \boldsymbol{\eta}^{(n)}\right)
         \min\left\{ p_{+}(v_{j}),p_{-}(v_{j})\right\} \notag \\
  && \mbox{}
         +\Probability_{v_{j}}\left(\left.\abs{\tilde{x}_{j}^{\top}\widehat{\theta}_{n}-Y_{j}
     -\frac{\sigma}{4p_{-}(v_{j})}}>\epsilon\right|\eta_{j} =
     \frac{\sigma}{4p_{-}(v_{j})},\boldsymbol{\eta}^{(n)}\right)
   \min\left\{ p_{+}(v_{j}),p_{-}(v_{j}) \right\} \notag \\  
  & \ge & \min\left\{ p_{+}(v_{j}),p_{-}(v_{j})\right\}, \label{eq:ProbLowBnd} 
\end{eqnarray}
where the last inequality holds because for $v_{j}\in\{-1,1\}$, 
\[
\left\{ u\in\Real:\abs{u-Y_{j}+\frac{\sigma}{4p_{+}(v_{j})}}>\epsilon\right\} \cup\left\{ u\in\Real:\abs{u-Y_{j}-\frac{\sigma}{4p_{-}(v_{j})}}>\epsilon\right\} =\Real.
\]
Define the function $\hat{f}(x,z): \Real \times \{-1,1\}$ as follows:
\[
\hat{f}(x,z):=\begin{cases}
x-\frac{\sigma}{4p_{+}(z)} & \text{if}\;p_{+}(z)>p_{-}(z)\\
x+\frac{\sigma}{4p_{-}(z)} & \text{o.w.}
\end{cases}.
\]
Then, the estimator $\hat{f}(Y_{i},v_{j})$ attains 
the minimum in \eqref{eq:ProbLowBnd} since
\begin{eqnarray*}
  \Probability_{v_{j}}\left(\abs{\hat{f}(Y_{j},v_{j})-y_{j}}>\epsilon\right)
  & = &
    \Probability_{v_{j}}\left(\left.\hat{f}(Y_{j},v_{j})=Y_{j}+\frac{\sigma}{4p_{-}(v_{j})}\right|
      \eta_{j}=-\frac{\sigma}{4p_{+}(v_{j})}\right)p_{+}(v_{j})\\  
 && \mbox{}
    +\Probability_{v_{j}}\left(\left. \hat{f}(Y_{j},v_{j}) =
    Y_{j}-\frac{\sigma}{4p_{+}(v_{k})}\right|\eta_{j}=\frac{\sigma}{4p_{-}(v_{j})}\right)p_{-}(v_{j})\\   
& = & \Indicator{p_{+}(v_{j})\le p_{-}(v_{j})}p_{+}(v_{j}) +
      \Indicator{p_{+}(v_{j})>p_{-}(v_{j})}p_{-}(v_{j})\\ 
& = & \min\left\{ p_{+}(v_{j}),p_{-}(v_{j})\right\} \\
& \le &
        \Probability_{v_{j}}\left(\left. \abs{\tilde{x}_{j}^{\top}\widehat{\theta}_{n}-Y_{j}-\eta_{j}}
        > \epsilon\right|\boldsymbol{\eta}^{(n)}\right). 
\end{eqnarray*}
Thus, for $j\in[d]$ and $v_{j}\in\{-1,1\}$, 
\[
\Probability_{v_{j}}\left(\left.\abs{\tilde{x}_{j}^{\top}\widehat{\theta}_{n}-Y_{j}-\eta_{j}}\le\epsilon\right|\boldsymbol{\eta}_{n}\right)\le\Probability_{v_{j}}\left(\abs{\widehat{f}(Y_{j},v_{j})-y_{j}}\le\epsilon\right).
\]
Plugging this upper bound in~\eqref{eq:P_v_decomp}, we get 
\begin{eqnarray*}
\Probability_{v}\left(
  \max_{j\in[d]} \abs{\tilde{x}_{j}^{\top}
  \widehat{\theta}_{n}(\tilde{x}_{j},Y_{j})-y_{j}} \le \epsilon\right) 
  & \le & 
    \Expectation_{v}^{(n)}\left[
          \prod_{j=1}^{d} \Probability_{v_{j}}
          \left(\abs{\hat{f}(Y_{j},v_{j})-y_{j}}\le\epsilon\right)\right]\\   
& = & \Probability_{v}\left(\bigcap_{j=1}^{d}\left\{
      \abs{\hat{f}(Y_{j},v_{j})-y_{j}}\le\epsilon\right\} \right).
\end{eqnarray*}
Let $n_{\min}=\min_{j\in[d]}n_{j}$ denote the minimum number of samples over arms in $[d]$. 
Let 
$\Probability_{v}^{(n_{\min})} =
\prod_{i=1}^{n_{\min}}\prod_{j=1}^{d}\Probability_{v_{j}}^{(i)}$ denote
the joint probability of the noise associated with $n_{\min}$ samples from
each arm $\tilde{x} \in \mS$. 
For $j\in[d]$ and any estimator $\widehat{v}_{j}$
for $v_{j}$ that uses $n_{\min}$ samples from arm $j=1,\ldots,d$,
\begin{eqnarray*}
\Probability_{v}\left(\bigcap_{j=1}^{d}\left\{
  \abs{\widehat{f}(Y_{j},v_{j})-y_{j}}\le\epsilon\right\}
  \right)&=&\Probability_{v}\left(\bigcap_{j=1}^{d}\left\{
             \abs{\widehat{f}(Y_{j},\widehat{v}_{j})-y_{j}}\le\epsilon\right\}
             \cap\left\{ \widehat{v}_{j}=v_{j}\right\} \right)\\ 
&\le&\Probability_{v}\left(\bigcap_{j=1}^{d}\left\{ \widehat{v}_{j}=v_{j}\right\} \right)\\
&=&\Probability_{v}^{(n_{\min})}\left(\bigcap_{j=1}^{d}\left\{
    \widehat{v}_{j}=v_{j}\right\} \right),
\end{eqnarray*}
where the last equality follows from the fact that $\widehat{v}_{j}$
only uses $n_{\min}$ samples from arm $j=1,\ldots,d$.
Therefore,
\begin{equation*}
\begin{split}
\Probability_{v}\left(\max_{k\in[K]}\abs{x_{k}^{\top}\left(\widehat{\theta}_{n}(x_{k},Y_{k})-\theta_{\star}\right)}>\epsilon\right)
&\geq \Probability_{v}\left(\bigcap_{j=1}^{d}\left\{
    \abs{\widehat{f}(Y_{j},v_{j})-y_{j}}>\epsilon\right\} \right)\\
&\ge \Probability_{v}^{(n_{\min})}\left(\bigcup_{j=1}^{d}\left\{
    \widehat{v}_{j}\neq v_{j}\right\} \right). 
\end{split}
\end{equation*}

\textbf{Step 3. Lower bound for the error probability.} Taking maximum
over $v$ gives, 
\begin{align*}
\sup_{v\in\{-1,1\}^{d}}\Probability_{v}\left(\max_{k\in[K]}\abs{x_{k}^{\top}\left(\widehat{\theta}_{n}(x_{k},Y_{k})-\theta_{\star}\right)}>\epsilon\right)\ge & \sup_{v\in\{-1,1\}^{d}}\Probability_{v}^{(n_{\min})}\left(\widehat{v}\neq v\right)\\
= & \sup_{v\in\{-1,1\}^{d}}\Expectation_{v}^{(n_{\min})}\left[\max_{i\in[d]}\Indicator{\widehat{v}_{i}\neq v_{i}}\right]\\
\ge & \sup_{v\in\{-1,1\}^{d}}\frac{1}{d}\Expectation_{v}^{(n_{\min})}\left[\sum_{i=1}^{d}\Indicator{\widehat{v}_{i}\neq v_{i}}\right].
\end{align*}
For two vectors $u$ and $w$ let $u\sim w$ denote that $u$ and
$w$ only differ in one coordinate. By Assouad's method (Lemma~\ref{lem:assouad}),
there exists at least one $\tilde{v}\in\{-1,1\}^{d}$ such that 
\[
\begin{split}
\frac{1}{d}\Expectation_{\tilde{v}}^{(n_{\min})}\left[\sum_{i=1}^{d}\Indicator{\widehat{v}_{i}\neq\tilde{v}_{i}}\right]
&\ge\frac{1}{2}\min_{u,w:u\sim w}\norm{\min\left\{ \Probability_{u}^{(n_{\min})},\Probability_{w}^{(n_{\min})}\right\} }_{1}\\
&=\frac{1}{2}\min_{u,w:u\sim w}\sum_{\boldsymbol{\eta}^{(n_{\min})}}\abs{\min\left\{ \Probability_{u}(\boldsymbol{\eta}^{(n_{\min})}),\Probability_{w}(\boldsymbol{\eta}^{(n_{\min})})\right\} }
\end{split}
\]
Thus, 
\begin{align*}
\sup_{v\in\{-1,1\}^{d}}\Probability_{v}\left(\max_{k\in[K]}\abs{x_{k}^{\top}\left(\widehat{\theta}_{n}(x_{k},Y_{k})-\theta_{\star}\right)}>\epsilon\right)\ge & \frac{1}{2}\min_{u,w:u\sim w}\norm{\min\left\{ \Probability_{u}^{(n_{\min})},\Probability_{w}^{(n_{\min})}\right\} }_{1}\\
= & \frac{1}{2}\min_{u,w:u\sim w}\left\{ 1-\frac{1}{2}\norm{\Probability_{u}^{(n_{\min})}-\Probability_{w}^{(n_{\min})}}_{1}\right\} \\
= & \frac{1}{2}\min_{u,w:u\sim w}\left\{ 1-TV\left(\Probability_{u}^{(n_{\min})},\Probability_{w}^{(n_{\min})}\right)\right\} ,
\end{align*}
where $TV(\Probability_{1},\Probability_{2}):=\sup\abs{\Probability_{1}\left(\cdot\right)-\Probability_{2}\left(\cdot\right)}$
is the total variation distance between two probability measures $\Probability_{1}$
and $\Probability_{2}$. By Bretagnolle-Huber inequality, 
\begin{align*}
1-TV\left(\Probability_{u}^{(n_{\min})},\Probability_{w}^{(n_{\min})}\right)\ge & \frac{1}{2}\exp\left(-KL\left(\Probability_{u}^{(n_{\min})},\Probability_{w}^{(n_{\min})}\right)\right)\\
= & \frac{1}{2}\exp\left(-n_{\min}KL\left(\Probability_{u},\Probability_{w}\right)\right).
\end{align*}
where the last equality uses the chain rule of entropy. Because there
exists only one $j\in[d]$ such that $u_{j}\neq w_{j}$, 
\begin{eqnarray*}
  KL\left(\Probability_{u},\Probability_{w}\right)
  & = & 
        \left(\frac{1}{2}+\frac{\epsilon}{\sigma}\right)\log\frac{\frac{1}{2}+\frac{\epsilon}{\sigma}}{\frac{1}{2}-\frac{\epsilon}{\sigma}}+\left(\frac{1}{2}-\frac{\epsilon}{\sigma}\right)\log\frac{\frac{1}{2}-\frac{\epsilon}{\sigma}}{\frac{1}{2}+\frac{\epsilon}{\sigma}}
  \\ 
  & = & \frac{2\epsilon}{\sigma}\log\frac{1+\frac{2\epsilon}{\sigma}}{1-\frac{2\epsilon}{\sigma}}\\
  & \le &
          \frac{2\epsilon}{\sigma}\left(\frac{2\epsilon}{\sigma}-\log\left(1-\frac{2\epsilon}{\sigma}\right)\right)\\ 
  & \le & \frac{2\epsilon}{\sigma}\left(\frac{2\epsilon}{\sigma}+\frac{\frac{2\epsilon}{\sigma}}{1-\frac{2\epsilon}{\sigma}}\right)\\
  & \le &
          \frac{2\epsilon}{\sigma}\left(\frac{2\epsilon}{\sigma}+\frac{4\epsilon}{\sigma}\right)
  \\ 
  & = & \frac{12\epsilon^{2}}{\sigma^{2}},
\end{eqnarray*}
where the third inequality holds by $\epsilon\in(0,\sigma/4)$.
Thus,
\begin{equation*}
\begin{split}
\sup_{v\in\{-1,1\}^{d}}\Probability_{v}\left(\max_{k\in[K]}\abs{x_{k}^{\top}\left(\widehat{\theta}_{n}(x_{k},Y_{k})-\theta_{\star}\right)}>\epsilon\right)
&\ge\frac{1}{2}\exp\left(-n_{\min}KL\left(\Probability_{u},\Probability_{w}\right)\right)\\
&\ge\frac{1}{2}\exp\left(-\frac{12n_{\min}\epsilon^{2}}{\sigma^{2}}\right).
\end{split}
\end{equation*}
Setting $n_{\min}\le\frac{\sigma^{2}}{12\epsilon^{2}}\log\frac{1}{4\delta}$
gives
\[
\sup_{v\in\{-1,1\}^{d}}\Probability_{v}\left(\max_{k\in[K]}\abs{x_{k}^{\top}\left(\widehat{\theta}_{n}(x_{k},Y_{k})-\theta_{\star}\right)}>\epsilon\right)\ge\delta.
\]

\textbf{Step 4. Computing the required number of samples:} Recall
that $n_{\min}:=\min_{j\in[d]}n_{j}$ is the minimum number of samples
over any $d$ linearly independent contexts. Thus, if $n\le\frac{d\sigma^{2}}{12\epsilon^{2}}\log\frac{1}{4\delta}$
then $n_{\min}\le\frac{\sigma^{2}}{12\epsilon^{2}}\log\frac{1}{4\delta}$
and the lower bound of the probability holds. 
\end{proof}

Now we are ready to prove the lower sample complexity bound for PFILIn.

\begin{proof}[Proof of Theorem~\ref{thm:lower_bound}]
\textbf{Step 1. Characterize the failure event:}
In order to meet the success condition \eqref{eq:PFI_condition} with $\delta\in(0,1/4)$,
the algorithm must produce an estimate $\widehat{y}_{k}\in\Real^{L}$ 
for $y_{k} = \Theta_{\star}x_k$ such that 
\[
\Probability\left(\bigcap_{k=1}^{K}\left\{
    \norm{\widehat{y}_{k}-y_{k}}_{\infty}\le\frac{1}{2} \Delta_{(k),\epsilon}\right\} \right)\ge1-\delta.
\]
Note that
\begin{align*}
\bigcap_{k=1}^{K}\left\{
  \norm{\widehat{y}_{k}-y_{k}}_{\infty}\le\frac{1}{2}\Delta_{(k),\epsilon}
\right\}
&\subseteq
\bigcap_{k=1}^{d}\left\{
  \norm{\widehat{y}_{k}-y_{k}}_{\infty}\le\frac{1}{2}\Delta_{(k),\epsilon}
\right\}
\\
&= \bigcap_{k=1}^{d}\bigcap_{l=1}^{L}\left\{ \abs{\widehat{y}_{k}^{(l)}-y_{k}^{(l)}}\le\frac{1}{2}\Delta_{(k),\epsilon}\right\} 
\end{align*}
For $\ell \in [L]$ and $k \in [d]$, let
\[
  \mathcal{B}_{k}^{\langle
  l\rangle}:=\left\{|\widehat{y}_{k}^{\langle \ell \rangle}-y_{k}^{\langle
  l\rangle}|>\frac{1}{2}\Delta_{(k)}\right\}.
\]
Thus, if $\Probability(\bigcup_{l=1}^{L}  \bigcup_{k=1}^{d}
\mathcal{B}_{k}^{\langle \ell \rangle})>\delta$, then
$\Probability\left(\bigcap_{k=1}^{K}\{\|\widehat{y}_{k}-y_{k}\|_{\infty}
  \le\frac{1}{2}\Delta_{(k),\epsilon}\}\right)< 1-\delta$, i.e. the algorithm cannot
satisfy the success condition \eqref{eq:PFI_condition}. 

\textbf{Step 2. Compute the required number of samples:}
For each arm $k\in[d]$, suppose the number of observations $t_k$ for the parameters $\{\Ereward{l}{k}:\ell \in [L]\}$ satisfies the upper bound $t_k\le\frac{\sigma^{2}}{3 \Delta_{(k),\epsilon}^{2}}\log\frac{1}{4p}$.  
Then by Lemma~\ref{lem:lower_noise}, for any estimator $\widehat{y}_{k}^{\langle \ell \rangle}$ 
\begin{equation}
\abs{\widehat{y}_{k}^{\langle \ell \rangle}-\Ereward lk}>\frac{\Delta_{(k),\epsilon}}{2},
\label{eq:fail_k}
\end{equation}
holds with probability at least $p$, and $\Probability(\mathcal{B}_{k}^{\langle \ell \rangle})\ge p$.
Since, for each $k\in[d]$, the estimators $\widehat{y}^{\langle 1
  \rangle}_{k}, \ldots, \widehat{y}^{\langle L \rangle}_{k}$  are
independent of each other, the events $\mathcal{B}_k^{\langle 1
  \rangle},\ldots, \mathcal{B}_k^{\langle L \rangle}$ are independent, and
therefore,  
\[
\Probability\left(\bigcap_{k=1}^{d}\bigcap_{l=1}^{L}(\mathcal{B}_{k}^{\langle \ell \rangle})^{c}\right)
=\prod_{l=1}^{L}\Probability\left(\bigcap_{k=1}^{d}(\mathcal{B}_{k}^{\langle \ell \rangle})^{c}\right)
=\prod_{l=1}^{L}\left\{
  1-\Probability\left(\bigcup_{k=1}^{d}\mathcal{B}_{k}^{\langle
      l\rangle}\right)\right\}.
\]
Therefore, if the total number of observations
$t\le \sum_{k=1}^d \frac{\sigma^{2}}{3 \Delta_{(k),\epsilon}^{2}}\log\frac{1}{4p}$, 
there exists an arm $k\in[d]$ such that the number of independent
  samples is less than $\frac{\sigma^{2}}{3
    \Delta_{(k),\epsilon}^{2}}\log\frac{1}{4p}$ and, 
\[
\Probability\left(\bigcap_{k=1}^{d}\bigcap_{l=1}^{L}(\mathcal{B}_{k}^{\langle \ell \rangle})^{c}\right)\le\left(1-p\right)^{L}\le\frac{1}{1+Lp}.
\]
Because $\delta\in(0,1/4)$, setting $p=\frac{4\delta}{3L}\ge\frac{\delta}{L(1-\delta)}$
gives $\Probability\left(\cap_{k=1}^{d}\cap_{l=1}^{L}(\mathcal{B}_{k}^{\langle \ell \rangle})^{c}\right)\le1-\delta$.
Thus, any algorithm requires at least  
\[
\sum_{k=1}^{d}\frac{\sigma^{2}}{3\Delta_{k}^{2}}\log\frac{3L}{4\delta}
\]
number of rounds to meet the success condition \eqref{eq:PFI_condition}. 
\end{proof}

\subsection{Proof of Lemma~\ref{lem:mixup}}

By definition~\eqref{eq:mixup_data}, for all $\ell \in [L]$, 
\[
\Ty \ell {\Action t}t-\Tx{\Action t}t^{\top}\Parameter \ell =w_{t}\left(\Reward \ell{\Action t}t-x_{\Action t}^{\top}\Parameter \ell \right)+\check{w}_{t}\left(\Reward \ell {\Action{\check{n}_{t}}}{\check{n}_{t}}-x_{\Action{\check{n}_{t}}}^{\top}\Parameter \ell \right).
\]
Taking conditional expectations on both sides, 
\begin{align*}
\CE{\Ty \ell {\Action t}t-\Tx{\Action t}t^{\top}\Parameter \ell }{\Filtration t}= & \CE{w_{t}}{\Filtration t}\left(\Reward \ell {\Action t}t-x_{\Action t}^{\top}\Parameter \ell\right)+\CE{\check{w}_{t}\left(\Reward \ell{\Action{\check{n}_{t}}}{\check{n}_{t}}-x_{\Action{\check{n}_{t}}}^{\top}\Parameter \ell\right)}{\Filtration t}\\
= & \CE{\check{w}_{t}\left(\Reward \ell{\Action{\check{n}_{t}}}{\check{n}_{t}}-x_{\Action{\check{n}_{t}}}^{\top}\Parameter \ell\right)}{\Filtration t}\\
= & \CE{\CE{\check{w}_{t}}{\Filtration t,\check{n}_{t}}\left(\Reward l{\Action{\check{n}_{t}}}{\check{n}_{t}}-x_{\Action{\check{n}_{t}}}^{\top}\Parameter \ell\right)}{\Filtration t}\\
= & 0,
\end{align*}
which proves the first identity. For the expected Gram matrix, by
definition~\eqref{eq:mixup}, 
\begin{align*}
\CE{\Tx{\Action t}t\Tx{\Action t}t^{\top}}{\Filtration t,\check{n}_{t}}= & \CE{w_{t}^{2}x_{\Action t}x_{\Action t}^{\top}+w_{t}\check{w}_{t}\left(x_{\Action t}x_{\Action{\check{n}_{t}}}^{\top}+x_{\Action{\check{n}_{t}}}x_{\Action t}^{\top}\right)+\check{w}_{t}^{2}x_{\Action{\check{n}_{t}}}x_{\Action{\check{n}_{t}}}^{\top}}{\Filtration t,\check{n}_{t}}\\
= & \CE{w_{t}^{2}}{\Filtration t,\check{n}_{t}}x_{\Action t}x_{\Action t}^{\top}+\CE{w_{t}\check{w}_{t}}{\Filtration t,\check{n}_{t}}\left(x_{\Action t}x_{\Action{\check{n}_{t}}}^{\top}+x_{\Action{\check{n}_{t}}}x_{\Action t}^{\top}\right)\\
 & +\CE{\check{w}_{t}^{2}}{\Filtration t,\check{n}_{t}}x_{\Action{\check{n}_{t}}}x_{\Action{\check{n}_{t}}}^{\top}\\
= & x_{\Action t}x_{\Action t}^{\top}+x_{\Action{\check{n}_{t}}}x_{\Action{\check{n}_{t}}}^{\top}.
\end{align*}
Taking conditional expectations on both sides, 
\begin{align*}
\CE{\Tx{\Action t}t\Tx{\Action t}t^{\top}}{\Filtration t}=&\Expectation\left[\CE{\Tx{\Action t}t\Tx{\Action t}t^{\top}}{\Filtration t,\check{n}_{t}}\Filtration t\right]\\
=&\CE{x_{\Action t}x_{\Action t}^{\top}+x_{\Action{\check{n}_{t}}}x_{\Action{\check{n}_{t}}}^{\top}}{\Filtration t}.
\end{align*}
By definition of $n_t$, we obtain $a_{n_t}=\check{a}_t$.
Because $\check{a}_t \sim \pi^\star$,
\begin{align*}
\CE{x_{\Action t}x_{\Action t}^{\top}+x_{\Action{\check{n}_{t}}}x_{\Action{\check{n}_{t}}}^{\top}}{\Filtration t}
=&x_{\Action t}x_{\Action t}^{\top}+\CE{x_{\Action{\check{n}_{t}}}x_{\Action{\check{n}_{t}}}^{\top}}{\Filtration t}\\
=&x_{\Action t}x_{\Action t}^{\top}+\sum_{k=1}^{K}\pi_{k}^{\star}x_{k}x_{k}^{\top}\\
\succeq&\sum_{k=1}^{K}\pi_{k}^{\star}x_{k}x_{k}^{\top}
\end{align*}
which completes the proof.

\subsection{Design Matrix with Context Basis}
\label{sec:optimal_design}

We provide a theoretical result on the design matrix $\sum_{i=1}^{d}\lambda_i u_i u_i^\top = \sum_{k=1}^{K} x_k x_k^\top$ constructed by the exploration strategy in Section~\ref{sec:dimension_reduction}.

\begin{lem}[A bound for the normalized norm of contexts.]
\label{lem:norm_bound} For any $k\in[K]$ and $t\ge1$, the normalized
norm $x_{k}^{\top}(\sum_{k^{\prime}=1}^{K}x_{k^{\prime}}x_{k^{\prime}}^{\top})^{-1}x_{k}\le1$. 
\end{lem}

\begin{proof}
For each $k\in[K]$, by Sherman-Morrison formula, for any $\epsilon>0$
\begin{align*}
x_{k}^{\top}(\sum_{k^{\prime}=1}^{K}x_{k^{\prime}}x_{k^{\prime}}^{\top}+\epsilon I_{d})^{-1}x_{k}\le & x_{k}^{\top}(x_{k}x_{k}^{\top}+\epsilon I_{d})^{-1}x_{k}\\
= & \frac{x_{k}^{\top}x_{k}}{\epsilon}-\frac{\epsilon^{-2}(x_{k}^{\top}x_{k})^{2}}{1+\epsilon^{-1}x_{k}^{\top}x_{k}}\\
= & \frac{x_{k}^{\top}x_{k}}{\epsilon+x_{k}^{\top}x_{k}}\\
\le & 1.
\end{align*}
Letting $\epsilon\downarrow0$ completes the proof. 
\end{proof}

Lemma~\ref{lem:norm_bound} implies,
\[
\max_{k\in[K]} x_{k}^{\top}(d^{-1}\sum_{k^{\prime}=1}^{K}x_{k^{\prime}}x_{k^{\prime}}^{\top})^{-1}x_{k}\le d,
\]
which has same bound with G-optimal design \citep{smith1918standard}.
Although we reduce $K$ contexts into $d$ context basis vectors, our estimation strategy enjoys the property of the optimal design for all $K$ context vectors.

For the Gram matrix of the DR-mix estimator $F_t:=\sum_{s=1}^{t} \sum_{k=1}^{K} x_kx_k^\top + I_d$, Lemma~\ref{lem:norm_bound} implies,
\[
\norm{x_k}_{F_t^{-1}} \le 1/\sqrt{t},
\]
for all $k\in[K]$.
The DR-mix estimator imputes the reward on the basis contexts and minimizes $\|\cdot\|_{F_{t}}$-norm error, which efficiently estimates the rewards on all $K$ arms.
In contrast, the exploration-mixed estimator minimizes the $\|\cdot\|_{\sum_{s=1}^{t}\Tx{\Action s}s\Tx{\Action s}s^{\top}}$-norm error.
Although the expected Gram matrix in the exploration-mixed estimator has $\sum_{k=1}^{K} x_k x_k^\top$ it is discounted by the factor of $d$ because it employs only one of $d$ context basis, not all $d$ context basis, in each round.
Therefore the DR-mix estimator converges faster than the exploration-mixed estimator on the rewards of all $K$ arms.

\subsection{Coupling with Resampling}
\label{subsec:coupling}

We provide the details on the coupling with resampling in the following lemma. 
The key idea is coupling the event of interest with IID samples and bound the probability with another IID sample.
\begin{lem}
\label{lem:resampling} Let $\pi_{t}$ and $\tilde{\pi}_{t}$ denote
the distribution for action on $[K]$ and pseudo-action on $[d+1]$,
respectively. Let $\{a_{t}^{(m)}\sim\pi_{t}:m\ge1\}$ and $\{\tilde{a}_{t}^{(m)}\sim\pi_{t}:m\ge1\}$
denote IID samples from the distribution $\pi_{t}$and $\tilde{\pi}_{t}$
and for the number of resampling $\rho_{t}\in\Natural$, define new
contexts
\[
\tilde{x}_{i,t}:=\begin{cases}
\sqrt{\lambda_{i}}u_{i} & \forall i\in[d]\\
x_{a_{t}} & i=d+1
\end{cases},
\]
and the stopping time,
\[
m_{t}:=\min\left[\inf\left\{ m\ge1:x_{a_{t}^{(m)}}=\tilde{x}_{\tilde{a}_{t}^{(m)},t},\Reward{\ell}{a_{t}^{(m)}}t=\Newy{\ell}{\tilde{a}_{t}^{(m)}}t,\forall\ell\in[L]\right\} ,\rho_{t}\right].
\]
Then, for any function $f:\Real^{2}\to\Real$ and a real number $x\in\Real$,
\begin{align*}
\Probability\left(f(x_{a_{t}^{(m_{t})}},\Reward{\ell}{a_{t}^{(m_{t})}}t)>x\right) & \le\Probability\left(f(\tilde{x}_{\tilde{a}_{t}^{(1)},t},\Newy{\ell}{\tilde{a}_{t}^{(1)}}t)>x\right)+\Probability\left(m_{t}>\rho_{t}\right),\\
\Probability\left(f(\tilde{x}_{\tilde{a}_{t}^{(m_{t})},t},\Newy{\ell}{\tilde{a}_{t}^{(m_{t})}}t)>x\right) & \le\Probability\left(f(\tilde{x}_{\tilde{a}_{t}^{(1)},t},\Newy{\ell}{\tilde{a}_{t}^{(1)}}t)>x\right)+\Probability\left(m_{t}>\rho_{t}\right),
\end{align*}
for all $\ell\in[L]$.
\end{lem}

\begin{proof}
For $m\ge1$, let
\[
\mathcal{M}^{(m)}:=\left\{ x_{a_{t}^{(m)}}=\tilde{x}_{\tilde{a}_{t}^{(m)},t}\right\} \cap\left\{ \Reward{\ell}{a_{t}^{(m)}}t=\Newy{\ell}{\tilde{a}_{t}^{(m)}}t\right\} 
\]
By definition of $m_{t}$, the 
\begin{align*}
\Probability\left(f(x_{a_{t}^{(m_{t})}},\Reward{\ell}{a_{t}^{(m_{t})}}t)>x\right) & \le\Probability\left(\left\{ f(x_{a_{t}^{(m_{t})}},\Reward{\ell}{a_{t}^{(m_{t})}}t)>x\right\} \cap\mathcal{M}^{(m_{t})}\right)+\Probability\left(\left\{ \mathcal{M}^{(m_{t})}\right\} ^{c}\right)\\
 & =\Probability\left(\left\{ f(x_{a_{t}^{(m_{t})}},\Reward{\ell}{a_{t}^{(m_{t})}}t)>x\right\} \cap\mathcal{M}^{(m_{t})}\right)+\Probability\left(m_{t}>\rho_{t}\right).
\end{align*}
On the event $\mathcal{M}^{(m_{t})}$,
\[
\Probability\left(\left\{ f(x_{a_{t}^{(m_{t})}},\Reward{\ell}{a_{t}^{(m_{t})}}t)>x\right\} \cap\mathcal{M}^{(m_{t})}\right)=\Probability\left(\left\{ f(\tilde{x}_{\tilde{a}_{t}^{(m_{t})},t},\Newy{\ell}{\tilde{a}_{t}^{(m_{t})}}t)>x\right\} \cap\mathcal{M}^{(m_{t})}\right).
\]
Because the event $\{f(\tilde{x}_{\tilde{a}_{t}^{(m)},t},\Newy{\ell}{\tilde{a}_{t}^{(m)}}t)>x\}\cap\mathcal{M}^{(m)}$
is IID over $m\ge1$ given $\Filtration t$. 
\begin{align*}
 & \Probability\left(\left\{ f(\tilde{x}_{\tilde{a}_{t}^{(m_{t})},t},\Newy{\ell}{\tilde{a}_{t}^{(m_{t})}}t)>x\right\} \cap\mathcal{M}^{(m_{t})}\right).\\
 & =\Expectation\left[\CP{\left\{ f(\tilde{x}_{\tilde{a}_{t}^{(m_{t})},t},\Newy{\ell}{\tilde{a}_{t}^{(m_{t})}}t)>x\right\} \cap\mathcal{M}^{(m_{t})}}{\Filtration t}\right]\\
 & =\Expectation\left[\CP{\left\{ f(\tilde{x}_{\tilde{a}_{t}^{(1)},t},\Newy{\ell}{\tilde{a}_{t}^{(1)}}t)>x\right\} \cap\mathcal{M}^{(1)}}{\Filtration t}\right]\\
 & \le\Probability\left(f(\Newy{\ell}{\tilde{a}_{t}^{(1)}}t)>x\right).
\end{align*}
The second inequality in the lemma can be derived in a similar way.
\end{proof}

\subsection{Concerns in Reusing Samples in Exploration-Mixed Estimator}
\label{sec:reuse_error}
\begin{lem}[A self-normalizing bound for the exploration-mixed
estimator.]
\label{lem:self_mix} Fix $\delta\in(0,1)$. Then,
for all $t\in\Natural$ that satisfy $4\cdot288d^{2}\log\frac{2dt^{2}}{\delta}$,
the exploration-mixed estimator defined in \eqref{eq:mixup} satisfies
\[
\norm{\Impute t{\ell}-\Parameter{\ell}}_{F_{t}}\le2\norm{\Parameter{\ell}}_{F_{t}^{-1}}+12\sigma\sqrt{3d\log\left(\frac{3Lt}{\delta}\right)}+24\sigma\sqrt{\frac{3td}{\abs{\mE_{t}}}\log\left(\frac{3Lt}{\delta}\right)},
\]
with probability at least $1-4\delta$,
\end{lem}

\begin{proof}
Let us fix $t$ throughout the proof. For $s\in[t]$
and $\ell\in[L]$ let $\tilde{\eta}_{s}^{\langle\ell\rangle}:=\Ty{\ell}{\Action s}s-\Tx{\Action s}s^{\top}\Parameter{\ell}$.
By definition of $\Impute t{\ell}$, 
\begin{align*}
 & \Impute t{\ell}-\!\Parameter{\ell}\\
 & =\left(\sum_{s\in\mE_{t}}x_{\Action s}x_{\Action s}^{\top}+\sum_{s\notin\mE_{t}}\Tx{\Action s}s\Tx{\Action s}s^{\top}+\frac{1}{2}I_{d}\right)^{-1}\left(\sum_{s\in\mE_{t}}x_{\Action s}\Error{\ell}s+\sum_{s\notin\mE_{t}}\Tx{\Action s}s\tilde{\eta}_{s}^{\langle\ell\rangle}-\frac{1}{2}\Parameter{\ell}\right)\\
 & =\left(\sum_{s\in\mE_{t}}2x_{\Action s}x_{\Action s}^{\top}+\sum_{s\notin\mE_{t}}2\Tx{\Action s}s\Tx{\Action s}s^{\top}+I_{d}\right)^{-1}\left(\sum_{s\in\mE_{t}}2x_{\Action s}\Error{\ell}s+\sum_{s\notin\mE_{t}}2\Tx{\Action s}s\tilde{\eta}_{s}^{\langle\ell\rangle}-\Parameter{\ell}\right)\\
 & :=\tilde{A}_{t}^{-1}\left(\sum_{s\in\mE_{t}}2x_{\Action s}\Error{\ell}s+\sum_{s\notin\mE_{t}}2\Tx{\Action s}s\tilde{\eta}_{s}^{\langle\ell\rangle}-\Parameter{\ell}\right)
\end{align*}
Define the new contexts, 
\[
\tilde{x}_{i,s}:=\begin{cases}
\sqrt{\lambda_{i}}u_{i} & \forall i\in[d]\\
x_{a_{s}} & i=d+1
\end{cases},\quad\check{X}_{i,s}:=\begin{cases}
\check{w}_{s}\tilde{x}_{i,s}+\check{w}_{s}\sqrt{\lambda_{i_{s}}}u_{i_{s}} & \forall i\in[d]\\
\check{w}_{s}x_{a_{s}}+\check{w}_{s}\sqrt{\lambda_{i_{s}}}u_{i_{s}} & i=d+1
\end{cases}.
\]
Setting the number of resampling $\rho_{s}=\log(s^{2}/\delta)/\log(2)$,
the probability of obtaining matching samples $\tilde{x}_{\tilde{a}_{s}}^{(m_{s})}=x_{a_{s}}^{(m_{s})}$
for $m_{s}\le\rho_{s}$ is at least $1-\delta$ for all $s\in[t]$,
where $m_{s}$ is the number of trials until the matching. With the
matching pseudo action $\tilde{a}_{s}^{(m_{s})}$, 
\begin{align*}
\tilde{A}_{t} & =\sum_{s\in\mE_{t}}\frac{1}{\tilde{\pi}_{s}(\tilde{a}_{s}^{(m_{s})})}\tilde{x}_{\tilde{a}_{s}^{(m_{s})},s}\tilde{x}_{\tilde{a}_{s}^{(m_{s})},s}^{\top}+\sum_{s\notin\mE_{t}}\frac{1}{\tilde{\pi}_{s}(\tilde{a}_{s}^{(m_{s})})}\Tx{\tilde{a}_{s}^{(m_{s})}}s\Tx{\tilde{a}_{s}^{(m_{s})}}s^{\top}+I_{d}.
\end{align*}
By the coupling lemma (Lemma\ref{lem:resampling}) and , with probability
at least $1-\delta$, 
\begin{align*}
\tilde{A}_{t}= & \sum_{s\in\mE_{t}}\frac{1}{\tilde{\pi}_{s}(\tilde{a}_{s})}\tilde{x}_{\tilde{a}_{s},s}\tilde{x}_{\tilde{a}_{s},s}^{\top}+\sum_{s\notin\mE_{t}}\frac{1}{\tilde{\pi}_{s}(\tilde{a}_{s})}\Tx{\tilde{a}_{s}}s\Tx{\tilde{a}_{s}}s^{\top}+I_{d}\\
= & \sum_{s\in\mE_{t}}\sum_{i=1}^{d+1}\frac{\Indicator{\tilde{a}_{s}=i}}{\tilde{\pi}_{s}(i)}\tilde{x}_{i,s}\tilde{x}_{i,s}^{\top}+\sum_{s\notin\mE_{t}}\frac{\Indicator{\tilde{a}_{s}=i}}{\tilde{\pi}_{s}(i)}\Tx is\Tx is^{\top}+I_{d}\\
:= & \sum_{s\in\mE_{t}}\sum_{i=1}^{d+1}U_{i,s}\tilde{x}_{i,s}\tilde{x}_{i,s}^{\top}+\sum_{s\notin\mE_{t}}U_{i,s}\Tx is\Tx is^{\top}+I_{d}.
\end{align*}
Then the self-normalized norm, 
\[
\norm{\Impute t{\ell}-\!\Parameter{\ell}}_{F_{t}}=\norm{\sum_{s\in\mE_{t}}2x_{a_{s}}\Error{\ell}s+\sum_{s\notin\mE_{t}}2\Tx{a_{s}}s\tilde{\eta}_{s}^{\langle\ell\rangle}-\Parameter{\ell}}_{\tilde{A_{t}}^{-1}F_{t}\tilde{A_{t}}^{-1}}.
\]
To find the expectation of $\tilde{A}_{t}$, recall that for $s\in[t]$,
the pseudo-action $\tilde{a}_{s}$ is sampled from $\tilde{\pi}_{s}$
defined in~\eqref{eq:tilde_distribution}. Let $\Expectation_{s}[\cdot]:=\Expectation[\cdot|\sigma(\cup_{u\notin\mE_{t},u<s}\{w_{s},\check{w}_{s},\check{n}_{s}\}\cup\{\tilde{a}_{1},\ldots,\tilde{a}_{s-1}\})]$
denote the conditional expectation at round $s$. Then $\Expectation_{s}[U_{i,s}]=1$.
Define 
\begin{equation}
D_{s}:=\begin{cases}
F_{t}^{-1/2}\left[\sum_{i=1}^{d+1}\left(U_{i,s}-1\right)\tilde{x}_{i,s}\tilde{x}_{i,s}^{\top}\right]F_{t}^{-1/2} & s\in\mE_{t}\\
F_{t}^{-1/2}\left[\sum_{i=1}^{d+1}\left(U_{i,s}\Tx is\Tx is^{\top}-\frac{2d+1}{d}\sum_{i=0}^{d}\tilde{x}_{i,s}\tilde{x}_{i,s}^{\top}-x_{a_{s}}x_{a_{s}}^{\top}\right)\right]F_{t}^{-1/2} & s\notin\mE_{t}
\end{cases}.\label{eq:mix_martingale_difference}
\end{equation}
Note that $D_{s}$ is martingale difference because for $s\notin\mE_{t}$,
\begin{align*}
 & \Expectation_{s}\left[F_{t}^{-1/2}\left(\sum_{i=0}^{d+1}U_{i,s}\Tx is\Tx is^{\top}\right)F_{t}^{-1/2}\right]\\
 & =F_{t}^{-1/2}\Expectation_{s}\left[\sum_{i=0}^{d+1}U_{i,s}\Tx is\Tx is^{\top}\right]F_{t}^{-1/2}\\
 & =F_{t}^{-1/2}\Expectation_{s}\left[\sum_{i=0}^{d+1}\Tx is\Tx is^{\top}\right]F_{t}^{-1/2}\\
 & =F_{t}^{-1/2}\Expectation_{s}\left[\sum_{i=0}^{d+1}\left(\check{w}_{s}\tilde{x}_{i,s}+\check{w}_{s}\sqrt{\lambda_{i_{s}}}u_{i_{s}}\right)\left(\check{w}_{s}\tilde{x}_{i,s}+\check{w}_{s}\sqrt{\lambda_{i_{s}}}u_{i_{s}}\right)^{\top}\right]F_{t}^{-1/2}\\
 & =F_{t}^{-1/2}\Expectation_{s}\left[\sum_{i=0}^{d+1}\tilde{x}_{i,s}\tilde{x}_{i,s}^{\top}+(d+1)\lambda_{i_{s}}u_{i_{s}}u_{i_{s}}^{\top}\right]F_{t}^{-1/2}\\
 & =F_{t}^{-1/2}\Expectation_{s}\left[\sum_{i=0}^{d+1}\left(\tilde{x}_{i,s}\tilde{x}_{i,s}^{\top}+\frac{d+1}{d}\lambda_{i}u_{i}u_{i}^{\top}\right)\right]F_{t}^{-1/2}\\
 & =F_{t}^{-1/2}\Expectation_{s}\left[\frac{2d+1}{d}\sum_{i=0}^{d}\tilde{x}_{i,s}\tilde{x}_{i,s}^{\top}+x_{a_{s}}x_{a_{s}}^{\top}\right]F_{t}^{-1/2}
\end{align*}
where the second last equality holds by $w_{t},\check{w}_{t}\sim\text{unif}[-\sqrt{3},\sqrt{3}]$
and the last equality holds by $i_{s}\sim\text{unif}[d]$. For $s\in\mE_{t}$,
\begin{align*}
\norm{D_{s}}_{2}\le & 2d\max_{i\in[d+1]}\norm{F_{t}^{-1/2}\tilde{x}_{i,s}\tilde{x}_{i,s}^{\top}F_{t}^{-1/2}}_{2}\\
\le & 2d\max_{i\in[d+1]}\norm{\tilde{x}_{i,s}}_{F_{t}^{-1}}^{2}\\
\le & \frac{2d}{t},
\end{align*}
where the last inequality can be easily found by the fact $F_{t}\succeq t\sum_{i=1}^{d}\tilde{x}_{i,s}\tilde{x}_{i,s}+I_{d}=t\sum_{k=1}^{K}x_{k}x_{k}^{\top}+I_{d}$
and following the proof of Lemma \ref{lem:norm_bound}. For $s\notin\mE_{t}$,
\begin{align*}
\norm{D_{s}}_{2}\le & 2d\max_{i\in[d+1]}\norm{F_{t}^{-1/2}\Tx is\Tx is^{\top}F_{t}^{-1/2}}_{2}\\
\le & 2\left(w_{t}^{2}+\check{w}_{t}^{2}\right)d\max_{i\in[d+1]}\norm{F_{t}^{-1/2}\tilde{x}_{i,s}\tilde{x}_{i,s}^{\top}F_{t}^{-1/2}}_{2}\\
\le & \frac{12d}{t},
\end{align*}
where the last inequality holds by $\check{w}_{s},\check{w}_{s}\sim$unif$[-\sqrt{3},\sqrt{3}]$.
Thus, the eigenvalue of the martingale difference matrix lies in $[-12d/t,12d/t].$
Then by Hoeffding bound for matrices (Lemma \ref{lem:matrix_hoeffding}),
\[
\Probability\left(\norm{\sum_{s=1}^{t}D_{s}}_{2}>x\right)\le2d\exp\left(-\frac{tx^{2}}{288d^{2}}\right).
\]
Note that 
\begin{align*}
F_{t}\preceq & t\sum_{k=1}^{K}x_{k}x_{k}^{\top}+\sum_{s=1}^{t}x_{a_{s}}x_{a_{s}}^{\top}+I_{d}\\
= & \sum_{s=1}^{t}\sum_{i=1}^{d+1}\tilde{x}_{i,s}\tilde{x}_{i,s}^{\top}+I_{d}\\
\preceq & \sum_{s\in\mE_{t}}\sum_{i=1}^{d+1}\tilde{x}_{i,s}\tilde{x}_{i,s}^{\top}+\sum_{s\notin\mE_{t}}\left(\frac{2d+1}{d}\sum_{i=0}^{d}\tilde{x}_{i,s}\tilde{x}_{i,s}^{\top}+x_{a_{s}}x_{a_{s}}^{\top}\right),
\end{align*}
where the last term appears in the martingale difference \eqref{eq:mix_martingale_difference}.
Thus, 
\begin{align*}
\Probability\left(\Maxeigen{I_{d}-F_{t}^{-1/2}\tilde{A}_{t}F_{t}^{-1/2}}>x\right) & =\Probability\left(\Maxeigen{F_{t}^{-1/2}\left\{ F_{t}-\tilde{A}_{t}\right\} F_{t}^{-1/2}}>x\right)\\
 & \le\Probability\left(\Maxeigen{\sum_{s=1}^{t}D_{s}}>x\right)\\
 & \le2d\exp\left(-\frac{tx^{2}}{288d^{2}}\right).
\end{align*}
Set $x=1/2$. For $t\in\Natural$ such that $t\ge4\cdot288d^{2}\log\frac{2dt^{2}}{\delta}$,
with probability at least $1-\delta/t^{2}$, 
\[
\Maxeigen{I_{d}-F_{t}^{-1/2}\tilde{A}_{t}F_{t}^{-1/2}}\le\frac{1}{2},
\]
which implies 
\[
I_{d}-F_{t}^{-1/2}\tilde{A}_{t}F_{t}^{-1/2}\preceq\frac{1}{2}I_{d}\Rightarrow\frac{1}{2}I_{d}\preceq F_{t}^{-1/2}\tilde{A}_{t}F_{t}^{-1/2}.
\]
Because the matrix $F_{t}^{-1/2}\tilde{A}_{t}F_{t}^{-1/2}$ is symmetric
and positive definite, $F_{t}^{1/2}\tilde{A}_{t}^{-1}F_{t}^{1/2}\preceq2I_{d}$,
and thus, 
\[
\tilde{A_{t}}^{-1}F_{t}\tilde{A_{t}}^{-1}=F_{t}^{-1/2}\left(F_{t}^{-1/2}\tilde{A}_{t}F_{t}^{-1/2}\right)^{-2}F_{t}^{-1/2}\preceq4F_{t}^{-1}.
\]
Then the self-normalized norm, 
\begin{equation}
\begin{split}\norm{\Impute t{\ell}-\!\Parameter{\ell}}_{F_{t}}= & \norm{\sum_{s\in\mE_{t}}2x_{a_{s}}\Error{\ell}s+\sum_{s\notin\mE_{t}}2\Tx{a_{s}}s\tilde{\eta}_{s}^{\langle\ell\rangle}-\Parameter{\ell}}_{\tilde{A_{t}}^{-1}F_{t}\tilde{A_{t}}^{-1}}\\
\le & \norm{\sum_{s\in\mE_{t}}2x_{a_{s}}\Error{\ell}s+\sum_{s\notin\mE_{t}}2\Tx{a_{s}}s\tilde{\eta}_{s}^{\langle\ell\rangle}-\Parameter{\ell}}_{4F_{t}^{-1}}\\
\le & 4\norm{\sum_{s\in\mE_{t}}x_{a_{s}}\Error{\ell}s+\sum_{s\notin\mE_{t}}\Tx{a_{s}}s\tilde{\eta}_{s}^{\langle\ell\rangle}-\Parameter{\ell}}_{F_{t}^{-1}}+2\norm{\Parameter{\ell}}_{F_{t}^{-1}}.
\end{split}
\label{eq:mix_self_norm}
\end{equation}
In the first term, 
\begin{align*}
\Tx{a_{s}}s\tilde{\eta}_{s}^{\langle\ell\rangle}= & \left(w_{s}x_{a_{s}}+\check{w}_{s}x_{\Action{\check{n}_{s}}}\right)\left(w_{s}\Error{\ell}s+\check{w}_{s}\Error{\ell}{\check{n}_{s}}\right)\\
= & \Error{\ell}su_{s}^{2}x_{a_{s}}+\Error{\ell}{\check{n}_{s}}\check{w}_{s}^{2}x_{\Action{\check{n}_{s}}}+w_{s}\check{w}_{s}\left(\Error{\ell}sx_{\Action{\check{n}_{s}}}+\Error{\ell}{\check{n}_{t}}x_{a_{s}}\right)\\
= & \Error{\ell}s\left(w_{s}^{2}x_{a_{s}}+w_{s}\check{w}_{s}x_{\Action{\check{n}_{s}}}\right)+\Error l{\check{n}_{s}}\left(w_{s}\check{w}_{s}x_{a_{s}}+\check{w}_{s}^{2}x_{\Action{\check{n}_{s}}}\right).
\end{align*}
Note that for all exploitation round $s\in[t]\setminus\mE_{t}$, the
sampled round for reuse $\check{n}_{s}\in\mE_{t}$. Thus, we can decompose,
\begin{align*}
 & \sum_{s\in\mE_{t}}x_{\Action s}\Error{\ell}{s}+\sum_{s\notin\mE_{t}}\Tx{a_{s}}s\tilde{\eta}_{s}^{\langle \ell \rangle}\\
 & =\sum_{s\in\mE_{t}}x_{\Action s}\Error{\ell}{s}+\sum_{s\notin\mE_{t}}\Error{\ell}{s}\left(w_{s}^{2}x_{\Action s}+w_{s}\check{w}_{s}x_{\Action{\check{n}_{s}}}\right)+\Error l{\check{n}_{s}}\left(w_{s}\check{w}_{s}x_{\Action s}+\check{w}_{s}^{2}x_{\Action{\check{n}_{s}}}\right)\\
 & =\sum_{s\in\mE_{t}}x_{\Action s}\Error{\ell}{s}+\sum_{s\notin\mE_{t}}\Error{\ell}{s}\left(w_{s}^{2}x_{\Action s}+w_{s}\check{w}_{s}x_{\Action{\check{n}_{s}}}\right)\\
 & +\sum_{s\in\mE_{t}}\left\{ \left(\sum_{\tau\notin\mE_{t}}\Indicator{\check{n}_{\tau}=s}\check{w}_{\tau}^{2}\right)x_{\Action s}+\sum_{\tau\notin\mE_{t}}\Indicator{\check{n}_{\tau}=s}w_{\tau}\check{w}_{\tau}x_{\Action{\tau}}\right\} \Error{\ell}{s},
\end{align*}
and thus, 
\begin{equation}
\begin{split} & \norm{\sum_{s\in\mE_{t}}x_{a_{s}}\Error{\ell}{s}+\sum_{s\notin\mE_{t}}\Tx{a_{s}}s\tilde{\eta}_{s}^{\langle}-\Parameter{\ell}}_{F_{t}^{-1}}\\
 & \le\norm{\sum_{s\in\mE_{t}}x_{a_{s}}\Error{\ell}{s}+\sum_{s\notin\mE_{t}}\Error{\ell}{s}\left(w_{s}^{2}x_{a_{s}}+w_{s}\check{w}_{s}x_{\Action{\check{n}_{s}}}\right)}_{F_{t}^{-1}}\\
 & +\norm{\sum_{s\in\mE_{t}}\left\{ \left(\sum_{\tau\notin\mE_{t}}\Indicator{\check{n}_{\tau}=s}\check{w}_{\tau}^{2}\right)x_{a_{s}}+\sum_{\tau\notin\mE_{t}}\Indicator{\check{n}_{\tau}=s}w_{\tau}\check{w}_{\tau}x_{a_{\tau}}\right\} \Error{\ell}{s}}_{F_{t}^{-1}}
\end{split}
\label{eq:mix_decomp}
\end{equation}
To bound the first term, for $\lambda\in\Real^{d}$, define 
\begin{align*}
D_{s}^{\lambda} & :=\begin{cases}
\frac{\lambda^{\top}F_{t}^{-1/2}x_{a_{s}}\Error{\ell}{s}}{\sigma}-\frac{\lambda^{\top}\lambda}{t} & s\in\mE_{t}\\
\frac{\lambda^{\top}F_{t}^{-1/2}\left(w_{s}^{2}x_{a_{s}}+w_{s}\check{w}_{s}x_{\Action{\check{n}_{s}}}\right)\Error{\ell}{s}}{3\sigma}-\frac{2\lambda^{\top}\lambda}{t} & s\notin\mE_{t}
\end{cases}.
\end{align*}
Let $\Expectation_{s}$ denote a conditional expectation given errors
$\{\Error{\ell}{\tau}:\tau\in[s-1],\ell\in[L]\}$, actions $\{a_{\tau}:\tau\in[s]\}$,
random indexes $\{\check{n}_{\tau}:\tau\in[s]\setminus\mE_{t}\}$
and weights $\{w_{\tau},\check{w}_{\tau}:\tau\in[s]\setminus\mE_{t}\}$. For
$s\in\mE_{t}$, because $\Error{\ell}{s}$ is $\sigma$-sub-Gaussian and
$\lambda_{\max}(F_{t}^{-1/2}x_{k}x_{k}^{\top}F_{t}^{-1/2})=x_{k}^{\top}F_{t}^{-1}x_{k}\le1/t$
for all $x\in\mX$, 
\begin{align*}
\Expectation_{s}\left[\exp\left(D_{s}^{\lambda}\right)\right]\le & \Expectation_{s}\left[\exp\left(\frac{\lambda^{\top}F_{t}^{-1/2}x_{a_{s}}\Error{\ell}{s}}{\sigma}-\frac{\lambda^{\top}F_{t}^{-1/2}x_{a_{s}}x_{a_{s}}^{\top}F_{t}^{-1/2}\lambda}{2}\right)\right]\\
= & \Expectation_{s}\left[\exp\left\{ \frac{\lambda^{\top}F_{t}^{-1/2}x_{a_{s}}\Error{\ell}{s}}{\sigma}-\frac{\left(\lambda^{\top}F_{t}^{-1/2}x_{a_{s}}\right)^{2}}{2}\right\} \right]\\
\le & 1.
\end{align*}
For $s\notin\mE_{t}$, 
\begin{align*}
&\Expectation_{s}\left[\exp\left(D_{s}^{\lambda}\right)\right]\\&\le\!\exp\!\Big\{\frac{\sigma^{2}\big(w_{s}^{2}\lambda^{\top}F_{t}^{-1/2}x_{\Action s}+w_{s}\check{w}_{s}\lambda^{\top}F_{t}^{-1/2}x_{\Action{\check{n}_{s}}}\big)^{2}}{18\sigma^{2}}-\big(\lambda^{\top}F_{t}^{-1/2}x_{\PseudoAction s}\big)^{2}-\big(\lambda^{\top}F_{t}^{-1/2}x_{a_{\check{n}_{s}}}\big)^{2}\Big\}\\&\le\!\exp\!\Big[\frac{\sigma^{2}\big\{2w_{s}^{4}(\lambda^{\top}F_{t}^{-\frac{1}{2}}x_{\Action s})^{2}+2w_{s}^{2}\check{w}_{s}^{2}(\lambda^{\top}F_{t}^{-\frac{1}{2}}x_{\Action{\check{n}_{s}}})^{2}\big\}}{18\sigma^{2}}\!-\!\big(\lambda^{\top}F_{t}^{-\frac{1}{2}}x_{\PseudoAction s}\big)^{2}\!-\!\big(\lambda^{\top}F_{t}^{-\frac{1}{2}}x_{a_{\check{n}_{s}}}\big)^{2}\Big]\\&\le\!\exp\!\left(\frac{18\sigma^{2}\big\{(\lambda^{\top}F_{t}^{-1/2}x_{\Action s})^{2}+(\lambda^{\top}F_{t}^{-1/2}x_{\Action{\check{n}_{s}}})^{2}\big\}}{18\sigma^{2}}-\big(\lambda^{\top}F_{t}^{-1/2}x_{\PseudoAction s}\big)^{2}-\big(\lambda^{\top}F_{t}^{-1/2}x_{a_{\check{n}_{s}}}\big)^{2}\right)\\
&\le1,
\end{align*}
where the second inequality holds by $(a+b)^{2}\le2a^{2}+2b^{2}$
and the third inequality holds by $w_{s},\check{w}_{s}\in[-\sqrt{3},\sqrt{3}]$
almost surely. Thus, for all $\lambda\in\Real^{d}$, 
\[
\Expectation\left[\exp\left(\sum_{s=1}^{t}D_{s}^{\lambda}-2\lambda^{\top}\lambda\right)\right]\le1.
\]
Following the proof of Theorem 1 in \citet{abbasi2011improved}, with
probability at least $1-\delta/L$, 
\begin{align*}
 & \norm{\sum_{s\in\mE_{t}}x_{a_{s}}\Error{\ell}{s}+\sum_{s\notin\mE_{t}}\Error{\ell}{s}\left(w_{s}^{2}x_{\Action s}+w_{s}\check{w}_{s}x_{\Action{\check{n}_{s}}}\right)}_{F_{t}^{-1}}\\
 & \le\sqrt{3}\norm{\sum_{s\in\mE_{t}}x_{a_{s}}F_{t}^{-1/2}\Error{\ell}{s}+\sum_{s\notin\mE_{t}}\Error{\ell}{s}\left(w_{s}^{2}F_{t}^{-1/2}x_{\Action s}+w_{s}\check{w}_{s}F_{t}^{-1/2}x_{\Action{\check{n}_{s}}}\right)}_{\left(2I_{d}+I_{d}\right)^{-1}}\\
 & \le3\sigma\sqrt{3d\log\left(\frac{3Lt}{\delta}\right)}.
\end{align*}
To bound the second term in \eqref{eq:mix_decomp}, define 
\[
M_{s}^{\lambda}:=\lambda^{\top}F_{t}^{-1/2}\left\{ \left(\sum_{\tau\notin\mE_{t}}\Indicator{\check{n}_{\tau}=s}\check{w}_{\tau}^{2}\right)x_{a_{s}}+\sum_{\tau\notin\mE_{t}}\Indicator{\check{n}_{\tau}=s}w_{\tau}\check{w}_{\tau}x_{\tilde{a}_{\tau}}\right\} \frac{\Error{\ell}{s}\sqrt{\gamma_{t}}}{6\sigma\sqrt{t}},
\]
for $\lambda\in\Real^{d}$. Let $\Expectation_{s}$ denote a conditional
expectation given errors $\{\Error l{\tau}:\tau\in[s-1]\}$, pseudo-actions
$\{\tilde{a}_{\tau}:\tau\in[s]\}$, random indexes $\{\check{n}_{\tau}:\tau\in[s]\setminus\mE_{t}\}$.
For $s\in\mE_{t}$, for each $s\in\mE_{t}$, by Hoeffding's Lemma,
since $w_{s}\in[-\sqrt{3},\sqrt{3}]$, for any $v\in\Real$, 
\[
\Expectation_{s}\left[\exp\left(w_{s}v\right)\right]\le\Expectation_{s}\left[\frac{\sqrt{3}-w_{s}}{2\sqrt{3}}e^{-\sqrt{3}v}+\frac{w_{s}+\sqrt{3}}{2\sqrt{3}}e^{\sqrt{3}v}\right]=\cosh\left(\sqrt{3}|v|\right).
\]
Thus, 
\begin{align*}
 & \CE{\exp\left\{ \sum_{\tau\notin\mE_{t}}\Indicator{\check{n}_{\tau}=s}w_{\tau}\check{w}_{\tau}\lambda^{\top}F_{t}^{-1/2}x_{a_{\tau}}\frac{\Error{\ell}{s}\sqrt{\gamma_{t}}}{6\sigma\sqrt{t}}\right\} }{\Error{\ell}{s},\{\check{n}_{\tau},\check{w}_{\tau},a_{\tau}\}_{\tau\notin\mE_{t}}}\\
 & \le\prod_{\tau\notin\mE_{t}}\cosh\left(\frac{\sqrt{3\gamma_{t}}\Error{\ell}{s}}{6\sigma\sqrt{t}}\Indicator{\check{n}_{\tau}=s}\abs{\check{w}_{\tau}}\abs{\lambda^{\top}F_{t}^{-1/2}x_{a_{\tau}}\Error{\ell}{s}}\right)\\
 & \le\prod_{\tau\notin\mE_{t}}\cosh\left(\frac{\sqrt{\gamma_{t}}}{2\sigma\sqrt{t}}\max_{k\in[K]}\abs{\lambda^{\top}F_{t}^{-1/2}x_{k}}\Indicator{\check{n}_{\tau}=s}\abs{\Error{\ell}{s}}\right)
\end{align*}
Because $\cosh(x)\cosh(y)\le\cosh(x+y)$ for $x,y>0$, we obtain 
\begin{align*}
 & \CE{\exp\left\{ \sum_{\tau\notin\mE_{t}}\Indicator{\check{n}_{\tau}=s}w_{\tau}\check{w}_{\tau}\lambda^{\top}F_{t}^{-1/2}x_{a_{\tau}}\frac{\Error{\ell}{s}\sqrt{\gamma_{t}}}{6\sigma\sqrt{t}}\right\} }{\Error{\ell}{s},\{\check{n}_{\tau},\check{w}_{\tau},a_{\tau}\}_{\tau\notin\mE_{t}}}\\
 & \le\cosh\left(\frac{\sqrt{\gamma_{t}}}{2\sigma\sqrt{t}}\max_{k\in[K]}\abs{\lambda^{\top}F_{t}^{-1/2}x_{a_{s}}}\sum_{\tau\notin\mE_{t}}\Indicator{\check{n}_{\tau}=s}\abs{\Error{\ell}{s}}\right).
\end{align*}
For $s\in\mE_{t}$, let $N_{s}:=\sum_{\tau\notin\mE_{t}}\Indicator{\check{n}_{\tau}=s}$. 
Because $\cosh(\cdot)$ is an even function, $\cosh(x|y|)=\cosh(xy)$
for any $x,y\in\Real$ and 
\begin{align*}
&\Expectation_{s}\left[\exp\left(M_{s}^{\lambda}\right)\right]\\&\le\Expectation_{s}\bigg[\exp\!\left(\left(N_{s}\check{w}_{\tau}^{2}\right)\lambda^{\top}F_{t}^{-1/2}x_{a_{s}}\frac{\Error{\ell}s\sqrt{\gamma_{t}}}{6\sigma\sqrt{t}}\right)\cosh\left(\frac{\sqrt{\gamma_{t}}}{2\sigma\sqrt{t}}\max_{k\in[K]}\abs{\lambda^{\top}F_{t}^{-1/2}x_{k}}N_{s}\Error{\ell}s\right)\bigg]\\&=\frac{1}{2}\Expectation_{s}\left[\exp\!\left(\left\{ N_{s}\check{w}_{\tau}^{2}\frac{\lambda^{\top}F_{t}^{-1/2}x_{a_{s}}\sqrt{\gamma_{t}}}{6\sigma\sqrt{t}}\!+\!\frac{\sqrt{\gamma_{t}}}{2\sigma\sqrt{t}}\max_{k\in[K]}\abs{\lambda^{\top}F_{t}^{-1/2}x_{k}}N_{s}\right\} \Error{\ell}s\right)\right]\\&\;+\frac{1}{2}\Expectation_{s}\left[\exp\!\left(\left\{ N_{s}\check{w}_{\tau}^{2}\frac{\lambda^{\top}F_{t}^{-1/2}x_{a_{s}}\sqrt{\gamma_{t}}}{6\sigma\sqrt{t}}\!-\!\frac{\sqrt{\gamma_{t}}}{2\sigma\sqrt{t}}\max_{k\in[K]}\abs{\lambda^{\top}F_{t}^{-1/2}x_{k}}N_{s}\right\} \Error{\ell}s\right)\right].
\end{align*}
Because $\Error{\ell}{s}$ is $\sigma$-sub-Gaussian and $(a-b)^{2}+(a+b)^{2}=2a^{2}+2b^{2}$
for $a,b\in\Real$, 
\begin{align*}
\Expectation_{s}\left[\exp\left(M_{s}^{\lambda}\right)\right]&\le\exp\left(\frac{\gamma_{t}}{8t}\left\{ \left(N_{s}\check{w}_{\tau}^{2}\frac{\lambda^{\top}F_{t}^{-1/2}x_{\tilde{a}_{s}}}{3}\right)^{2}+\left(\max_{k\in[K]}\abs{\lambda^{\top}F_{t}^{-1/2}x_{k}}N_{s}\right)^{2}\right\} \right)\\&\le\exp\left(\frac{\gamma_{t}}{4t}\max_{k\in[K]}\abs{\lambda^{\top}F_{t}^{-1/2}x_{k}}^{2}N_{s}^{2}\right).
\end{align*}
By definition of $\check{n}_{\tau}$ the number of reusing round $s\in\mE_{t}$,
\[
N_{s}:=\sum_{\tau\notin\mE_{t}}\Indicator{\check{n}_{\tau}=s}\le\frac{\sum_{\tau\notin\mE_{t}}\Indicator{a_{\check{n}_{\tau}}=a_{s}}}{\sum_{u\in\mE_{t}}\Indicator{a_{u}=a_{s}}}+1\le\frac{\sum_{\tau=1}^{t}\Indicator{a_{\check{n}_{\tau}}=a_{s}}}{\sum_{u\in\mE_{t}}\Indicator{a_{u}=a_{s}}}\le\frac{t}{\gamma_{t}},
\]
where the last inequality holds by construction of the exploration set \eqref{eq:exploration_set}. 
Thus, 
\[
\Expectation_{s}\left[\exp\left(M_{s}^{\lambda}\right)\right]\le\exp\left(\frac{1}{4}\max_{k\in[K]}\abs{\lambda^{\top}F_{t}^{-1/2}x_{k}}^{2}N_{s}\right)
\]
By the fact that $t=\sum_{s\in\mE_{t}}N_s$
\begin{align*}
\Expectation\left[\exp\left(\sum_{s\in\mE_{t}}M_{s}^{\lambda}-2\lambda^{\top}\lambda\right)\right]\le&\Expectation\left[\exp\left(\sum_{s\in\mE_{t}}M_{s}^{\lambda}-t\max_{k\in[K]}\abs{\lambda^{\top}F_{t}^{-1/2}x_{k}}^{2}\right)\right]\\
=&\Expectation\left[\exp\left(\sum_{s\in\mE_{t}}\left(M_{s}^{\lambda}-\max_{k\in[K]}\abs{\lambda^{\top}F_{t}^{-1/2}x_{k}}^{2}N_{s}\right)\right)\right]\\
\le&1.
\end{align*}
Following the proof of Theorem 1 in \citet{abbasi2011improved}, with
probability at least $1-\delta/L$, 
\begin{align*}
 & \norm{\sum_{s\in\mE_{t}}\left\{ \left(\sum_{\tau\notin\mE_{t}}\Indicator{\check{n}_{\tau}=s}\check{w}_{\tau}^{2}\right)x_{a_{s}}+\sum_{\tau\notin\mE_{t}}\Indicator{\check{n}_{\tau}=s}w_{\tau}\check{w}_{\tau}x_{a_{\tau}}\right\} \Error{\ell}{s}}_{F_{t}^{-1}}\\
 & =\sqrt{3}\norm{\sum_{s\in\mE_{t}}\left\{ \left(\sum_{\tau\notin\mE_{t}}\Indicator{\check{n}_{\tau}=s}\check{w}_{\tau}^{2}\right)x_{a_{s}}+\sum_{\tau\notin\mE_{t}}\Indicator{\check{n}_{\tau}=s}w_{\tau}\check{w}_{\tau}x_{a_{\tau}}\right\} \Error{\ell}{s}}_{(2I_{d}+I_{d})^{-1}}\\
 & \le6\sigma\sqrt{\frac{3td}{\abs{\mE_{t}}}\log\left(\frac{3Lt}{\delta}\right)}.
\end{align*}
In summary, with probability at least $1-4\delta$, 
\begin{align*}
\norm{\Impute t{\ell}-\!\Parameter{\ell}}_{F_{t}}\le & 2\norm{\Parameter{\ell}}_{F_{t}^{-1}}+12\sigma\sqrt{3d\log\left(\frac{3Lt}{\delta}\right)}+24\sigma\sqrt{\frac{3td}{\abs{\mE_{t}}}\log\left(\frac{3Lt}{\delta}\right)},
\end{align*}
for all $\ell\in[L]$.
\end{proof}

\subsection{Robustness of the Doubly-Robust Estimation}

\label{sec:DR_robustness}

We prove a lemma on the robustness of the general DR estimator. 
\begin{lem}[Robustness of DR estimator to the error of imputation estimator.]
\label{lem:DR_robust} For $t\ge1$ let $F_{t}:=t\sum_{k=1}^{K}x_{k}x_{k}^{\top}+I_{d}$.
For any $x\in\mX$ and $\delta\in(0,1)$, the DR estimator $\Estimator{\ell}t$
employing $\tilde{\theta}^{\langle\ell\rangle}\in\Real^{d}$as an
imputation estimator satisfies 
\[
\abs{x^{\top}(\Estimator{\ell}t-\Parameter{\ell})}\le\norm x_{F_{t}^{-1}}\Big(\theta_{\max}+2\sigma\sqrt{d\log\frac{Lt}{\delta}}+3d\sqrt{\frac{2}{t}\log\frac{2dt^{2}}{\delta}}\norm{\tilde{\theta}^{\langle\ell\rangle}-\!\Parameter{\ell}}_{F_{t}}\Big)
\]
with probability at least $1-3\delta$ for all $t\ge1$ and $\ell\in[L]$.
For each $x\in\mX,$
\[
\abs{x^{\top}(\Estimator{\ell}t-\Parameter{\ell})}\le\norm x_{F_{t}^{-1}}\Big(\theta_{\max}+8\sigma\sqrt{2\log\frac{4Lt^{2}}{\delta}}+3d\sqrt{\frac{2}{t}\log\frac{2dt^{2}}{\delta}}\norm{\tilde{\theta}^{\langle\ell\rangle}-\!\Parameter{\ell}}_{F_{t}}\Big),
\]
with probability at least $1-3\delta$ for all $t\ge1$ and $\ell\in[L]$. 
\end{lem}

The first term and second terms correspond to the convergence rate
obtained from conventional self-normalizing bound. The third term
is the $\|\cdot\|_{F_{t}}$-error of the estimator $\tilde{\theta}$.
The $\|\cdot\|_{F_{t}}$-error of $\tilde{\theta}$ is multiplied
with the $O(dt^{-1/2})$ term, which comes from the fact that in the
pseudo-rewards, 
\[
\PseudoY{\ell}it:=\sqrt{\lambda_{i}}u_{i}^{\top}\tilde{\theta}^{\langle\ell\rangle}\left\{ 1-\frac{\mathbb{I}(\PseudoAction t=i)}{\tilde{\pi}_{i}^{(t)}}\right\} +\frac{\mathbb{I}(\PseudoAction t=i)}{\tilde{\pi}_{i}^{(t)}}\Newy{\ell}it,\:\forall i=1,\ldots,d+1,
\]
the reward estimate $\sqrt{\lambda_{i}}u_{i}^{\top}\tilde{\theta}$
is multiplied by the mean-zero random variable, $\left\{ 1-\frac{\mathbb{I}(\PseudoAction t=i)}{\tilde{\pi}_{i}^{(t)}}\right\} $.

The error of the imputation estimator is normalized by the Gram matrix
$F_{t}$ which consist of all $K$ contexts. Thus, the $\|\cdot\|_{F_{t}}$
error of the conventional ridge estimator, which uses only selected
contexts and rewards in every round, is $\Omega(\sqrt{t})$, which
yields slow convergence rate.
\begin{proof}
Let us fix $t\ge1$ throughout the proof. For $s\in[t]$ and $\ell\in[L]$,
let $\widehat{\eta}_{i,s}^{\langle\ell\rangle}:=\PseudoY{\ell}is-\tilde{x}_{i,s}\Parameter{\ell}$,
where
\[
\tilde{x}_{i,s}:=\begin{cases}
\sqrt{\lambda_{i}}u_{i} & \forall i\in[d]\\
x_{a_{s}} & i=d+1
\end{cases}.
\]
Let $V_{t}:=\sum_{s=1}^{t}\Indicator{\mathcal{M}_{s}}\sum_{i=1}^{d+1}\tilde{x}_{i,s}\tilde{x}_{i,s}^{\top}+I_{d}$.
By the definition of the estimator and the pseudo-reward $\PseudoY{\ell}is$,
for $x\in\mX$
\[
\begin{split}x^{\top}\left(\Estimator{\ell}t-\Parameter{\ell}\right)= & -x^{\top}V_{t}^{-1}\Parameter{\ell}+x^{\top}V_{t}^{-1}\left\{ \sum_{s=1}^{t}\sum_{i=1}^{d+1}\widehat{\eta}_{i,s}^{\langle\ell\rangle}\tilde{x}_{i,s}\right\} \\
= & -x^{\top}V_{t}^{-1}\Parameter{\ell}+x^{\top}V_{t}^{-1}\left\{ \sum_{s=1}^{t}\frac{\Reward{\ell}{\PseudoAction s}s-x_{\PseudoAction s}^{\top}\Parameter{\ell}}{\tilde{\pi}_{\tilde{a}_{s}}}\tilde{x}_{\tilde{a}_{s},s}\right\} \\
 & \mbox{}+x^{\top}V_{t}^{-1}\left\{ \sum_{s=1}^{t}\sum_{i=1}^{d+1}\left(\!1\!-\!\frac{\Indicator{\PseudoAction s=i}}{\tilde{\pi}_{\tilde{a}_{s}}}\!\right)\tilde{x}_{i,s}\tilde{x}_{i,s}^{\top}\right\} \left(\tilde{\theta}^{\langle\ell\rangle}-\!\Parameter{\ell}\right).
\end{split}
\]
On the coupling event $\cap_{s=1}^{t}\mathcal{M}_{s}$, we have $\Reward{\ell}{\PseudoAction s}s=\Reward{\ell}{a_{s}}s$,
$\tilde{x}_{\tilde{a}_{s},s}=x_{a_{s}}$ and $\tilde{\pi}_{\tilde{a}_{s}}=1/2$.
Thus, the first term, 
\begin{align*}
x^{\top}V_{t}^{-1}\left\{ \sum_{s=1}^{t}\frac{\Reward{\ell}{\PseudoAction s}s-x_{\PseudoAction s}^{\top}\Parameter{\ell}}{\tilde{\pi}_{\tilde{a}_{s}}}\tilde{x}_{\tilde{a}_{s},s}\right\} = & 2x^{\top}V_{t}^{-1}\left\{ \sum_{s=1}^{t}\left(\Reward{\ell}{a_{s}}s-x_{a_{s}}^{\top}\Parameter{\ell}\right)x_{a_{s}}\right\} \\
= & 2x^{\top}V_{t}^{-1}\left\{ \sum_{s=1}^{t}\Error{\ell}sx_{a_{s}}\right\} \\
\le & 2\norm x_{V_{t}^{-1}}\norm{\sum_{s=1}^{t}\Error{\ell}sx_{a_{s}}}_{V_{t}^{-1}}.
\end{align*}
By definition of $F_{t}$, on the coupling event $\cap_{s=1}^{t}\mathcal{M}_{s},$
\[
V_{t}=\sum_{s=1}^{t}\sum_{i=1}^{d+1}\tilde{x}_{i,s}\tilde{x}_{i,s}^{\top}+I_{d}\succeq t\sum_{k=1}^{K}x_{k}x_{k}^{\top}+I_{d}:=F_{t}.
\]
Thus,
\[
x^{\top}V_{t}^{-1}\left\{ \sum_{s=1}^{t}\frac{\Reward{\ell}{\PseudoAction s}s-x_{\PseudoAction s}^{\top}\Parameter{\ell}}{\tilde{\pi}_{\tilde{a}_{s}}}\tilde{x}_{\tilde{a}_{s},s}\right\} \le2\norm x_{F_{t}^{-1}}\norm{\sum_{s=1}^{t}\Error{\ell}sx_{a_{s}}}_{F_{t}^{-1}}
\]
For the second term, define $A_{t}:=\sum_{s=1}^{t}\sum_{i=1}^{d+1}\frac{\Indicator{\PseudoAction s=i}}{\tilde{\pi}_{\tilde{a}_{s}}}\tilde{x}_{i,s}\tilde{x}_{i,s}^{\top}+I_{d}$.
Then, 
\begin{equation}
\begin{split} & \abs{x^{\top}\left(\Estimator{\ell}t-\Parameter{\ell}\right)}\\
 & \le\abs{x^{\top}V_{t}^{-1}\Parameter{\ell}}+2\norm x_{V_{t}^{-1}}\norm{\sum_{s=1}^{t}\Error{\ell}sx_{\Action s}}_{V_{t}^{-1}}+\abs{x^{\top}V_{t}^{-1}\left(V_{t}-A_{t}\right)\left(\tilde{\theta}^{\langle\ell\rangle}-\!\Parameter{\ell}\right)}\\
 & \le\norm x_{F_{t}^{-1}}\left(\norm{\Parameter{\ell}}_{F_{t}^{-1}}+2\norm{\sum_{s=1}^{t}\Error{\ell}sx_{\Action s}}_{F_{t}^{-1}}+\norm{V_{t}^{-1/2}\left(V_{t}-A_{t}\right)\left(\tilde{\theta}^{\langle\ell\rangle}-\!\Parameter{\ell}\right)}_{2}\right)\\
 & \le\norm x_{F_{t}^{-1}}\left(\norm{\Parameter{\ell}}_{F_{t}^{-1}}+2\norm{\sum_{s=1}^{t}\Error{\ell}sx_{\Action s}}_{F_{t}^{-1}}+\norm{V_{t}^{-1/2}\left(V_{t}-A_{t}\right)F_{t}^{-1/2}}_{2}\norm{\tilde{\theta}^{\langle\ell\rangle}-\!\Parameter{\ell}}_{F_{t}}\right).
\end{split}
\label{eq:DR2}
\end{equation}
Because $V_{t}\succeq F_{t}$,
\[
\norm{V_{t}^{-1/2}\left(V_{t}-A_{t}\right)F_{t}^{-1/2}}_{2}\le\norm{F_{t}^{-1/2}\left(V_{t}-A_{t}\right)F_{t}^{-1/2}}_{2}
\]
In the last term of \eqref{eq:DR2}, 
\begin{align*}
\norm{F_{t}^{-1/2}\left(V_{t}-A_{t}\right)F_{t}^{-1/2}}_{2}= & \norm{\sum_{s=1}^{t}\sum_{i=1}^{d+1}\left\{ 1-\frac{\Indicator{\PseudoAction s=i}}{\tilde{\pi}_{\tilde{a}_{s}}}\right\} F_{t}^{-1/2}\tilde{x}_{i,s}\tilde{x}_{i,s}^{\top}F_{t}^{-1/2}}_{2}
\end{align*}
For each $s\in[t]$, the matrix 
\[
\sum_{i=1}^{d+1}\left\{ 1-\frac{\Indicator{\PseudoAction s=i}}{\tilde{\pi}_{\tilde{a}_{s}}}\right\} F_{t}^{-1/2}\tilde{x}_{i,s}\tilde{x}_{i,s}^{\top}F_{t}^{-1/2}
\]
is symmetric and a martingale difference matrix. Moreover, 
\begin{align*}
 & \norm{\sum_{i=1}^{d+1}\left\{ 1-\frac{\Indicator{\PseudoAction s=i}}{\tilde{\pi}_{\tilde{a}_{s}}}\right\} F_{t}^{-1/2}\tilde{x}_{i,s}\tilde{x}_{i,s}^{\top}F_{t}^{-1/2}}_{2}\\
 & \le\sum_{i=1}^{d+1}\abs{1-\frac{\Indicator{\PseudoAction s=i}}{\tilde{\pi}_{\tilde{a}_{s}}}}\max_{i\in[d+1]}\norm{F_{t}^{-1/2}\tilde{x}_{i,s}\tilde{x}_{i,s}^{\top}F_{t}^{-1/2}}_{2}\\
 & \le\left(d+2d-1\right)\max_{i\in[d+1]}\norm{F_{t}^{-1/2}\tilde{x}_{i,s}\tilde{x}_{i,s}^{\top}F_{t}^{-1/2}}_{2}\\
 & =3d\max_{i\in[d+1]}\norm{F_{t}^{-1/2}\tilde{x}_{i,s}\tilde{x}_{i,s}^{\top}F_{t}^{-1/2}}_{2}.
\end{align*}
Because $\max_{i\in[d+1]}\tilde{x}_{i,s}^{\top}F_{t}^{-1}\tilde{x}_{i,s}\le1/t$
for $s\in[t]$, 
\[
\norm{\sum_{i=1}^{d+1}\left\{ 1-\frac{\Indicator{\PseudoAction s=i}}{\tilde{\pi}_{\tilde{a}_{s}}}\right\} F_{t}^{-1/2}\tilde{x}_{i,s}\tilde{x}_{i,s}^{\top}F_{t}^{-1/2}}_{2}\le\frac{3d}{t},
\]
almost surely. By the Hoeffding bound for the matrix (Lemma~\ref{lem:matrix_hoeffding}),
with probability at least $1-\delta/t^{2}$, 
\begin{equation}
\norm{\sum_{s=1}^{t}\sum_{i=1}^{d+1}\left\{ 1-\frac{\Indicator{\PseudoAction s=i}}{\tilde{\pi}_{\tilde{a}_{s}}}\right\} F_{t}^{-1/2}\tilde{x}_{i,s}\tilde{x}_{i,s}^{\top}F_{t}^{-1/2}}_{2}\le3d\sqrt{\frac{2}{t}\log\frac{2dt^{2}}{\delta}}.\label{eq:DR_matrix_bound}
\end{equation}
Plugging in \eqref{eq:DR2}, 
\begin{equation}
\begin{split} & \abs{x^{\top}\left(\Estimator{\ell}t-\Parameter{\ell}\right)}\\
 & \le\norm x_{F_{t}^{-1}}\left(\norm{\Parameter{\ell}}_{F_{t}^{-1}}+2\norm{\sum_{s=1}^{t}\Error{\ell}sx_{\Action s}}_{F_{t}^{-1}}+3d\sqrt{\frac{2}{t}\log\frac{2dt^{2}}{\delta}}\norm{\tilde{\theta}^{\langle\ell\rangle}-\!\Parameter{\ell}}_{F_{t}}\right)\\
 & \le\norm x_{F_{t}^{-1}}\left(\theta_{\max}+2\norm{\sum_{s=1}^{t}\Error{\ell}sx_{\Action s}}_{F_{t}^{-1}}+3d\sqrt{\frac{2}{t}\log\frac{2dt^{2}}{\delta}}\norm{\tilde{\theta}^{\langle\ell\rangle}-\!\Parameter{\ell}}_{F_{t}}\right).
\end{split}
\label{eq:DR3}
\end{equation}
Note that $F_{t}$ is not random and $\Error{\ell}s$ is $\sigma$-sub-Gaussian.
Thus, by Lemma 9 in~\citet{abbasi2011improved}, with probability at least $1-\delta/L$,
\[
\norm{\sum_{s=1}^{t}\Error{\ell}sx_{\Action s}}_{F_{t}^{-1}}\le\norm{\sum_{s=1}^{t}\Error{\ell}sx_{\Action s}}_{(\sum_{s=1}^{t}x_{a_{s}}x_{a_{s}}^{\top}+I_{d})^{-1}}\le\sigma\sqrt{\log\frac{L\det\left(\sum_{s=1}^{t}x_{a_{s}}x_{a_{s}}^{\top}+I_{d}\right)}{\delta}},
\]
for $t\ge1$. Because $\|x_{k}\|_{2}\le1$ for all $k\in[K]$, 
\begin{align*}
\det\left(\sum_{s=1}^{t}x_{a_{s}}x_{a_{s}}^{\top}+I_{d}\right)\le & \left\{ \frac{\Trace{\sum_{s=1}^{t}x_{a_{s}}x_{a_{s}}^{\top}+I_{d}}}{d}\right\} ^{d}\\
\le & \left\{ \frac{t+d}{d}\right\} ^{d}\\
\le & t^{d},
\end{align*}
which implies,
\[
\norm{\sum_{s=1}^{t}\Error{\ell}sx_{\Action s}}_{F_{t}^{-1}}\le\sigma\sqrt{d\log\frac{Lt}{\delta}},
\]
and proves the first bound.

For the second bound, by Lemma \ref{lem:dim_free_bound}, with probability
at least $1-\delta/t^{2}$, 
\[
\norm{\sum_{s=1}^{t}\Error{\ell}sx_{a_{s}}}_{F_{t}^{-1}}\le\norm{\sum_{s=1}^{t}\Error lsF_{t}^{-1/2}x_{a_{s}}}_{2}\le4\sigma\sqrt{2\left(\sum_{s=1}^{t}\norm{F_{t}^{-1/2}x_{a_{s}}}_{2}^{2}\right)\frac{4Lt^{2}}{\delta}}
\]
By Lemma \ref{lem:norm_bound},
\[
\sum_{s=1}^{t}\norm{F_{t}^{-1/2}x_{a_{s}}}_{2}^{2}\le t\max_{k\in[K]}x_{k}^{\top}F_{t}^{-1}x_{k}\le1
\]
which implies 
\[
\norm{\sum_{s=1}^{t}\Error{\ell}sx_{a_{s}}}_{F_{t}^{-1}}\le4\sigma\sqrt{2\log\frac{4Lt^{2}}{\delta}}.
\]
which proves the second bound. 
\end{proof}

\subsection{Proof of Theorem~\ref{thm:self}}

\begin{thm}[Theorem \ref{thm:self} restated.]
Let $\Estimator \ell t$ denote the DR-mix estimator~\eqref{eq:estimator} with the exploration-mixed estimator~\eqref{eq:mixup} as the imputation estimator and pseudo-rewards~\eqref{eq:pseudoY}. 
Let $F_{t}:=\sum_{k=1}^{K}x_{k}x_{k}^{\top}+I_{d}$,
Then, for all $k\in[K]$, $\ell \in [L]$, and $t\ge T_{\gamma}$,
\begin{equation}
\big|x_k^{\top}(\Estimator \ell t-\Parameter \ell)\big|\le3\| x_k\|_{F_{t}^{-1}}\{\theta_{\max}+\sigma\sqrt{d\log(Lt/\delta)}\}.
\label{eq:est_bound_union}
\end{equation}
with probability at least $1-7\delta$.
For each $k\in[K]$, with probability at least $1-7\delta$,
\begin{equation}
\big|x_k^{\top}(\Estimator \ell t-\Parameter \ell)\big|\le3\| x_k\|_{F_{t}^{-1}} \big\{ \theta_{\max}+3\sigma\sqrt{\log(4Lt^{2}/\delta)}\big\}.
\label{eq:est_bound}
\end{equation}
\end{thm}

\begin{proof}
The proof for the two bounds are derived by simple computation using
Lemma \ref{lem:DR_robust} and we only prove the second bound. By
Lemma \ref{lem:DR_robust}, with probability at least $1-3\delta$,
\[
\abs{x^{\top}\left(\Estimator{\ell}t-\Parameter{\ell}\right)}\le\norm x_{F_{t}^{-1}}\left(\theta_{\max}+8\sigma\sqrt{2\log\frac{4Lt^{2}}{\delta}}+3d\sqrt{\frac{2}{t}\log\frac{2dt^{2}}{\delta}}\norm{\Impute t{\ell}-\!\Parameter{\ell}}_{F_{t}}\right).
\]
for all $x\in\mX$, $\ell\in[L]$ and $t\ge1$. By Lemma \ref{lem:self_mix},
with probability at least $1-4\delta$, 
\begin{align*}
\norm{\Impute t{\ell}-\!\Parameter{\ell}}_{F_{t}}\le & 2\norm{\Parameter{\ell}}_{F_{t}^{-1}}+12\sigma\sqrt{3d\log\left(\frac{3Lt}{\delta}\right)}+24\sigma\sqrt{\frac{3td}{\gamma_{t}}\log\left(\frac{3Lt}{\delta}\right)}\\
\le & 2\theta_{\max}+25\sigma\sqrt{\frac{3td}{\gamma_{t}}\log\frac{3Lt}{\delta}},
\end{align*}
where the last inequality holds by sufficiently large $t\ge144\gamma_{t}$.
Then, 
\[
\abs{x^{\top}\left(\Estimator{\ell}t-\Parameter{\ell}\right)}\le\norm x_{F_{t}^{-1}}\left(3\theta_{\max}+8\sigma\sqrt{\log\frac{4Lt^{2}}{\delta}}+75\sigma d\sqrt{\frac{6d}{\gamma_{t}}\log\frac{3Lt}{\delta}\log\frac{2dt^{2}}{\delta}}\right)
\]
for all $\ell\in[L]$. By construction, 
\begin{equation}
\gamma_{t}=C\cdot d^{3}\log(2dt^{2}/\delta):=6\cdot(75)^{2}d^{3}\log(2dt^{2}/\delta)\label{eq:C}
\end{equation}
 and 
\[
\abs{x^{\top}\left(\Estimator{\ell}t-\Parameter{\ell}\right)}\le3\norm x_{F_{t}^{-1}}\left(\theta_{\max}+3\sigma\sqrt{\log\frac{4Lt^{2}}{\delta}}\right).
\]
\end{proof}

\subsection{Instant Arm Elimination of \texttt{PFIwR}}
\label{sec:elimination}

\begin{lem}[Instant arm elimination in \texttt{PFIwR}.]
\label{lem:elimination} 
For $t\ge T_{\gamma}$ such that $\max_{k\in[K]}\beta_{k,t}<\Delta_{(k),\epsilon}/4$, the arm $k$ is correctly identified as suboptimal or Pareto optimal, and $k\notin\mA_{t+1}$.
\end{lem}

\begin{proof}
\textbf{Case 1. $k\notin\mP_{\star}$ and $\Delta_{k}^{\star}>\epsilon$:}
Suppose $k\in[K]\setminus\mP_{\star}$ and $\Delta_{(k),\epsilon}/4>\beta_{t}$.
If $\Delta_{(k),\epsilon}=\Delta_{k}^{\star}$, then there exists $k_{\star}\in\mP_{\star}\subseteq\mA_{t}\cup\mP_{t}$
such that $\Delta_{k}^{\star}=m(k,k^{\star})$ and 
\[
\widehat{m}_{t}(k,k^{\star})\ge m(k,k^{\star})-\beta_{k,t}-\beta_{k^{\prime},t}>4\beta_{t}-\beta_{k,t}-\beta_{k^{\prime},t}>\beta_{k,t}+\beta_{k^{\prime},t}.
\]
Thus, $k\in\mA_{t}\setminus\mC_{t}$, and $k\notin\mA_{t+1}$. 

\textbf{Case 2. $k\notin\mP_{\star}$ and $\Delta_{k}^{\star}\le\epsilon$:}
If $k\notin\mC_{t}$ then $k\notin\mA_{t}$. Consider the case of
$k\in\mC_{t}$. Because $\mP_{\star}\subset\mC_{t}\cup\mP_{t-1}$,
we obtain $\Delta_{k}^{\star}=\max_{k^{\prime}\in\mC_{t}\cup\mP_{t-1}}m(k,k^{\prime}).$Then
for all $k^{\prime}\in\mC_{t}\cup\mP_{t-1}\setminus\{k\}$, 
\begin{align*}
\widehat{M}_{t}^{2\epsilon}(k,k^{\prime})\ge & \max\left\{ 0,2\epsilon+\max_{\ell \in [L]}(y_{k}^{\langle \ell \rangle}-y_{k^{\prime}}^{\langle \ell \rangle})\right\} -\beta_{k,t}-\beta_{k^{\prime},t}\\
\ge & 2\epsilon-\min_{\ell \in [L]}(y_{k^{\prime}}^{\langle \ell \rangle}-y_{k}^{\langle \ell \rangle})-\beta_{k,t}-\beta_{k^{\prime},t}\\
\ge & 2\epsilon-m(k,k^{\prime})-\beta_{k,t}-\beta_{k^{\prime},t}\\
\ge & 2\epsilon-\Delta_{k}^{\star}-\beta_{k,t}-\beta_{k^{\prime},t}.
\end{align*}
Because $\Delta_{k}^{\star}\le\epsilon$,
\begin{align*}
\widehat{M}_{t}^{2\epsilon}(k,k^{\prime})\ge & \epsilon-\beta_{k,t}-\beta_{k^{\prime},t}\\
\ge & 4\max_{k\in[K]}\beta_{k,t}-\beta_{k,t}-\beta_{k^{\prime},t}\\
\ge & \beta_{k,t}+\beta_{k^{\prime},t},
\end{align*}
and $k\in\mP_{t}^{(1)}$. Thus, $k\in\mP_{t+1}$ and $k\notin\mA_{t+1}$. 

\textbf{Case 3. $k\in\mP_{\star}$ and $\Delta_{k}>\epsilon$:}
Suppose $k\in\mP_{\star}$. Then for all $k^{\prime}\in[K]$, 
\[
\widehat{m}_{t}(k,k^{\prime})\le m(k,k^{\prime})+\beta_{k,t}+\beta_{k^{\prime},t}=\beta_{k,t}+\beta_{k^{\prime},t},
\]
and $k\in\mC_{t}$. Suppose $\Delta_{(k),\epsilon}/4>\beta_{t}$. For
a Pareto optimal arm $k^{\prime}\in\mC_{t}\cup\mP_{t-1}\setminus\{k\}$,
\begin{align*}
\widehat{M}_{t}^{2\epsilon}(k,k^{\prime})\ge & M(k,k^{\prime})-\beta_{k,t}-\beta_{k^{\prime},t}\\
\ge & \Delta_{k}^{+}-\beta_{k,t}-\beta_{k^{\prime},t}\\
\ge & \Delta_{k}-\beta_{k,t}-\beta_{k^{\prime},t}\\
\ge & \beta_{k,t}+\beta_{k^{\prime},t}.
\end{align*}
For a suboptimal arm $k^{-}\in\mC_{t}\cup\mP_{t-1}\setminus\{k\}$,
there exists $k^{+}\in\mP_{\star}$ such that $y_{k^{-}}+\Delta_{k^{-}}^{\star}$
is weakly dominated by $y_{k^{+}}$, and 
\begin{align*}
\widehat{M}_{t}^{2\epsilon}(k,k^{-})\ge & M(k,k^{-})-\beta_{k,t}-\beta_{k^{-},t}\\
\ge & M(k,k^{+})+\Delta_{k^{-}}^{\star}-\beta_{k,t}-\beta_{k^{-},t}.
\end{align*}
Consider the case $k^{+}\neq k$, then
\begin{align*}
\widehat{M}_{t}^{2\epsilon}(k,k^{-})\ge & M(k,k^{+})-\beta_{k,t}-\beta_{k^{-},t}\\
\ge & \Delta_{k}^{+}-\beta_{k,t}-\beta_{k^{-},t}\\
\ge & \beta_{k,t}+\beta_{k^{-},t},
\end{align*}
and $k\in\mP_{t}^{(1)}$ and $k\notin\mA_{t+1}$. For the case of
$k^{+}=k$, 
\begin{align*}
\widehat{M}_{t}^{2\epsilon}(k,k^{-})\ge & \Delta_{k^{-}}^{\star}-\beta_{k,t}-\beta_{k^{-},t}\\
= & M(k^{-},k)+\Delta_{k^{-}}^{\star}-\beta_{k,t}-\beta_{k^{-},t}\\
\ge & \Delta_{k}^{-}-\beta_{k,t}-\beta_{k^{-},t}\\
\ge & \beta_{k,t}+\beta_{k^{-},t}.
\end{align*}

\textbf{Case 4. $k\in\mP_{\star}$ and $\Delta_{k}\le\epsilon$:}
If $\Delta_{(k),\epsilon}=\epsilon$ then $\Delta_{k}\le\epsilon$.
Thus, for all $k^{\prime}\in\mC_{t}\cup\mP_{t-1}\setminus\{k\}$,
because $k\in\mP_{\star}$, we obtain $\max_{\ell \in [L]}(y_{k}^{\langle \ell \rangle}-y_{k^{\prime}}^{\langle \ell \rangle})\ge0$
and
\[
\widehat{M}_{t}^{2\epsilon}(k,k^{\prime})\ge\max\left\{ 0,2\epsilon+\max_{\ell \in [L]}(y_{k}^{\langle \ell \rangle}-y_{k^{\prime}}^{\langle \ell \rangle})\right\} -\beta_{k,t}-\beta_{k^{\prime},t}\ge2\epsilon-\beta_{k,t}-\beta_{k^{\prime},t}\ge\beta_{k,t}+\beta_{k^{\prime},t},
\]
and $k\in\mP_{t}^{(1)}$. Thus, $k\notin\mA_{t+1}$. 
\end{proof}

\subsection{Proof of Theorem \ref{thm:sample}}
\label{subsec:sample_proof}

Before we prove the sample complexity, we provide an important properties of our proposed \PFIwR\ algorithm.

\begin{lem}
\label{lem:Pareto_inclusion} For $t\ge1$, the Pareto optimal arms
are either in $\mA_{t}$ or $\mP_{t}$ in $\PFIwR$, i.e., $\mP_{\star}\subseteq\mA_{t}\cup\mP_{t}$. 
\end{lem}

\begin{proof}
When $t=0$, the result holds by definition of $\mA_{0}=[K]$. For
$t\ge1$, suppose $\mP_{\star}\subseteq\mA_{t-1}\cup\mP_{t-1}$ holds.
While updating $\mA_{t}$ and $\mP_{t}$, only arms in $\mA_{t-1}\backslash\mC_{t}$
are eliminated. Thus, we prove the results by showing that $\mA_{t-1}\backslash\mC_{t}\subseteq\mP_{\star}^{c}$.
For each round $t$, suppose an arm $k\in\mA_{t-1}\backslash\mC_{t}$.
Then there exists $k^{\prime}\in[K]$ such that 
\[
\widehat{m}_{t}(k,k^{\prime})>\beta_{k,t}+\beta_{k^{\prime},t},
\]
which implies 
\[
\Ereward lk\le\Hreward lkt+\beta_{k,t}\le\Hreward l{k^{\prime}}t-\beta_{k^{\prime},t}\le\Ereward l{k^{\prime}},
\]
for all $\ell \in [L]$ and $k\notin\mP_{\star}$. Thus, $\mA_{t-1}\backslash\mC_{t}\subseteq\mP_{\star}^{c}$
is proved.
\end{proof}

Now we are ready to prove the sample complexity of \PFIwR\ . 
\begin{thm}[Theorem~\ref{thm:sample} restated]
Fix~$\epsilon>0$ and $\delta\in(0,1)$. Define $\Delta_{(k),\epsilon}=\max\{\epsilon,\Delta_{k}\}$,
where $\Delta_{k}$ is the required accuracy defined in~\eqref{eq:est_acc}
with ascending order $\Delta_{1}\le\cdots,\le\Delta_{K}$. Then the
stopping time $\tau_{\epsilon,\delta}$ of \texttt{PFIwR} is bounded
by: 
\[
\max\!\bigg\{\!O\bigg(\sum_{k=1}^{d}\frac{d}{\Delta_{(k),\epsilon}^{2}}\!\log\frac{dL}{\Delta_{\epsilon}^{2}\delta}\bigg),T_{\gamma}\!\bigg\}.
\]
\end{thm}

\begin{proof}
\textbf{Step 1. Sample complexity for accuracy of the estimator:}
For $k\in[K]$, let $\beta_{k,t}$ denote the confidence bound defined
in \eqref{eq:confidence_bound}. Because $\max_{k\in[K]}x_{k}^{\top}F_{t}^{-1}x_{k}\le1/t$,
for $k\in[K]$ 
\[
\beta_{k,t}\le\beta_{t}:=\begin{cases}
\frac{3}{\sqrt{t}}\{\theta_{\max}+\sigma\sqrt{d\log(7Lt/\delta)}\} & \abs{\mA_{t}}>d\\
\frac{3}{\sqrt{t}}\{\theta_{\max}+3\sigma\sqrt{\log(56Ldt^{2}/\delta)}\} & \abs{\mA_{t}}\le d
\end{cases}
\]
By Theorem \ref{thm:self}, with probability at least $1-\delta$,
\[
\abs{\Ereward{\ell}k-\Hreward{\ell}kt}\le\beta_{k,t}\le\beta_{t}.
\]
holds for all $t\ge T_{\gamma}$, $k\in[K]$ and $\ell\in[L]$ such
that $|\mA_{t}|>d$. Thus, for any $\Delta>0$, if 
\begin{equation}
t\ge\frac{9\left(\theta_{\max}+\sigma\right)^{2}}{\Delta^{2}}d\log\frac{7Lt}{\delta}\label{eq:sample1}
\end{equation}
then $\abs{\Ereward lk-\Hreward lkt}\le\beta_{k,t}\le\beta_{t}\le\Delta$
for all $k\in[K]$ when $|\mA_{t}|>d$. By Lemma \ref{lem:logt},
the condition \eqref{eq:sample1} is implied by 
\begin{equation}
t\ge\frac{4\cdot9d(\theta_{\max}+\sigma)^{2}}{\Delta^{2}}\left\{ 1+\log\frac{2\cdot9d(\theta_{\max}+\sigma)^{2}}{e\Delta^{2}}\sqrt{\frac{7L}{\delta}}\right\} .\label{eq:sample2}
\end{equation}
For $|\mA_{t}|\le d$, we only need confidence interval for at most
$2d$ arms that affects PFI. Let $\mathcal{N}_{t}$ denote the arms
that are nearest to $\mA_{t}$. Then for $k\in\mathcal{N}_{t}\cup\mA_{t}$,
\[
\abs{\Ereward{\ell}k-\Hreward{\ell}kt}\le\beta_{k,t}\le\beta_{t}
\]
holds for $|\mA_{t}|\le d$. Similarly,
\[
t\ge\frac{4\cdot9(\theta_{\max}+\sigma)^{2}}{\Delta^{2}}\left\{ 1+\log\frac{2\cdot9(\theta_{\max}+3\sigma)^{2}}{e\Delta^{2}}\sqrt{\frac{56Ld}{\delta}}\right\} ,
\]
implies $\abs{\Ereward{\ell}k-\Hreward{\ell}kt}\le\beta_{k,t}\le\beta_{t}\le\Delta$
for $k\in\mathcal{N}_{t}\cup\mA_{t}$. 

\textbf{Step 2. Finding the sample complexity:} From \eqref{eq:sample2},
for 
\[
t\ge\frac{16\cdot4\cdot9d(\theta_{\max}+\sigma)^{2}}{\Delta_{(d+1),\epsilon}^{2}}\left\{ 1+\log\frac{16\cdot2\cdot9d(\theta_{\max}+\sigma)^{2}}{e\Delta_{(d+1),\epsilon}^{2}}\sqrt{\frac{7L}{\delta}}\right\} 
\]
implies $\Delta_{(i),\epsilon}\ge4\beta_{i,t}$ for $i=d+1,\ldots,K$.
Then by Lemma \ref{lem:elimination}, $|\mA_{t+1}|\le d$. If 
\begin{align*}
t\ge & \frac{16\cdot4\cdot9d(\theta_{\max}+\sigma)^{2}}{\Delta_{(d+1),\epsilon}^{2}}\left\{ 1+\log\frac{16\cdot2\cdot9d(\theta_{\max}+\sigma)^{2}}{e\Delta_{(d+1),\epsilon}^{2}}\sqrt{\frac{7L}{\delta}}\right\} \\
 & +\sum_{k=1}^{d}\frac{16\cdot4\cdot9(\theta_{\max}+\sigma)^{2}}{\Delta_{(k),\epsilon}^{2}}\left\{ 1+\log\frac{16\cdot2\cdot9(\theta_{\max}+3\sigma)^{2}}{e\Delta_{(k),\epsilon}^{2}}\sqrt{\frac{56Ld}{\delta}}\right\} 
\end{align*}
then $\Delta_{k,\epsilon}\ge4\beta_{k,t}$ for all $k\in[K]$ and
$\mA_{t+1}=\emptyset$ by Lemma \ref{lem:elimination}. Since Theorem~\ref{thm:self}
requires $t\ge T_{\gamma}$, the sample complexity is bounded as 
\begin{align*}
\tau_{\epsilon,\delta}\! & \le\max\bigg[T_{\gamma},\frac{16\cdot4\cdot9d(\theta_{\max}+\sigma)^{2}}{\Delta_{(d+1),\epsilon}^{2}}\left\{ 1+\log\frac{16\cdot2\cdot9d(\theta_{\max}+\sigma)^{2}}{e\Delta_{(d+1),\epsilon}^{2}}\sqrt{\frac{7L}{\delta}}\right\} \\
 & \qquad+\sum_{k=1}^{d}\frac{16\cdot4\cdot9(\theta_{\max}+\sigma)^{2}}{\Delta_{(k),\epsilon}^{2}}\left\{ 1+\log\frac{16\cdot2\cdot9(\theta_{\max}+3\sigma)^{2}}{e\Delta_{(k),\epsilon}^{2}}\sqrt{\frac{56Ld}{\delta}}\right\} \bigg].
\end{align*}
The proof completes by the fact that $d\Delta_{(d+1),\epsilon}^{-2}\le\sum_{k=1}^{d}\Delta_{(k),\epsilon}^{-2}$. 
\end{proof}

\subsection{Proof of Theorem \ref{thm:regret}}

\label{subsec:regret_proof} 
\begin{proof}
By Lemma \ref{lem:Pareto_inclusion}, the Pareto front $\mP_{\star}\subseteq\mA_{t}\cup\mP_{t}$.
By definition of $\mA_{t}$ and $\mP_{t}$ in the algorithm, 
\begin{align*}
\mA_{t}\cup\mP_{t}= & \mC_{t}\cup\mP_{t-1}\\
= & \mC_{t}\cup\mP_{t-1}^{(2)}\cup\mP_{t-2}\\
\subseteq & \mC_{t}\cup\mC_{t-1}\cup\mP_{t-2}\\
\vdots & \vdots\\
\subseteq & \bigcup_{s=1}^{t}\mC_{s}.
\end{align*}
Note that $\mC_{s+1}\subseteq\mA_{s}\subseteq\mC_{s}$ for $s\ge1$.
Thus, $\mP_{\star}\subseteq\mC_{1}$, and the Pareto regret 
\[
\Delta_{\Action t}^{\star}=\max_{k\in\mP_{\star}}m(\Action t,k)\le\max_{k\in\mC_{1}}m(\Action t,k).
\]
For $s\in[t-1]$, suppose $k_{s}\in\mC_{s}\setminus\mC_{s+1}$. Then
there exists $k_{s}^{\prime}\in\mA_{s}$ such that $\widehat{m}_{s}(k_{s},k_{s}^{\prime})>D_{k_{s},k_{2}^{\prime},s}$.
By Theorem \ref{thm:self}, $m(k_{s},k_{s}^{\prime})>0$ with probability
at least $1-\delta$ and $k_{s}$ is dominated by $k_{s}^{\prime}\in\mA_{s}$,
which is dominated by the arms in $\mC_{s+1}$ by definition of $\mC_{s+1}$.
Thus, 
\begin{align*}
\max_{k\in\mC_{1}}m(\Action t,k)= & \max_{k\in\mC_{1}\setminus\mC_{2}\cup\mC_{2}}m(\Action t,k)\\
\le & \max_{k\in\mC_{2}}m(\Action t,k)\\
\vdots & \vdots\\
\le & \max_{k\in\mC_{t}}m(a_{t},k).
\end{align*}
By definition of $\Action t$, 
\begin{align*}
\Delta_{\Action t}^{\star}\le & \max_{k\in\mC_{t}}m(a_{t},k)\\
\le & \max_{k\in\mA_{t-1}}m(a_{t},k)\\
\le & 2\max_{j\in\mA_{t-1}}\beta_{j,t-1}+\max_{k\in\mA_{t-1}}\widehat{m}_{t-1}(a_{t},k)\\
= & 2\max_{j\in\mA_{t-1}}\beta_{j,t-1},
\end{align*}
which proves the instantaneous regret bound.

To prove the cumulative regret bound, summing up the regret over $s\in[\tau_{\epsilon,\delta}]$,
with probability at least $1-\delta,$ 
\begin{align*}
R(\tau_{\epsilon,\delta})\le & \sum_{t=1}^{\tau_{\epsilon,\delta}}\Indicator{t\in\mE_{t}}2\theta_{\max}+\Indicator{t\notin\mE_{t}}\Delta_{\Action t}^{\star}\\
= & 2\theta_{\max}\gamma_{\tau_{\epsilon,\delta}}+\sum_{t=1}^{\tau_{\epsilon,\delta}}\Indicator{t\notin\mE_{t}}\Delta_{\Action t}^{\star}\\
= & \bar{O}\left(\theta_{\max}d^{3}\log\frac{\theta_{\max}d}{\delta\Delta_{(1),\epsilon}}\right)+\sum_{t=1}^{\tau_{\epsilon,\delta}}\Indicator{t\notin\mE_{t}}\Delta_{\Action t}^{\star},
\end{align*}
where $\bar{O}$ ignores $\log\log(\cdot)$ terms and the last equality
holds by the sample complexity bound (Theorem \ref{thm:sample}).
Because the instantaeneous regret is bounded by $\Delta_{\Action t}^{\star}\le2\max_{k\in\mA_{t-1}}\beta_{k,t-1}$,
the regret is zero when $2\max_{k\in\mA_{t-1}}\beta_{k,t-1}\le\min_{k\in[K]\setminus\mP_{\star}}\Delta_{k}^{\star}$,
which is implied by $4\beta_{t-1}\le\min_{k\in[K]\setminus\mP_{\star}}\Delta_{k}^{\star}$.
In addition, by Lemma \ref{lem:elimination}, the algorithm terminates
when $\max_{k\in\mA_{t-1}}\beta_{k,t-1}\le\beta_{t-1}\le\epsilon/4\le\min_{k\in[K]}\Delta_{(k),\epsilon}$.
Thus, 
\begin{align*}
R(\tau_{\epsilon,\delta})\le & \bar{O}\left(\theta_{\max}d^{3}\log\frac{\theta_{\max}d}{\delta\Delta_{(1),\epsilon}}\right)+\sum_{t=1}^{\tau_{\epsilon,\delta}}\Indicator{4\beta_{t-1}>\max\left\{ \min_{k\in[K]\setminus\mP_{\star}}\Delta_{k}^{\star},\epsilon\right\} }\Indicator{t\notin\mE_{t}}\Delta_{\Action t}^{\star}\\
= & \bar{O}\left(\theta_{\max}d^{3}\log\frac{\theta_{\max}d}{\delta\Delta_{(1),\epsilon}}\right)+\sum_{t=1}^{\tau_{\epsilon,\delta}}\Indicator{4\beta_{t-1}>\Delta_{\epsilon}^{\star}}\Indicator{t\notin\mE_{t}}\Delta_{\Action t}^{\star}
\end{align*}
By \eqref{eq:sample2}, let 
\begin{equation}
g(\frac{\Delta_{\epsilon}^{\star}}{4}):=\frac{16\cdot4\cdot9d(\theta_{\max}+\sigma)^{2}}{(\Delta_{\epsilon}^{\star})^{2}}\left\{ 1+\log\frac{16\cdot2\cdot9d(\theta_{\max}+\sigma)^{2}}{e(\Delta_{\epsilon}^{\star})^{2}}\sqrt{\frac{7L}{\delta}}\right\} \label{eq:g_function}
\end{equation}
Then $t\ge g(\Delta_{\epsilon}^{\star}/4)$ implies $\Delta_{\epsilon}^{\star}\ge4\beta_{t}$
and 
\begin{align*}
\sum_{t=1}^{\tau_{\epsilon,\delta}}\Indicator{4\beta_{t-1}>\Delta_{\epsilon}^{\star}}\Indicator{t\notin\mE_{t}}\Delta_{\Action t}^{\star}\le & \sum_{t=1}^{\tau_{\epsilon,\delta}}\Indicator{t-1\le g(\Delta_{\epsilon}^{\star}/4)}\Indicator{t\notin\mE_{t}}\Delta_{\Action t}^{\star}\\
\le & \sum_{t=1}^{1+g(\Delta_{\epsilon}^{\star}/4)}\Delta_{\Action t}^{\star}.
\end{align*}
Because $\Delta_{a_{t}}^{\star}\le2\max_{k\in\mA_{t-1}}\beta_{k,t}$,
\begin{align*}
\sum_{t=1}^{1+g(\Delta_{\epsilon}^{\star}/4)}\Delta_{\Action t}^{\star}\le & 2\sum_{t=1}^{1+g(\Delta_{\epsilon}^{\star}/4)}\max_{k\in\mA_{t-1}}\beta_{k,t}\\
\le & 6\left(\theta_{\max}+\sigma\right)\sum_{t=1}^{1+g(\Delta_{\epsilon}^{\star}/4)}\frac{\sqrt{d\log(7Lt/\delta)}}{\sqrt{t}}\\
\le & 12\left(\theta_{\max}+\sigma\right)\sqrt{\{1+g(\Delta_{\epsilon}^{\star}/4)\}\log\frac{7L\{1+g(\Delta_{\epsilon}^{\star}/4)\}^{2}}{\delta}}.
\end{align*}
Plugging in \eqref{eq:g_function} and ignoring $\log\log(\cdot)$
terms, 
\begin{equation}
\begin{split}R(\tau_{\epsilon,\delta})\!= & \bar{O}\left(\theta_{\max}d^{3}\log\frac{\theta_{\max}d}{\delta\Delta_{(1),\epsilon}}+\frac{\theta_{\max}d\sigma^{2}}{\Delta_{\epsilon}^{\star}}\log\frac{\theta_{\max}d\sigma}{\Delta_{\epsilon}^{\star}\delta}\right)\end{split}
.\label{eq:regret_bound}
\end{equation}
\end{proof}

\subsection{Proof of Theorem~\ref{thm:lower_bound}}
\label{sec:regret_lower_bound}
\begin{thm}[Theorem~\ref{thm:lower_bound} restated.]
For $\epsilon>0$, let $\Delta_{\epsilon}^{\star}:=\max\{\epsilon,
\min_{k\in[K]\setminus\mathcal{P}_{\star}}\Delta_k^{\star}\}$ denote the
minimum Pareto regret over suboptimal arms. 
Suppose the set of context vectors $\mathcal{X}$ span $\mathbb{R}^d$ and $\min_{\ell \in [L]}\|\theta_{\star}^{\langle \ell \rangle}\|_0 = d$.
Then, for any $\delta \in (0,1/4)$ and $\sigma>0$, there exists a
$\sigma$-sub-Gaussian distribution for the i.i.d. noise sequence
$\{\eta_t\}_{t\ge1}$ such that for any PFI algorithms that satisfies PFI
success condition (1) with failure probability $\delta$, 
\[
  R(\tau_{\epsilon,\delta}) \ge
  \frac{\sqrt{3}d\sigma}{8\Delta_{\epsilon}^{\star}} \log
  \frac{1}{4\delta}. 
\]
\end{thm}

Theorem~\ref{thm:regret_lower_bound} shows that \PFIwR\ establishes nearly optimal regret among algorithms that achieve PFI and it is the first result on the trade-off between PFI and Pareto regret minimization. 
For $L=1$ and the contexts are Euclidean basis,  Theorem~\ref{thm:regret_lower_bound} recovers the lower bound for regret of BAI algorithms developed by~\citet{zhong2023achieving}. 
Note that the lower bound applies only to the algorithms that guarantee PFI; it is possible for an algorithm that does not guarantee PFI to have a regret lower bound that is lower than the one in Theorem~\ref{thm:regret_lower_bound}.

\begin{proof}
By Lemma~\ref{lem:lower_noise}, for $\delta\in(0,1/4)$ and any
estimator $\widehat{\theta}_{t}^{\langle \ell \rangle}$ with round $t\le d\sigma^{2}/(12(\Delta_{\epsilon}^{\star})^{-2}\log\frac{3(d+1)L}{4\delta})$
for $\Parameter \ell$,
\[
\Probability\left(\max_{\ell \in [L]}\max_{x\in\mX}\abs{x^{\top}\left(\widehat{\theta}_{t}^{\langle \ell \rangle}-\Parameter \ell\right)}>\Delta_{\epsilon}^{\star}\right)\ge1-\left(1-\frac{\delta}{3(d+1)L}\right)^{L}\ge\delta,
\]
and any estimator cannot find an arm with zero Pareto regret with probability at least $1-\delta$. 
Thus, we need at least $d\sigma^{2}/(12(\Delta_{\epsilon}^{\star})^{-2}\log\frac{3(d+1)L}{4\delta})$ number of rounds to ensure that the estimation error is less than the minimum Pareto regret $\Delta_{\epsilon}^{\star}$. 
By Theorem 5.6 in~\citet{kim2022squeeze}, for any horizon $T\ge d$, the expected regret is $\Omega(\sqrt{dT})$ in the single objective linear bandit setting where the number of arms is finite and the contexts span $\mathbb{R}^{d}$.
Since the same lower bound applies to the multi-dimensional rewards as well, setting $T=d\sigma^{2}/(12(\Delta_{\epsilon}^{\star})^{-2}\log\frac{1}{4\delta})$
gives the lower bound.
\end{proof}

\section{Technical Lemmas}
In this section, we provide technical lemmas cited from the literature and novel lemmas (Lemma~\ref{lem:matrix_hoeffding} and Lemma~\ref{lem:logt}).

\begin{lem}[Assouad's method, \citep{yu1997assouad}]
\label{lem:assouad} For $v\in\{\pm 1\}^{d}$, let $\Probability_{v}$
denote the probability measure on the data space $\mathcal{D}$ whose
parameter is $v$. For any collection of estimators $f=(f_{1},\ldots,f_{d}),f_{i}:\mathcal{D}\to\{\pm1\}$
there exists at least one $v\in\{\pm1\}^{d}$ such that 
\[
\Expectation_{v}\left[\sum_{j=1}^{d}\Indicator{f_{j}\neq v}\right]\ge\frac{d}{2}\min_{v,v^{\prime}:v\sim v^{\prime}}\norm{\min(\Probability_{v},\Probability_{v^{\prime}})}_{1},
\]
where $v\sim v^{\prime}$ indicates that $v$ and $v^{\prime}$ only
differ in one coordinate. 
\end{lem}

\begin{lem}[A dimension-free bound for vector-valued martingales. Lemma C.6 in
\citet{kim2023improved}.]
\label{lem:dim_free_bound} Let $\{\Filtration s\}_{s=0}^{t}$ be a
filtration and $\{\eta_{s}\}_{s=1}^{t}$ be a real-valued stochastic
process such that $\eta_{s}$ is $\Filtration{\tau}$-measurable.
Let $\left\{ X_{s}\right\} _{s=1}^{t}$ be an $\Real^{d}$-valued
stochastic process where $X_{s}$ is $\Filtration 0$-measurable.
Assume that $\{\eta_{s}\}_{s=1}^{t}$ are $\sigma$-sub-Gaussian given
$\{\Filtration s\}_{s=1}^{t}$. Then with probability at least $1-\delta$,
\begin{equation}
\norm{\sum_{s=1}^{t}\eta_{s}X_{s}}_{2}\le 4 \sigma\sqrt{\sum_{s=1}^{t}\norm{X_{s}}_{2}^{2}}\sqrt{2\log\frac{4t^{2}}{\delta}}.\label{eq:etaX_bound}
\end{equation}
\end{lem}

\begin{remark}
While the constant in \citet{kim2023improved} is $12$, we prove that the bound also holds with $4$ using the following lemma. 
\end{remark}

\begin{lem}
\label{lem:sub_Gaussian}
Suppose a random variable $X$ satisfies
$\Expectation[X]=0$, and let $Y$ be an $\sigma$-sub-Gaussian random variable. 
If $\abs X\le\abs Y$ almost surely, then $X$ is $2\sigma$-sub-Gaussian.
\end{lem}

\begin{lem}[A Hoeffding bound for the matrices] 
\label{lem:matrix_hoeffding} 
Let $\{M_{\tau}:\tau\in[t]\}$
be a $\mathbb{R}^{d\times d}$-valued stochastic process adapted to
the filtration $\{\mathcal{F_{\tau}}:\tau\in[t]\}$, i.e., $M_{\tau}$
is $\mathcal{F}_{\tau}$-measurable for $\tau\in[t]$. Suppose the
matrix $M_{\tau}$ is symmetric and the eigenvalues of the difference
$M_{\tau}-\mathbb{E}[M_{\tau}|\mathcal{F}_{\tau-1}]$ lies in $[-b,b]$
for some $b>0$. Then for $x>0$, 
\[
\mathbb{P}\left(\left\Vert \sum_{\tau=1}^{t}M_{\tau}-\mathbb{E}[M_{\tau}|\mathcal{F}_{\tau-1}]\right\Vert _{2}\ge x\right)\le2d\exp\left(-\frac{x^{2}}{2tb^{2}}\right)
\]
\end{lem}

\begin{proof}
The proof is an adapted version of Hoeffding's inequality for matrix
stochastic process with the argument of \cite{tropp2012user}. Let
$D_{\tau}:=M_{\tau}-\mathbb{E}[M_{\tau}|\mathcal{F}_{\tau-1}]$. Then,
for $x>0$, 
\[
\mathbb{P}\left(\left\Vert \sum_{\tau=1}^{t}D_{\tau}\right\Vert _{2}\ge x\right)\le\mathbb{P}\left(\lambda_{\max}\left(\sum_{\tau=1}^{t}D_{\tau}\right)\ge x\right)+\mathbb{P}\left(\lambda_{\max}\left(-\sum_{\tau=1}^{t}D_{\tau}\right)\ge x\right)
\]
We bound the first term and the second term is bounded with similar argument. 
For any $v>0$, 
\begin{align*}
\mathbb{P}\left(\lambda_{\max}\left(\sum_{\tau=1}^{t}D_{\tau}\right)\ge x\right)
&\le\mathbb{P}\left(\exp\left\{ v\lambda_{\max}\left(\sum_{\tau=1}^{t}D_{\tau}\right)\right\} \ge e^{vx}\right)\\
&\le e^{-vx}\mathbb{E}\left[\exp\left\{ v\lambda_{\max}\left(\sum_{\tau=1}^{t}D_{\tau}\right)\right\} \right].
\end{align*}
Because $\sum_{\tau=1}^{t}D_{\tau}$ is a real symmetric matrix, 
\begin{align*}
\exp\left\{ v\lambda_{\max}\left(\sum_{\tau=1}^{t}D_{\tau}\right)\right\} = & \lambda_{\max}\left\{ \exp\left(v\sum_{\tau=1}^{t}D_{\tau}\right)\right\} \le\text{Tr}\left\{ \exp\left(v\sum_{\tau=1}^{t}D_{\tau}\right)\right\} ,
\end{align*}
where the last inequality holds since $\exp(v\sum_{\tau=1}^{t}D_{\tau})$
has nonnegative eigenvalues. Taking expectation on both side gives,
\begin{align*}
\mathbb{E}\left[\exp\left\{ v\lambda_{\max}\left(\sum_{\tau=1}^{t}D_{\tau}\right)\right\} \right]\le & \mathbb{E}\left[\text{Tr}\left\{ \exp\left(v\sum_{\tau=1}^{t}D_{\tau}\right)\right\} \right]\\
= & \text{Tr}\mathbb{E}\left[\exp\left(v\sum_{\tau=1}^{t}D_{\tau}\right)\right]\\
= & \text{Tr}\mathbb{E}\left[\exp\left(v\sum_{\tau=1}^{t-1}D_{\tau}+\log\exp(vD_{t})\right)\right].
\end{align*}
By Lieb's theorem \cite{tropp2015introduction} the mapping $D\mapsto\exp(H+\log D)$
is concave on positive symmetric matrices for any symmetric positive
definite $H$. 
By Jensen's inequality,
\[
\text{Tr}\mathbb{E}\left[\exp\left(v\sum_{\tau=1}^{t-1}D_{\tau}+\log\exp(vD_{t})\right)\right]\le\text{Tr}\mathbb{E}\left[\exp\left(v\sum_{\tau=1}^{t-1}D_{\tau}+\log\CE{\exp(vD_{t})}{\mathcal{F}_{t-1}}\right)\right]
\]
By Hoeffding's lemma, 
\[
e^{vx}\le\frac{b-x}{2b}e^{-vb}+\frac{x+b}{2b}e^{vb}
\]
for all $x\in[-b,b]$. Because the eigenvalue of $D_{\tau}$ lies
in $[-b,b]$, we have 
\begin{align*}
\CE{\exp(vD_{t})}{\mathcal{F}_{t-1}}\preceq & \CE{\frac{e^{-vb}}{2b}\left(bI_{d}-D_{t}\right)+\frac{e^{vb}}{2b}\left(D_{t}+bI_{d}\right)}{\mathcal{F}_{t-1}}\\
= & \frac{e^{-vb}+e^{vb}}{2}I_{d}\\
\preceq & \exp(\frac{v^{2}b^{2}}{2})I_{d}.
\end{align*}
Recursively,
\begin{align*}
\mathbb{E}\left[\exp\left\{ v\lambda_{\max}\left(\sum_{\tau=1}^{t}D_{\tau}\right)\right\} \right]\le & \text{Tr}\mathbb{E}\left[\exp\left(v\sum_{\tau=1}^{t-1}D_{\tau}+\log\CE{\exp(vD_{t})}{\mathcal{F}_{t-1}}\right)\right]\\
\le & \text{Tr}\mathbb{E}\left[\exp\left(v\sum_{\tau=1}^{t-1}D_{\tau}+\frac{v^{2}b^{2}}{2}I_{d}\right)\right]\\
\le & \text{Tr}\mathbb{E}\left[\exp\left(v\sum_{\tau=1}^{t-2}D_{\tau}\!+\!\frac{v^{2}b^{2}}{2}I_{d}\!+\!\log\CE{\exp(vD_{t-1})}{\mathcal{F}_{t-2}}\right)\right]\\
\le & \text{Tr}\mathbb{E}\left[\exp\left(v\sum_{\tau=1}^{t-2}D_{\tau}+\frac{2v^{2}b^{2}}{2}I_{d}\right)\right]\\
\vdots & \vdots\\
\le & \text{Tr}\exp\left((\frac{tv^{2}b^{2}}{2})I_{d}\right)\\
= & \exp\left(\frac{tv^{2}b^{2}}{2}\right)\text{Tr}\left(I_{d}\right)\\
= & d\exp\left(\frac{tv^{2}b^{2}}{2}\right).
\end{align*}
Thus we have 
\[
\mathbb{P}\left(\lambda_{\max}\left(\sum_{\tau=1}^{t}D_{\tau}\right)\ge x\right)\le d\exp\left(-vx+\frac{tv^{2}b^{2}}{2}\right).
\]
Minimizing over $v>0$ gives $v=x/(tb^{2})$ and
\[
\mathbb{P}\left(\lambda_{\max}\left(\sum_{\tau=1}^{t}D_{\tau}\right)\ge x\right)\le d\exp\left(-\frac{x^{2}}{2tb^{2}}\right),
\]
which proves the lemma.
\end{proof}

\begin{lem}[Threshold for logarithmic inequality.]
\label{lem:logt} For $a>1/2$ and $b>e^{2}$, $t\ge4a\left(1+\log\frac{2a\sqrt{b}}{e}\right)$
implies $t\ge a\log bt^{2}$. 
\end{lem}

\begin{proof}
If $b>e^{2}$ the function $x\mapsto(a/x)\log bx^{2}$ has negative
derivatives and is decreasing on $x\ge1$. Then, there exists a unique
$\tilde{t}\ge1$ such that $1=(a/\tilde{t})\log b(\tilde{t})^{2}$.
Now, it is sufficient to show that 
\begin{equation}
\tilde{t}\le2a\left(1+\sqrt{\log(2a\sqrt{b})/e}\right)^{2}\le4a\left(1+\log\frac{2a\sqrt{b}}{e}\right).\label{eq:logt_1}
\end{equation}
Let $W:[-1/e,0)\to\Real$ denote the Lambert function which satisfies
$W(x)e^{W(x)}=x$ for $x\in[-1/e,0)$. By definition of $\tilde{t}$,
\begin{align*}
-\frac{1}{2a\sqrt{b}}=e^{\log\frac{1}{\sqrt{b}\tilde{t}}}\log\frac{1}{\sqrt{b}\tilde{t}}\Longrightarrow & W\left(-\frac{1}{2a\sqrt{b}}\right)=\log\frac{1}{\sqrt{b}\tilde{t}}\\
\Longrightarrow & \tilde{t}=\frac{1}{\sqrt{b}}\exp\left(-W\left(-\frac{1}{2a\sqrt{b}}\right)\right).
\end{align*}
By definition of $W$, 
\[
\tilde{t}=-\frac{2a\sqrt{b}}{\sqrt{b}}W\left(-\frac{1}{2a\sqrt{b}}\right)=-2aW\left(-\frac{1}{2a\sqrt{b}}\right).
\]
By Theorem 1 in \citet{chatzigeorgiou2013bounds}, $-W(-e^{-u-1})<1+\sqrt{2u}+u\le(1+\sqrt{u})^{2}$
for $u>0$. Setting $u=\log\frac{2a\sqrt{b}}{e}$ proves \eqref{eq:logt_1}. 
\end{proof}

\section{Limitation}
\label{sec:limitation}
Although our main contribution is novel and improves current linear bandit algorithms, we found the following limitations are in need to be handled in the future work:
\begin{itemize}
    \item The number of exploration $T_\gamma$ can dominate the sample complexity in Theorem~\ref{thm:sample} when the problem complexity gaps are large.
    The term $T_\gamma$ does not have problem complexity gap and is not considered as the main term in theoretical analysis; while our proposed estimator may not efficient for the large gaps in practice.
    \item Our PFI comparison experiments are limited to MAB setting, although we design our algorithm for general contexts with possibly exponentially large number of arms.
    We choose the MAB setting for the sake of comparison with the previous algorithm~\citep{auer2016pareto}; we believe the superior performance of our algorithm may be drastically visible on general contexts with large number of arms.
\end{itemize}


\end{document}